\documentclass{article}

\usepackage[final,nonatbib]{neurips_2021}

\usepackage[utf8]{inputenc} 
\usepackage[T1]{fontenc}    
\usepackage{hyperref}       
\usepackage{url}            
\usepackage{booktabs}       
\usepackage{amsfonts}       
\usepackage{nicefrac}       
\usepackage{microtype}      
\usepackage{enumerate}
\usepackage{multirow}
\usepackage{wrapfig}

\usepackage{graphicx}
\usepackage{caption}
\usepackage{subcaption}
\usepackage{soul}
\usepackage[table]{xcolor}
\usepackage{amssymb}
\usepackage{amsmath}
\usepackage{amsthm}
\usepackage{pifont}
\usepackage[numbers, sort, comma, square]{natbib}
\usepackage{soul}
\usepackage{tabularx,ragged2e,booktabs,caption}
\usepackage{pifont}
\usepackage{xcolor}

\hypersetup{
    colorlinks=true,
    citecolor = blue,
    linkcolor=blue
}

\newcommand{\indep}{\perp \!\!\! \perp}
\DeclareMathOperator{\pa}{Pa}

\newtheorem{theorem}{Theorem}
\newtheorem{definition}{Definition}
\newtheorem{corollary}{Corollary}
\newtheorem{proposition}{Proposition}

\definecolor{customred}{HTML}{CC0000}
\definecolor{customyellow}{HTML}{FCB900}
\definecolor{custompink}{HTML}{CC00CC}
\definecolor{customgreen}{HTML}{00CC66}
\definecolor{customblue}{HTML}{0066CC}

\newcommand{\xmarkred}{\textcolor{customred}{\ding{55}}\hfill}
\newcommand{\xmarkyellow}{\textcolor{customyellow}{\ding{55}}\hfill}
\newcommand{\xmarkpink}{\textcolor{custompink}{\ding{55}}\hfill}
\newcommand{\xmarkgreen}{\textcolor{customgreen}{\ding{55}}\hfill}
\newcommand{\xmarkblue}{\textcolor{customblue}{\ding{55}}\hfill}

\newcommand{\cmark}{\ding{51}\hfill}
\newcommand{\xmark}{\textcolor{lightgray}{\ding{55}}\hfill}

\newcolumntype{S}{>{\hsize=.7\hsize}X}
\newcolumntype{B}{>{\hsize=1.7\hsize}X}

\usepackage{xcolor}
\definecolor{mygreen}{RGB}{0, 102, 0}

\title{DECAF:  Generating Fair Synthetic Data Using Causally-Aware Generative Networks}

\author{%
  Boris van Breugel$^*$ \\
  University of Cambridge \\
  \texttt{bv292@cam.ac.uk} \\
  \And
  Trent Kyono\thanks{Equal contribution} \\
  University of California, Los Angeles\\
  \texttt{tmkyono@ucla.edu} \\
  \And
  Jeroen Berrevoets \\
  University of Cambridge \\
  \texttt{jb2384@cam.ac.uk} \\
  \And
  Mihaela van der Schaar \\
  University of Cambridge \\
  University of California, Los Angeles\\
  The Alan Turing Institute \\
  \texttt{mv472@cam.ac.uk} \\
}

\begin{document}

\maketitle


\begin{abstract}
Machine learning models have been criticized for reflecting unfair biases in the training data. 
Instead of solving for this by introducing fair learning algorithms directly, we focus on generating fair synthetic data, such that any downstream learner is fair.
Generating fair synthetic data from unfair data--- \emph{while remaining truthful to the underlying data-generating process (DGP)} ---is non-trivial.
In this paper, we introduce DECAF: a GAN-based {\it fair} synthetic data generator for tabular data. 
With DECAF we embed the DGP explicitly as a structural causal model in the input layers of the generator, allowing each variable to be reconstructed conditioned on its causal parents. 
This procedure enables \textit{inference-time} debiasing, where biased edges can be strategically removed for satisfying user-defined fairness requirements. 
The DECAF framework is versatile and compatible with several popular definitions of fairness. 
In our experiments, we show that DECAF successfully removes undesired bias and--- in contrast to existing methods ---is capable of generating high-quality synthetic data. 
Furthermore, we provide theoretical guarantees on the generator's convergence and the fairness of downstream models.
\end{abstract}

\section{Introduction}
Generative models are optimized to approximate the original data distribution as closely as possible. Most research focuses on three objectives \cite{alaa2021faithful}: fidelity, diversity, and privacy. The first and second are concerned with how closely synthetic samples resemble real data and how much of the real data's distribution is covered by the new distribution, respectively. The third objective aims to avoid simply reproducing samples from the original data, which is important if the data contains privacy-sensitive information \cite{dpgan, adsgan}. We explore a much-less studied concept: synthetic data fairness.

\textbf{Motivation.} 
Deployed machine learning models have been shown to reflect the bias of the data on which they are trained  \cite{tashea_17AD,dastin,Lu_contextual_bias,manela2021stereotype,kadambi2021achieving}. This has not only unfairly damaged the discriminated individuals but also society's trust in machine learning as a whole. A large body of work has explored ways of detecting bias and creating fair predictors \cite{kamiran2009classifying,feldman2015certifying,zhang2016causal,hardt2016equality, kusner2017counterfactual, kilbertus2017avoiding, zhang2018fairness}, while other authors propose debiasing the data itself \cite{kamiran2009classifying,feldman2015certifying,zhang2016causal,calmon2017optimized}. This work's aim is related to the work of \cite{xu2018fairgan}: to generate fair synthetic data based on unfair data. Being able to generate fair data is important because end-users creating models based on publicly available data might be unaware they are inadvertently including bias or insufficiently knowledgeable to remove it from their model. Furthermore, by debiasing the data prior to public release, one can guarantee {\it any} downstream model satisfies desired fairness requirements by assigning the responsibility of debiasing to the data generating entities.

{\bf Goal.} From a biased dataset $\mathcal{X}$, we are interested in learning a model $G$, that is able to generate an equivalent {\it synthetic} unbiased dataset $\mathcal{X}'$ with minimal loss of data utility. Furthermore, a downstream model trained on the synthetic data needs to make not only unbiased predictions on the synthetic data, but also on real-life datasets (as formalized in Section \ref{sec:dist_fairness}).



\textbf{Solution.}  We approach fairness from a causal standpoint because it provides an intuitive perspective on different definitions of fairness and discrimination \cite{zhang2016causal,kusner2017counterfactual,kilbertus2017avoiding,nabi2018fair,zhang2018fairness}.  We introduce DEbiasing CAusal Fairness (DECAF), a generative adversarial network (GAN) that leverages causal structure for synthesizing data.  Specifically, DECAF is comprised of $d$ generators (one for each variable) that learn the causal conditionals observed in the data. At inference-time, variables are synthesized topologically starting from the root nodes in the causal graph then synthesized sequentially, terminating at the leave nodes. Because of this, DECAF can remove bias at inference-time through targeted (biased) edge removal. As a result, various datasets can be created for desired (or evolving) definitions of fairness.

\textbf{Contributions.} We propose a framework of using causal knowledge for fair synthetic data generation. We make three main contributions: i) DECAF, 
a causal GAN-based model for generating synthetic data, ii) a flexible causal approach for modifying this model such that it can generate fair data, and iii) guarantees that downstream models trained on the synthetic data will also give fair predictions in other settings. Experimentally, we show how DECAF is compatible with several fairness/discrimination definitions used in literature while still maintaining high downstream utility of generated data.

\section{Related Works}
Here we focus on the related work concerned with data generation, in contrast to fairness definitions for which we provide a detailed overview in Section~\ref{sec:defining_fairness} and Appendix \ref{appx:fairness_defs_continued}. As an overview of how data generation methods relate to one another, we refer to Table~\ref{tab:related} which presents all relevant related methods.

\begin{table}[t]
    \centering
    \caption{{\bf Overview of related work for synthetic data.} We organize related work according to our key areas of interest: {\bf (1)} Allows post-hoc distribution changes, {\bf (2)} provides fairness, {\bf (3)} supports causal notion of fairness, {\bf (4)} allows inference-time fairness, {\bf (5)} requires minimal assumptions. We highlight the key contribution, and identify non-neural approaches with ``\textdagger''.}
    \label{tab:related}
    \begin{tabularx}{\textwidth}{l X *{5}{c} l}
    \toprule
        {\bf Model}& {\bf Reference}& {\bf (1)}& {\bf (2)}& {\bf (3)}& {\bf (4)}& {\bf (5)}& {\bf Goal}\\
    \rowcolor{black!5}\midrule\multicolumn{8}{c}{Standard {\it synthetic data generation}}\\

    VAE & \cite{vae}& \xmark & \xmark & \xmark & \xmark & \cmark & Realistic synth. data.\\
    GANs & \cite{goodfellow2014generative, gulrajani2017improved, adsgan, dpgan} & \xmark & \xmark & \xmark & \xmark & \cmark & Realistic synth. data.\\
    
    \rowcolor{black!5}\multicolumn{8}{c}{Methods that {\it detect and/or remove bias}}\\
    
    PSE-DD/DR\textsuperscript{\textdagger}& \cite{zhang2016causal}& \cmark& \cmark& \cmark& \xmark& \xmark& Discover/Remove bias.\\
    OPPDP\textsuperscript{\textdagger}& \cite{calmon2017optimized}& \xmark& \cmark& \xmark& \xmark& \xmark& Remove bias.\\
    DI\textsuperscript{\textdagger}& \cite{feldman2015certifying}& \xmark& \cmark& \xmark& \xmark& \xmark& Discover/Remove bias.\\
    LFR& \cite{zemel2013learning}& \xmark& \cmark& \xmark& \xmark& \cmark& Learn fair representation.\\
    FairGAN& \cite{xu2018fairgan}& \xmark& \cmark& \xmark& \xmark& \cmark& Realistic and fair synth. data.\\
    CFGAN& \cite{xu2019achieving}& \xmark& \cmark& \cmark& \xmark& \cmark& Realistic and fair synth. data.\\
    \midrule
    DECAF & (ours) & \cmark & \cmark & \cmark & \cmark & \cmark & Realistic and fair synth. data.\\

    \bottomrule
    \end{tabularx}
    \vspace{-0.5cm}
\end{table}

{\bf Non-parametric generative modeling.} The standard models for synthetic data generation are either based on VAEs \cite{vae} or GANs \cite{goodfellow2014generative, gulrajani2017improved, adsgan, dpgan}. While these models are well known for their highly realistic synthetic data, they are unable to alter the synthetic data distribution to encourage fairness (except for \cite{xu2018fairgan,xu2019achieving}, discussed below). Furthermore, these methods have no causal notion, which prohibits targeted interventions for synthesizing fair data (Section \ref{sec:defining_fairness}). We explicitly leave out CausalGAN \cite{kocaoglu2017causalgan} and CausalVAE \cite{yang2021causalvae}, which appear similar by incorporating causality-derived ideas but are different in both method and aim (i.e., image generation).

{\bf Fair data generation.} In the bottom section of Table~\ref{tab:related}, we present methods that, in some way, alter the training data of classifiers to adhere to a notion of fairness \cite{zhang2016causal, calmon2017optimized, xu2018fairgan,xu2019achieving, feldman2015certifying, zemel2013learning}. While these methods have proven successful, they lack some important features. For example, none of the related methods allow for post-hoc changes of the synthetic data distribution. This is an important feature, as each situation requires a different perspective on fairness and thus requires a flexible framework for selecting protected variables. Additionally, only \cite{zhang2016causal,xu2019achieving} allow a causal perspective on fairness, despite causal notions underlying multiple interpretations of what should be considered fair \cite{kusner2017counterfactual}. Furthermore, only \cite{xu2018fairgan,xu2019achieving,zemel2013learning} offer a flexible framework, while the others are limited to binary \cite{zhang2016causal,feldman2015certifying} or discrete \cite{calmon2017optimized} settings. \citet{xu2019achieving} also use a causal architecture for the generator, however their method is not as flexible---e.g. it does not easily extend to multiple protected attributes. Finally, in contrast to other methods DECAF is directly concerned with fairness of the downstream model---which is dependent on the setting in which the downstream model is employed (Section \ref{sec:dist_fairness}). In essence, from Table~\ref{tab:related} we learn that DECAF is the only method that combines all key areas of interest. At last, we would like to mention \cite{choi2020fair}, who aim to generate data that resembles a small unbiased reference dataset, by leveraging a large but biased dataset. This is very different to our aim, as we are interested in the downstream model's fairness and explicit notions of fairness.

\section{Preliminaries}
Let $X \in \mathcal{X} \subseteq \mathbb{R}^d$ denote a random variable with distribution $P_X(X)$, with protected attributes $A\in \mathcal{A}\subset \mathcal{X}$ and target variable $Y\in\mathcal{Y}\subset \mathcal{X}$, let $\hat{Y}$ denote a prediction of $Y$. Let the data be given by $\mathcal{D} = \{ \mathbf{x}^{(k)} \}_{k=1}^N$, where each $\mathbf{x}^{(k)} \in \mathcal{D}$ is a realization of $X$. We assume the data generating process can be represented by a directed acyclic graph (DAG)---such that the generation of features can be written as a structural equation model (SEM) \cite{pearl2009}---and that this DAG is causally sufficient. Let $X_i$ denote the $i$\textsuperscript{th} feature in $X$ with causal parents $\pa(X_i) \subset \{X_j:j\neq i\}$, the SEM is given by:
\begin{equation} \label{eq:SEM}
    X_i = f_i(\pa(X_i), Z_i), \forall i
\end{equation}
where $\{Z_i\}_{i=1}^d$ are independent random noise variables, that is $\pa(Z_i) = \emptyset,\ \forall i$. Note that each $f_i$ is a deterministic function that places all randomness of the conditional $P(X_i|\pa(X_i))$ in the respective noise variable, $Z_i$.

\section{Fairness of Synthetic Data} \label{sec:defining_fairness}
Algorithmic fairness is a popular topic (e.g., see \cite{barocas2016big,kusner2017counterfactual}), but \textit{fair synthetic data} has been much less explored. This section highlights how the underlying graphs of the synthetic and downstream data determine whether a model trained on the synthetic data will be fair in practice. We start with the two most popular definitions of fairness, relating to the legal concepts of \textit{direct} and \textit{indirect} discrimination. We also explore \textit{conditional fairness} \cite{kamiran2012}, which is a generalization of the two. In Appendix \ref{appx:fairness_defs_continued} we discuss how the ideas in this section transfer to other independence-based definitions \cite{barocas2017fairness}. Throughout this section, we separate $Y$ from $X$ by defining $\Bar{X}=X\backslash Y$, and we will write $X\leftarrow \Bar{X}$ for ease of notation.

\subsection{Algorithmic fairness} \label{sec:alg_fair}
The first definition is called \textit{Fairness Through Unawareness} (e.g. \cite{grgic2016case}).
\begin{definition}
(Fairness Through Unawareness (FTU): algorithm). A predictor $f:X\mapsto\hat{Y}$ is fair iff protected attributes $A$ are not explicitly used by $f$ to predict $\hat Y$.
\end{definition}
This definition prohibits \textit{disparate treatment} \cite{barocas2016big, zafar2017fairness}, and is related to the legal concept of {\it direct discrimination}, i.e., two equally qualified people deserve the same job opportunity independent of their race, gender, beliefs, among others. 

Though FTU fairness is commonly used, it might result in \textit{indirect discrimination}: covariates that influence the prediction $\hat Y$ might not be identically distributed across different groups $a, a'$, which means an algorithm might have \textit{disparate impact} on a protected group \cite{feldman2015certifying}. The second definition of fairness, \textit{demographic parity} \cite{zafar2017fairness}, does not allow this:
\begin{definition}
(Demographic Parity (DP): algorithm) A predictor $\hat{Y}$ is fair iff $A\indep \hat{Y}$, i.e. $\forall a,a': P(\hat{Y} |A =
a) = P(\hat{Y} |A = a')$.
\end{definition}
Evidently, DP puts stringent constraints on the algorithm, whereas FTU might be too lenient. The third definition we include is based on the work of \cite{kamiran2012}, related to \textit{unresolved discrimination} \cite{kilbertus2017avoiding}. The idea is that we do not allow indirect discrimination unless it runs through \textit{explanatory factors} $R\subset X$. For example, in Simpson's paradox \cite{simpson1951interpretation} there seems to be a bias between gender and college admissions, but this is only due to women applying to more competitive courses. In this case, one would want to regard fairness conditioned on the choice of study \cite{kilbertus2017avoiding}. Let us define this as \textit{conditional fairness}:
\begin{definition}
(Conditional Fairness (CF): algorithm) A predictor $\hat{Y}$ is fair iff $A\indep \hat{Y}|R$, i.e. $\forall r, a,a': P(\hat{Y} |R=r, A =
a) = P(\hat{Y} |R=r, A = a')$.
\end{definition}

{\bf CF generalizes FTU and DP} Note that conditional fairness is a generalization of FTU and DP, by setting $R=X\backslash A$ 
and $R=\emptyset$, respectively. In Appendix \ref{appx:fairness_defs_continued} we elaborate on the connection between these, and more, definitions.

\subsection{Synthetic data fairness} \label{sec:dist_fairness}
Algorithmic definitions can be extended to distributional fairness for synthetic data. 
Let $P(X),P'(X)$ be probability distributions with protected attributes $A\subset X$ and labels $Y\subset X$. Let $\mathcal{I}(A,Y)$ be a definition of algorithmic fairness (e.g., FTU). Note, that under CF, $\mathcal{I}(A,Y)$ is a function of $R$ as well. We propose $(\mathcal{I}(A,Y),P)$-fairness of distribution $P'(X)$:
\begin{definition}
(Distributional fairness) A probability distribution $P'(X)$ is $(\mathcal{I}(A,Y),P)$-fair, iff the optimal predictor $\hat{Y} = f^*(X)$ of $Y$ trained on $P'(X)$ satisfies $\mathcal{I}(A,Y)$ when evaluated on $P(X)$.  
\end{definition}
In other words, when we train a predictor on $(\mathcal{I}(A,Y), P)$-fair distribution $P'(X)$, we can only reach maximum performance if our model is fair. Note the explicit reference to $P(X)$, the distribution on which fairness is evaluated, which does not need to coincide with $P'(X)$. This is a small but relevant detail. For example, when training a model on data $\mathcal{D}'\sim P'(X)$ it could seem like the model is fair when we evaluate it on a hold-out set of the data (e.g., if we simply remove the protected attribute from the data). However, when we use the model for real-world predictions of data $\mathcal{D}\sim P(X)$, disparate impact is possibly observed due to a distributional shift.

By extension, we define synthetic data as $(\mathcal{I}(A,Y),P)$-fair, iff it is sampled from an $(\mathcal{I}(A,Y),P)$-fair distribution. Defining synthetic data as fair w.r.t. an optimal predictor is especially useful when we want to publish a dataset and do not trust end-users to consider anything but performance.\footnote{Finding the optimal predictor is possible if we assume the downstream user employs any universal function approximator (e.g., MLP) and the amount of synthetic data is sufficiently large.}

{\bf Choosing $\mathbf{P(X)}$.} The setting $P(X) = P'(X)$ corresponds to data being fair with respect to itself. For synthetic data generation, this setting is uninteresting as any dataset can be made fair by randomly sampling or removing $A$; if $A$ is random, the prediction should not directly or indirectly depend on it. This ignores, however, that a downstream user might use the trained model on a real-world dataset in which other variables $B$ are correlated with $A$, and thus their model (which is trained to use $B$ for predicting $Y$) will be biased. Of specific interest is the setting where $P(X)$ corresponds to the original data distribution $P_X(X)$ that contains unfairness. In this scenario, we construct $P'(X)$ by learning $P_X(X)$ and removing the unfair characteristics. The data from $P'(X)$ can be published online, and models trained on this data can be deployed fairly in real-life scenarios where data follows $P_X(X)$. Unless otherwise stated, henceforth, we assume $P(X)=P_X(X)$. 

\subsection{Graphical perspective}
\begin{wrapfigure}[12]{R}{3.1cm}
    \vspace{-0.3cm}
    \centering
    \vspace{-0.1cm}
    \includegraphics[width=2.3cm]{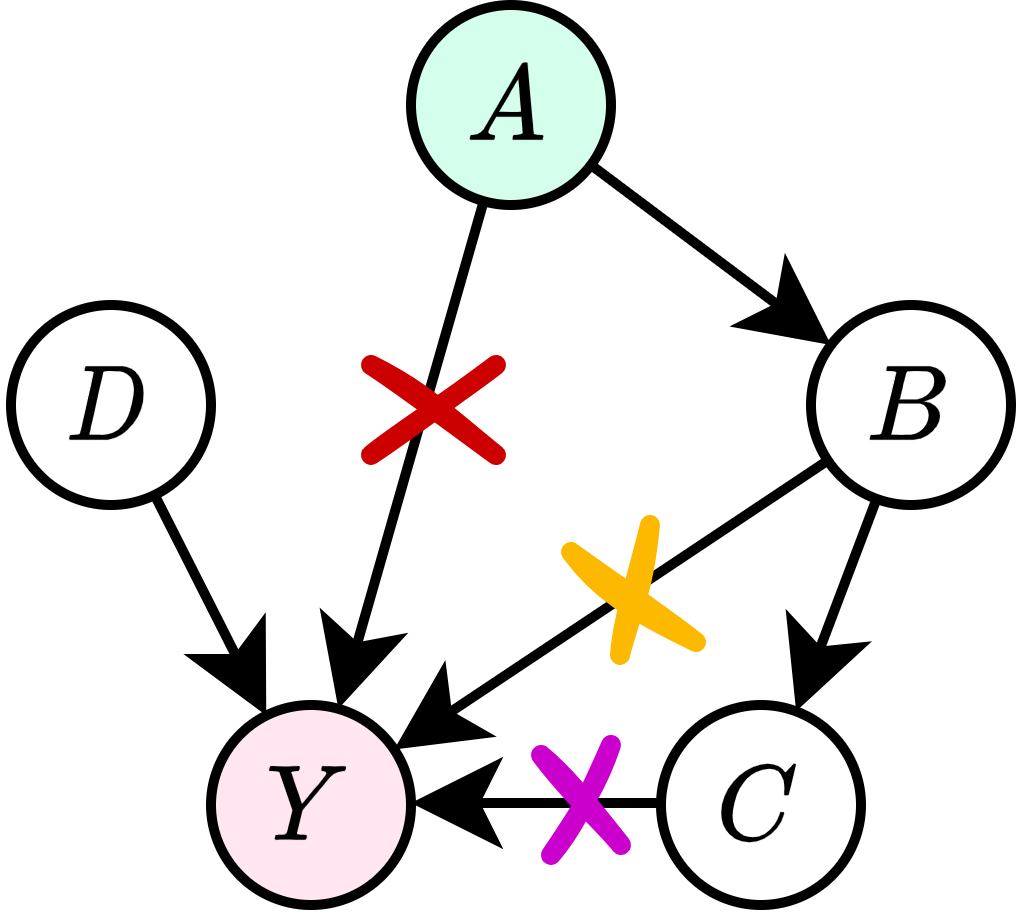}\label{fig:example_single}
    \vspace{-0.2cm}
    \caption[caption placeholder]{{\bf Edge removal for fairness.} FTU: \xmarkred; DP: \xmarkred\xmarkyellow\xmarkpink; CF when $R=C$: \xmarkred\xmarkyellow; CF when $B\in R$: \xmarkred.}
    \label{fig:example_simple}
    \vspace{-0.2cm}
    \rule{\linewidth}{.75pt}
\end{wrapfigure}
As reflected in the widely accepted terms direct versus indirect discrimination, it is natural to define distributional fairness from a causal standpoint. Let $\mathcal{G}'$ and $\mathcal{G}$ respectively denote the graphs underlying $P'(X)$ (the synthetic data distribution which we can control) and $P(X)$ (the evaluation distribution that we cannot control). Let $\partial_{\mathcal{G}}Y$ denote the Markov boundary of $Y$ in graph $\mathcal{G}$. We focus on the conditional fairness definition because it subsumes the definition of DP and FTU (Section \ref{sec:alg_fair}). Let $R\subset X$ be the set of explanatory features.
\begin{proposition} \label{proposition:cf_condition}
(CF: graphical condition) If for all $B\in \partial_{\mathcal{G}'}Y$, $A\indep_{\mathcal{G}} B|R$,\footnote{Where $\indep_{\mathcal{G}}$ denotes d-separation in $\mathcal{G}$. Here we define $A\indep_{\mathcal{G}} B|R$ to be true for all $B\in R$.} then distribution $P'(X)$ is CF fair w.r.t $P(X)$ given explanatory factors $R$.
\end{proposition}
\vspace{-0.3cm}
\begin{proof}
Without loss of generality, let us assume the label is binary.\footnote{If $Y$ is continuous the same result holds, though the ``optimal'' predictor will depend on the statistic of interest, e.g. mode, mean, median or the entire distribution $f(X,Y)\approx P(Y|X)$.} The optimal predictor $f^*(X) = P(Y|X) = P(Y|\partial_{\mathcal{G}'}Y)$. Thus, if $\partial_{\mathcal{G}'}Y$ is d-separated from $A$ in $\mathcal{G}$ given $R$, prediction $\hat{Y}=f^*(X)$ is independent of $A$ given $R$ and CF holds.
\end{proof}
\begin{corollary}
\label{cor:cf}
(CF debiasing) Any distribution $P'(X)$ with graph $\mathcal{G}'$ can be made CF fair w.r.t. $P(X)$ and explanatory features $R$ by removing from $\mathcal{G}'$ edges $\tilde E=\{(B\rightarrow Y)$ and $(Y\rightarrow B): \forall B\in \partial_{\mathcal{G}'}Y$ with $B\not \indep_{\mathcal{G}} A|R$\}.
\end{corollary}
\vspace{-0.3cm}
\begin{proof}
First note $\tilde{E}$ is the necessary and sufficient set of edges to remove for $(\forall B\in \partial_{\mathcal{G}'}Y$, $A\indep_{\mathcal{G}} B|R)$ to be true, subsequently the result follows from Proposition \ref{proposition:cf_condition}.
\end{proof}
For FTU (i.e. $R=X\backslash A$) and DP (i.e. $R=\emptyset$), this corollary simplifies to:
\begin{corollary}
\label{cor:ftu}
(FTU debiasing) Any distribution $P'(X)$ with graph $\mathcal{G}'$ can be made FTU fair w.r.t. any distribution $P(X)$ by removing, if present, i) the edge between $A$ and $Y$ and ii) the edge $A\rightarrow C$ or $Y\rightarrow C$ for all shared children $C$.
\end{corollary}
\begin{corollary}
\label{cor:dp}
(DP debiasing) Any distribution $P'(X)$ with graph $\mathcal{G}'$ can be made DP fair w.r.t. $P(X)$ by removing, if present, the edge between $B$ and $Y$ for any $B\in \partial_{\mathcal{G}'}Y$ with $B\not \indep_{\mathcal{G}} A$.
\end{corollary}
Figure \ref{fig:example_simple} shows how the different fairness definitions lead to different sets of edges to be removed.

{\bf Faithfulness.} Usually one assumes distributions are faithful w.r.t. their respective graphs, in which case the if-statement in Proposition \ref{proposition:cf_condition} become equivalence statements: fairness is \textit{only} possible when the graphical conditions hold.
\begin{theorem}
If $P(X)$ and $P'(X)$ are faithful with respect to their respective graphs $\mathcal{G}$ and $\mathcal{G}'$, then Proposition \ref{proposition:cf_condition} becomes an equivalence statement and Corollaries~\ref{cor:cf}, \ref{cor:ftu} and \ref{cor:dp} describe the necessary and sufficient sets of edges to remove for achieving CF, FTU and DP fairness, respectively.
\end{theorem}
\vspace{-0.3cm}
\begin{proof}
Faithfulness implies $A\indep_{P(X)} B|R\implies A\indep_{\mathcal{G}} B|R$, e.g. \cite{peters2017}. Thus, if $\exists B\in \partial_{\mathcal{G}'}Y$ for which $A\not\indep_{\mathcal{G}} B|R$, then $A\not\indep B|R$. Because $B\in \partial_{\mathcal{G}'}Y$ and $P'(X)$ is faithful to $\mathcal{G}'$, $\hat{Y}=f^*(X)$ depends on $B$, and thus $\hat{Y}\not\indep A|R$: CF does not hold.
\end{proof}
{\bf Other definitions.} Some authors define similar fairness measures in terms of directed paths (cf. d-separation) \citep{zhang2016causal,kilbertus2017avoiding,nabi2018fair}, which is a milder requirement as it allows correlation via non-causal paths. In Appendix \ref{appx:fairness_defs_continued} we highlight the graphical conditions for these definitions. 

\section{Method: DECAF} \label{sec:method}
The primary design goal of DECAF is to generate fair synthetic data from unfair data. We separate DECAF into two stages. The training stage learns the causal conditionals that are observed in the data through a causally-informed GAN. At the generation (inference) stage, we intervene on the learned conditionals via Corollaries \ref{cor:cf}-\ref{cor:dp}, in such a way that the generator creates fair data. We assume the underlying DGP's graph $\mathcal{G}$ is known; otherwise, $\mathcal{G}$ needs to be approximated first using any causal discovery method, see Section \ref{sec:experiments}.

\subsection{Training}
{\bf Overview.} This stage strives to learn the causal mechanisms $\{f_i(\pa(X_i),Z_i)\}$. Each structural equation $f_i$ (Eq. \ref{eq:SEM}) is modelled by a separate generator $G_i:\mathbb{R}^{|Pa(X_i)|+1}\rightarrow \mathbb{R}$. We achieve this by employing a conditional GAN framework with a causal generator. This process is illustrated in Figure~\ref{fig:architecture} and detailed below.

\begin{figure}[bt]
    \centering
    \includegraphics[width=\textwidth]{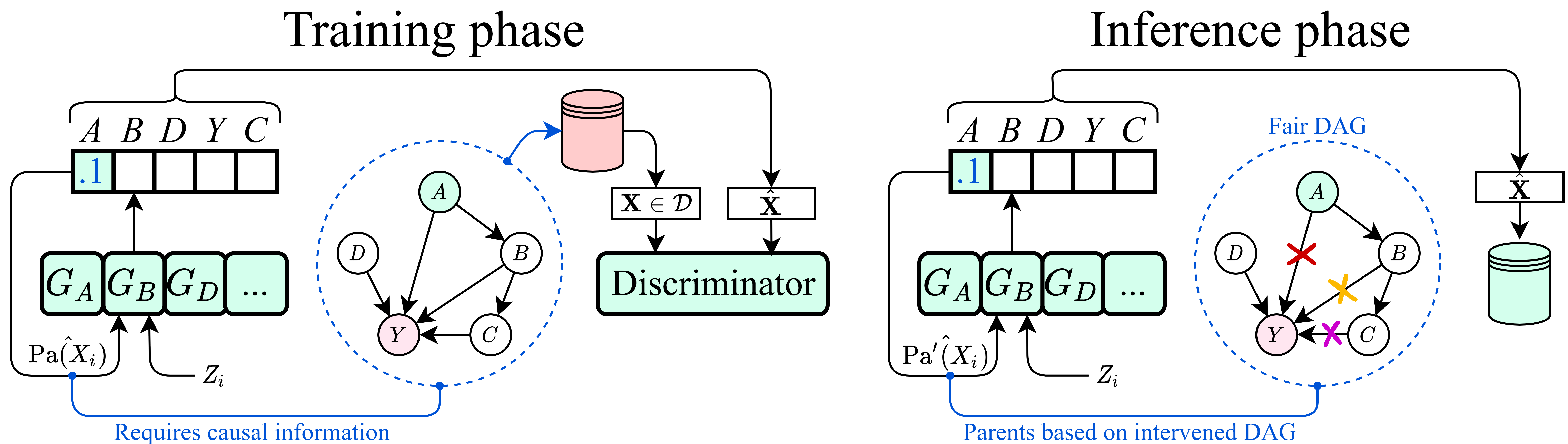}
    \vspace{-0.6cm}
    \caption{{\bf Architecture.} {\it Training phase---} Each component in $\hat{\mathbf{X}}$ is generated sequentially as a function (where the function is that component's generator $G_i$) of the component's parents. Parental knowledge is provided by the DAG governing the data. {\it Inference phase---} As the component-wise generation of the generator network is independent of the DAG governing the data, we can easily replace (or intervene on) the DAG governing parental information. The resulting synthetic data (right) will be governed by the intervened DAG. \textit{FTU is achieved by removing: \xmarkred; DP: \xmarkred\xmarkyellow\xmarkpink; e.g. CF when $R=C$: \xmarkred\xmarkyellow}.\hspace{12.2cm} }
    \label{fig:architecture}
    \vspace{-0.2cm}
    \rule{\linewidth}{.75pt}
    \vspace{-0.7cm}
\end{figure}


Features are generated sequentially following the topological ordering of the underlying causal DAG: first root nodes are generated, then their children (from generated causal parents), etc. Variable $\hat X_i$ is modelled by the associated generator $G_i$:
\begin{equation}
\label{eq:G_i}
    \hat{X}_i = G_i(\hat{\pa}(X_i), Z_i)\quad \forall i,
\end{equation}
where $\hat{\pa}(X_i)$ denotes the generated causal parents of $X_i$ (for root nodes the empty set), and each $Z_i$ is independently sampled from $P(Z)$ (e.g. standard Gaussian). We denote the full sequential generator by $G(Z) = [G_1(Z_1), ..., G_d(\cdot,Z_d)]$.

Subsequently, the synthetic sample $\hat{\mathbf{x}}$ is passed to a discriminator $D:\mathbb{R}^d\rightarrow \mathbb{R}$, which is trained to distinguish the generated samples from original samples. A typical minimax objective is employed for creating generated samples that confuse the discriminator most:
\begin{equation}
\label{eq:objective}
     \underset{\{G_i\}_{i=1}^d}{\max}\underset{D}{\min}\ \mathbb{E} [\log D(G(Z))+\log(1-D(X)],
\end{equation}
with $X$ sampled from the original data. We optimize the discriminator and generator iteratively and add a regularization loss to both networks. Network parameters are updated using gradient descent. 

If we assume $P_X(X)$ is compatible with graph $\mathcal{G}$, we can show that the sequential generator has the same theoretical convergence guarantees as standard GANs \cite{goodfellow2014generative}:
\begin{theorem}(Convergence guarantee) \label{theorem:conv}
Assuming the following three conditions hold:
\begin{enumerate}[(i)]
    \item data generating distribution $P_X$ is Markov compatible with a known DAG $\mathcal{G}$;
    \item generator $G$ and discriminator $D$ have enough capacity; and
    \item in every training step the discriminator is trained to optimality given fixed $G$, and $G$ is subsequently updated as to maximize the discriminator loss (Eq. \ref{eq:objective});
\end{enumerate}
then generator distribution $P_G$ converges to true data distribution $P_X$
\end{theorem}
\begin{proof}
See Appendix \ref{sec:theory}
\end{proof}
Condition (i), compatibility with $\mathcal{G}$, is a weaker assumption than assuming perfect causal knowledge. For example, suppose the Markov equivalence class of the true underlying DAG has been determined through causal discovery. In that case, any graph $\mathcal{G}$ in the equivalence class is compatible with the data and can thus be used for synthetic data generation. However, we note that debiasing can require the correct directionality for some definitions of fairness, see Discussion.

{\bf Remark.} The causal GAN we propose, DECAF, is simple and extendable to other generative methods, e.g., VAEs. Furthermore, from the post-processing theorem \cite{dwork2014algorithmic} it follows that DECAF can be directly used for generating \textit{private} synthetic data by replacing the standard discriminator by a differentially private discriminator \cite{dpgan,pate-gan}.

\subsection{Inference-time Debiasing} \label{sec:synthetic}
The training phase yields conditional generators $\{G_i\}_{i=1}^d$, which can be sequentially applied to generate data with the same output distribution as the original data (proof in Appendix \ref{sec:theory}). 
The causal model allows us to go one step further: when the original data has characteristics that we do not want to propagate to the synthetic data (e.g., gender bias), individual generators can be modified to remove these characteristics. Given the generator's graph  $\mathcal{G}=(X,E)$, fairness is achieved by removing edges such that the fairness criteria are met, see Section \ref{sec:defining_fairness}. Let $\tilde{E}\in E$ be the set of edges to remove for satisfying the required fairness definition. For CF, FTU and DP,\footnote{Just like in Corollaries \ref{cor:cf} and \ref{cor:dp}, we assume the downstream evaluation distribution is the same as the biased training data distribution: a predictor trained on the synthetic debiased data, is required to give fair predictions in real-life settings with distribution $P_X(X)$.} the sets $\tilde{E}$ are given by Corollaries~\ref{cor:cf}, \ref{cor:ftu} and \ref{cor:dp}, respectively.

Removing an edge constitutes to what we call a ``surrogate'' $do$-operation \cite{pearl2009} on the conditional distribution. For example, suppose we only want to remove $(i\rightarrow j)$. For a given sample, $X_i$ is generated normally (Eq. \ref{eq:G_i}), but $X_j$ is generated using the modified:
\begin{equation}
\label{eq:interven}
    \hat{X}_j^{do(X_i)=\tilde{x}_{ij}} = G_j(..., X_i=\tilde{x}_{ij}),
\end{equation}
where $X_i=\tilde{x}_{ij}$ is the surrogate parent assignment. Value $\hat{X}_j^{do(X_i)}$ can be interpreted as the counterfactual value of $\hat{X}_j$, had $X_i$ been equal to $\tilde{x}_{ij}$ (see also \cite{zhang2018fairness}).

Choosing the value of surrogate variable $\tilde{x}_{ij}$ requires background knowledge of the task and bias at hand. For example, surrogate variable $\tilde{x}_{ij}$ can be sampled independently from a distribution for each synthetic sample (e.g., the marginal $P(X_i)$), be set to a fixed value for all samples in the synthetic data (e.g., if $X_i$: gender, always set $\tilde{x}_{ij}=male$ when generating feature $X_j$: job opportunity) or be chosen as to maximize/minimize some feature (e.g. $\tilde{x}_{ij}=\arg\max_x \hat{X}_j^{do(X_i)=x}$). We emphasize that we do not set $X_i=\tilde{x}_{ij}$ in the synthetic sample; $X_i=\tilde{x}_{ij}$ is only used for substitution of the removed dependence. We provide more details in Appendix \ref{appx:surrogate}.

More generally, we create surrogate variables for all edges we remove, \{$\tilde{x}_{ij}: (i\rightarrow j)\in \tilde{E}\}$. Each sample is sequentially generated by Eq. \ref{eq:interven}, with a surrogate variable for each removed incoming edge. 

{\bf Remark.} Multiple datasets can be created based on different definitions of fairness and/or different downstream prediction targets. Because debiasing happens at inference-time, this does not require retraining the model.


\section{Experiments} \label{sec:experiments}

In this section, we validate the performance of DECAF for synthesizing bias-free data based on two datasets: i) real data with existing bias and ii) real data with synthetically injected bias. The aim of the former is to show that we can remove real, existing bias. The latter experiment provides a ground-truth unbiased target distribution, which means we can evaluate the quality of the synthetic dataset with respect to this ground truth. 
For example, when historically biased data is first debiased, a model trained on the synthetic data will likely create better predictions in contemporary, unbiased/less-biased settings than benchmarks that do not debias first.  

In both experiments, the ground-truth DAG is unknown. We use causal discovery to uncover the underlying DAG and show empirically that the performance is still good.

\textbf{Benchmarks.}
We compare DECAF against the following benchmark generative methods: 
a GAN, a Wasserstein GAN with gradient penalty (WGAN-GP) \cite{gulrajani2017improved} 
and FairGAN \cite{xu2018fairgan}.  FairGAN is the only benchmark designed to generate synthetic fair data,\footnote{The works of \cite{zhang2016causal,calmon2017optimized} are not applicable here, as these methods are constrained to discrete data.} whereas GAN and WGAN-GP 
only aim to match the original data's distribution, regardless of inherent underlying bias.  For these benchmarks, fair data can be generated by naively removing the protected variable -- we refer to these methods with the PR (protected removal) suffix and provide more experimental results and insight into PR in Appendix~\ref{appx:protected_removal}.  We benchmark DECAF debiasing in four ways: i) with \textit{no inference-time debiasing} (DECAF-ND), ii) under FTU (DECAF-FTU), iii) under CF (DECAF-CF) and iv) under DP fairness (DECAF-DP).  We provide DECAF\footnote{\texttt{PyTorch Lightning} source code at \url{https://github.com/vanderschaarlab/DECAF}.} implementation details in Appendix~\ref{appx:implementation_details}.

\textbf{Evaluation criteria.}
We evaluate DECAF using the following metrics:  
\vspace{-0.3cm}
\begin{itemize}
    \item \textbf{Data quality} is assessed using metrics of precision and recall \cite{sajjadi2018assessing,kynkaanniemi2019improved,precision-recall}.  Additionally, we evaluate all methods in terms of AUROC of predicting the target variable using a downstream classifier (MLP in these experiments) trained on synthetic data.
    \vspace{-0.15cm}
    \item \textbf{FTU} is measured by calculating the difference between the predictions of a downstream classifier for setting $A$ to 1 and 0, respectively, such that $|P_{A=0}(\hat{Y}|X) - P_{A=1}(\hat{Y}|X)|$, while keeping all other features the same. This difference measures the direct influence of $A$ on the prediction.
    \vspace{-0.15cm}
    \item \textbf{DP} is measured in terms of the \textit{Total Variation} \cite{zhang2018fairness}: the difference between the predictions of a downstream classifier in terms of positive to negative ratio between the different classes of protected variable $A$, i.e., $|P(\hat{Y}|A=0) - P(\hat{Y}|A = 1)|$.
    \vspace{-0.3cm}
\end{itemize}


\begin{table}[!t]
\caption{Bias removal experiment on the Adult dataset \cite{uci}. The full table with protected attribute removal can be found in Appendix~\ref{appx:protected_removal}. } 
\vspace{0.1cm}
\centering
\resizebox{0.9\linewidth}{!}{
\centering
\begin{tabular}{l|ccc|ccc}
    \toprule
     &\multicolumn{3}{c}{\textbf{Data Quality}}  &\multicolumn{2}{c}{\textbf{Fairness}}   \\ 
    \cmidrule{2-6}  
     \textbf{Method} & Precision$\uparrow$ & Recall$\uparrow$ & AUROC$\uparrow$ & FTU$\downarrow$ & DP$\downarrow$\\ 
     \midrule
Original data $\mathcal{D}$&$0.920\pm0.006$&$0.936\pm0.008$&$0.807\pm0.004$&$0.116\pm0.028$&$0.180\pm0.010$\\
GAN&$0.607\pm0.080$&$0.439\pm0.037$&$0.567\pm0.132$&$0.023\pm0.010$&$0.089\pm0.008$\\
WGAN-GP&$0.683\pm0.015$&$0.914\pm0.005$&$0.798\pm0.009$&$0.120\pm0.014$&$0.189\pm0.024$\\
FairGAN&$0.681\pm0.023$&$0.814\pm0.079$&$0.766\pm0.029$&$0.009\pm0.002$&$0.097\pm0.018$\\
DECAF-ND&$0.780\pm0.023$&$0.920\pm0.045$&$0.781\pm0.007$&$0.152\pm0.013$&$0.198\pm0.013$\\
DECAF-FTU&$0.763\pm0.033$&$0.925\pm0.040$&$0.765\pm0.010$&$0.004\pm0.004$&$0.054\pm0.005$\\
DECAF-CF&$0.743\pm0.022$&$0.875\pm0.038$&$0.769\pm0.004$&$0.003\pm0.006$&$0.039\pm0.011$\\
DECAF-DP&$0.781\pm0.018$&$0.881\pm0.050$&$0.672\pm0.014$&$0.001\pm0.002$&$0.001\pm0.001$\\

    \bottomrule
\end{tabular}
}
\vspace{-.4cm}
\label{table:adult}
\end{table}

\subsection{Debiasing Census Data}
\vspace{-0.2cm}
In this experiment, we are given a biased dataset $\mathcal{D}\sim P(X)$ and wish to create a synthetic (and debiased) dataset $\mathcal{D}'$, with which a downstream classifier can be trained and subsequently be rolled out in a setting with distribution $P(X)$. 
We experiment on the Adult dataset \cite{uci}, with known bias between \texttt{gender} and \texttt{income} \citep{feldman2015certifying,zhang2016causal}.  The Adult dataset contains over 65,000 samples and has 11 attributes, such as \texttt{age}, \texttt{education}, \texttt{gender}, \texttt{income}, among others. Following \citep{zhang2016causal}, we treat \texttt{gender} as the protected variable and use \texttt{income} as the binary target variable representing whether a person earns over \$50K or not. For DAG $\mathcal{G}$, we use the graph discovered and presented by \citet{zhang2016causal}. In Appendix~\ref{appx:census_details}, we specify edge removals for DECAF-DP, DECAF-CF, and DECAF-FTU. 

Synthetic data is generated using each benchmark method, after which a separate MLP is trained on each dataset for computing the metrics; see Appendix \ref{appx:census_details} for details. We repeat this experiment 10 times for each benchmark method and report the average in Table~\ref{table:adult}.
As shown, DECAF-ND (no debiasing) performs amongst the best methods in terms of data utility. Because the data utility in this experiment is measured with respect to the original (biased) dataset, we see that the methods DECAF-FTU, DECAF-CF, and DECAF-DP score lower than DECAF-ND because these methods distort the distribution -- with DECAF-DP distorting the label's conditional distribution most and thus scoring worst in terms of AUROC. Note also that a downstream user who is only focused on performance would choose the synthetic data from WGAN-GP or DECAF-ND, which are also the most biased methods. Thus, we see that there is a trade-off between fairness and data utility when the evaluation distribution $P(X)$ is the original biased data. 

\vspace{-0.2cm}
\subsection{Fair Credit Approval} \label{sec:exp_ii}
\vspace{-0.2cm}
In this experiment, direct bias, which was not previously present, is synthetically injected into a dataset $\mathcal{D}$ resulting in a biased dataset $\tilde{\mathcal{D}}$. We show how DECAF can remove the injected bias, resulting in dataset $\mathcal{D}'$ that can be used to train a downstream classifier. This is a relevant scenario if the training data $\tilde{D}$ does not follow real-world distribution $P(X)$, but instead a biased distribution $\tilde{P}(X)$ (due to, e.g., historical bias). In this case, we want downstream models trained on synthetic data $\mathcal{D}'$ to perform well on the real-world data $\mathcal{D}$ instead of $\tilde{\mathcal{D}}$. We show that DECAF is successful at removing the bias and how this results in higher data utility than benchmarks methods trained on $\tilde{D}$.

We use the Credit Approval dataset from \cite{uci}, with graph $\mathcal{G}$ as discovered by the causal discovery algorithm FGES \cite{fges} using \texttt{Tetrad} \cite{tetrad} (details in Appendix~\ref{appx:credit_approval}). We inject direct bias by decreasing the probability that a sample will have their credit approved based on the chosen $A$.\footnote{We let $A$ equal (anonymized) \texttt{ethnicity} \citep{bias-lending2009,bias-lending2012, bias-lending2018, bias-lending2019}, with randomly chosen $A=4$ as the disadvantaged population.} The \texttt{credit\_approval} for this population was synthetically denied (set to 0) with some bias probability $\beta$, adding a directed edge between label and protected attribute. 

In Figure~\ref{fig:lineplots}, we show the results of running our experiment 10 times over various bias probabilities $\beta$.  We benchmark against FairGAN, as it is the only benchmark designed for synthetic debiased data.  Note that in this case, the causal DAG has only one indirect biased edge between the protected variable (see Appendix~\ref{appx:additionalexps}), and thus DECAF-DP and DECAF-CF remove the same edges and are the same for this experiment. The plots show that DECAF-FTU and DECAF-DP have similar performance to FairGAN in terms of debiasing; however, all of the DECAF-* methods have significantly better data quality metrics: precision, recall, and AUROC. DECAF-DP is one of the best performers across all 5 of the evaluation metrics and has better DP performance under higher bias.  As expected, DECAF-ND (no debiasing) has the same data quality performance in terms of precision and recall as DECAF-FTU and DECAF-DP and has diminishing performance in terms of downstream AUROC, FTU, and DP as bias strength increases.  See Appendix~\ref{appx:additionalexps} for other benchmarks, and the same experiment under hidden confounding in Appendix~\ref{appx:hiddenconfounder}.


\begin{figure}
    \centering
    \vspace{-0.2cm}
    \begin{subfigure}[b]{0.194\textwidth}
         \centering
         \includegraphics[width=\textwidth]{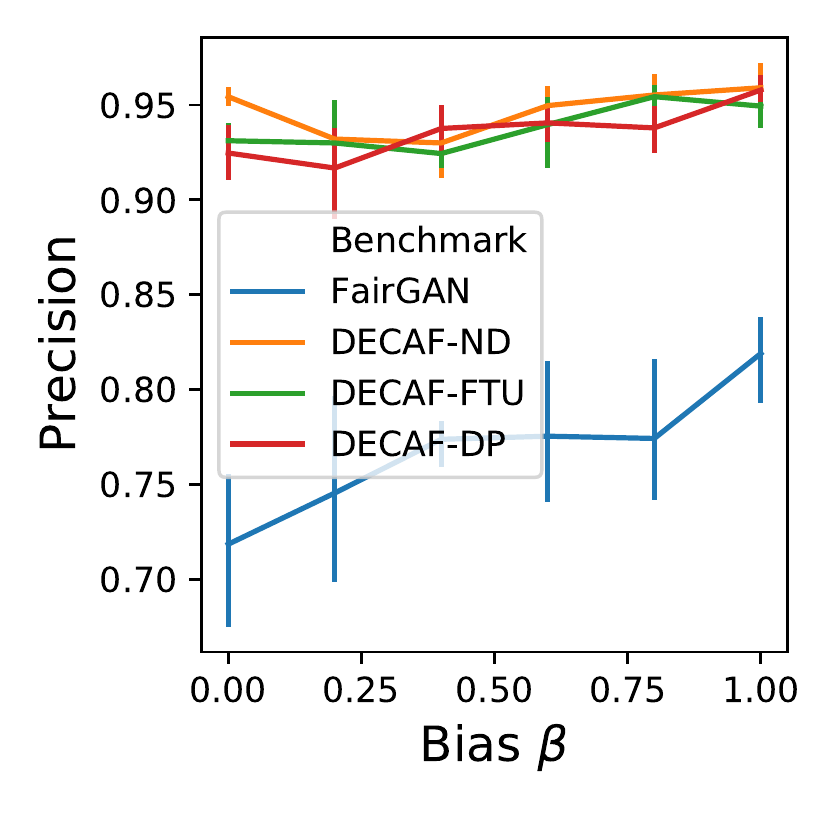}    \vspace{-0.7cm}
         \caption{Precision$\uparrow$}
         \label{fig:prec}
     \end{subfigure}
        \begin{subfigure}[b]{0.194\textwidth}
         \centering
         \includegraphics[width=\textwidth]{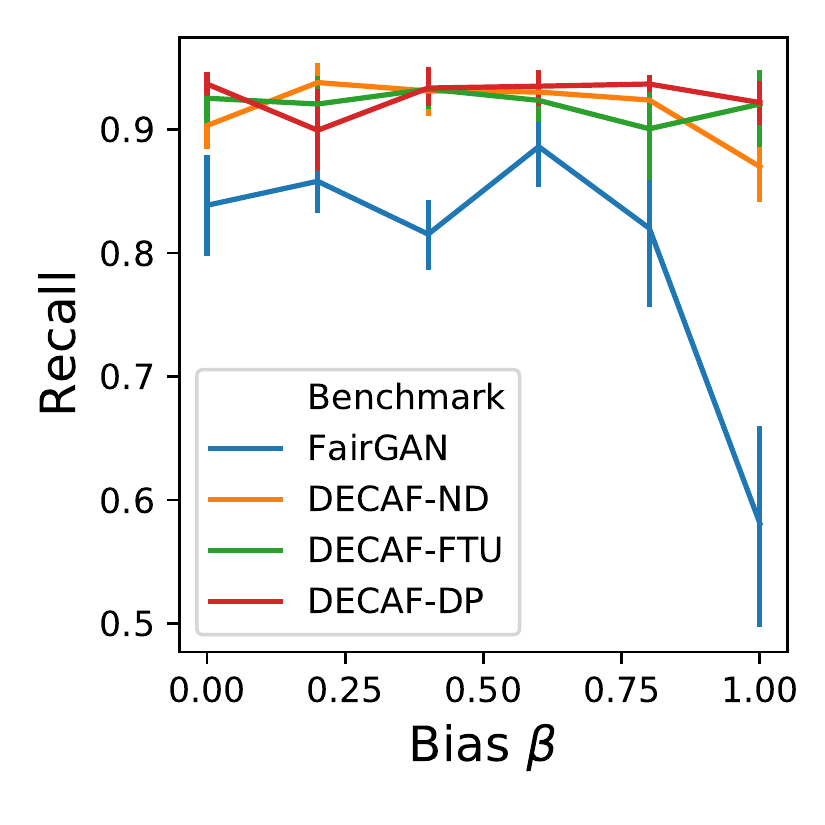}\vspace{-0.3cm}
         \caption{Recall$\uparrow$}
         \label{fig:rec}
     \end{subfigure}
    \begin{subfigure}[b]{0.194\textwidth}
         \centering
         \includegraphics[width=\textwidth]{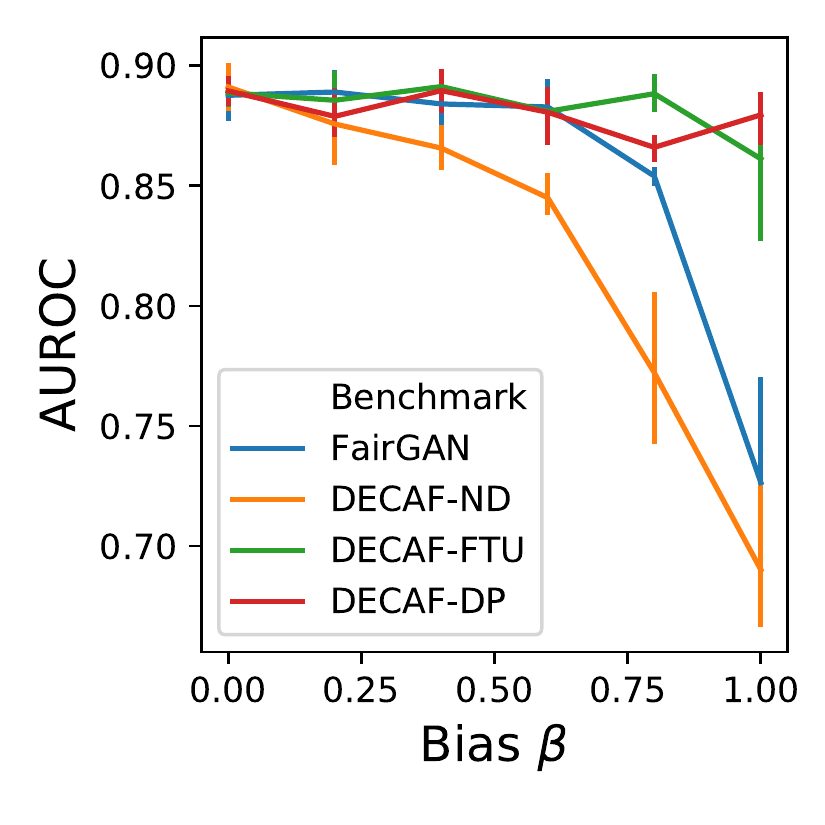}\vspace{-0.3cm}
         \caption{AUROC$\uparrow$}
         \label{fig:auc}
     \end{subfigure}
        \begin{subfigure}[b]{0.194\textwidth}
         \centering
         \includegraphics[width=\textwidth]{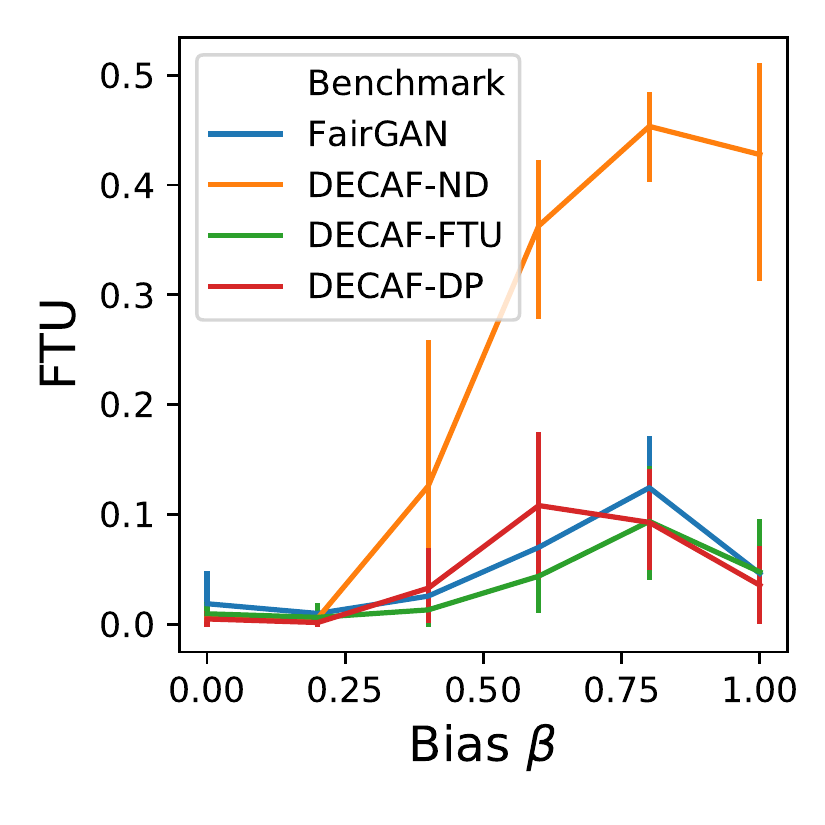}\vspace{-0.3cm}
         \caption{FTU$\downarrow$}
         \label{fig:ftu}
     \end{subfigure}    
    \begin{subfigure}[b]{0.194\textwidth}
         \centering
         \includegraphics[width=\textwidth]{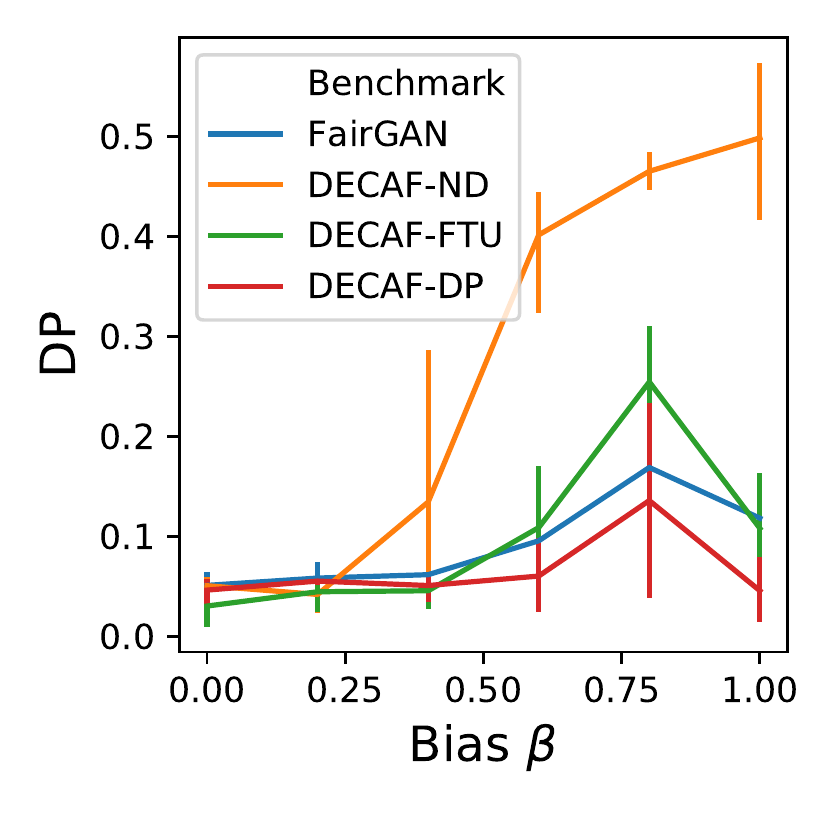}\vspace{-0.3cm}
         \caption{DP$\downarrow$}
         \label{fig:dp}
     \end{subfigure}  
     \vspace{-0.2cm}
    \caption{Plot of precision \textbf{(a)}, recall \textbf{(b)}, AUROC \textbf{(c)}, FTU \textbf{(d)}, and DP \textbf{(e)} over bias strength $\beta$.  FairGAN performs similarly in terms of DP and FTU, but DECAF-FTU and DECAF-DP have significantly better data quality as well as down stream prediction capability (AUROC).}
    \vspace{-0.5cm}
    \label{fig:lineplots}
\end{figure}

\vspace{-0.2cm}
\section{Discussion} \label{sec:discssuion}
\vspace{-0.2cm}
We have proposed DECAF, a causally-aware GAN that generates fair synthetic data. DECAF's sequential generation provides a natural way of removing these edges, with the advantage that the conditional generation of other features is left unaltered. We demonstrated on real datasets that the DECAF framework is both versatile and compatible with several popular definitions of fairness.  Lastly, we provided theoretical guarantees on the generator's convergence and fairness of downstream models. We next discuss limitations as well as applications and opportunities for future work.

\textbf{Definitions.} DECAF achieves fairness by removing edges between features, as we have shown for the popular FTU and DP definitions. Other independence-based \cite{barocas2017fairness} fairness definitions  can be achieved by DECAF too, as we show in Appendix \ref{appx:fairness_defs_continued}. Just like related debiasing works \cite{feldman2015certifying,calmon2017optimized,zhang2016causal,xu2018fairgan}, DECAF is not compatible with fairness definitions based on separation or sufficiency \cite{barocas2017fairness}, as these definitions depend on the downstream model more explicitly (e.g. Equality of Opportunity \cite{hardt2016equality}). More on this in Appendix \ref{appx:fairness_defs_continued}.

\textbf{Incorrect DAG specification.} Our method relies on the provision of causal structure in the form of a DAG for i) deciding the sequential order of feature generation and ii) deciding which edges to remove to achieve fairness. This graph need not be known a priori and can be discovered instead. If discovered, the DAG needs not equal the true DAG for many definitions of fairness, including FTU and DP, but only some (in)dependence statements are required to be correct (see Proposition 1). This is shown in the Experiments, where the DAG was discovered with the PC algorithm \cite{spirtes2000causation} and TETRAD \cite{tetrad}. Furthermore, in Appendix \ref{sec:theory} we prove that the causal generator converges to the right distribution for any graph that is Markov compatible with the data. We reiterate, however, that knowing (part of) the true graph is still helpful because i) it often leads to simpler functions $\{f_i\}_{i=1}^d$ to approximate,\footnote{Specifically, this is the case if modeling the causal direction is simpler than modeling the anti-causal direction. For many classes of models this is true when algorithmic independence holds, see \cite{peters2017}.} and ii) some causal fairness definitions do require correct directionality---see Appendix \ref{appx:fairness_defs_continued}. In Appendix \ref{appx:ablation}, we include an ablation study on how errors in the DAG specification affect data quality and downstream fairness.

\textbf{Causal sufficiency.} We have focused on just one type of graph: causally-sufficient directed graphs. Extending this to undirected or mixed graphs is possible as long as the generation order reflects a valid factorization of the observed distribution. This includes settings with hidden confounders. We note that for some definitions of bias, e.g., counterfactual bias, directionality is essential and hidden confounders would need to be corrected for (which is not generally possible).

\textbf{Time-series.} We have focused on the tabular domain. The method can be extended to other domains with causal interaction between features, e.g., time-series. Application to image data is non-trivial, partly because, in this instance, the protected attribute (e.g., skin color) does not correspond to a single observed feature. DECAF might be extended to this setting in the future by first constructing a graph in a disentangled latent space (e.g., \cite{kocaoglu2017causalgan,yang2021causalvae}).

\textbf{Social implications.} Fairness is task and context-dependent, requiring careful public debate. With that being said, DECAF empowers data issuers to take responsibility for downstream model fairness. We hope that this progresses the ubiquity of fairness in machine learning.


\section*{Acknowledgements}
We would like to thank the reviewers for their time and valuable feedback. This research was funded by the \textit{Office of Naval Research} and the \textit{WD Armstrong Trust}.

\bibliographystyle{unsrtnat}
\bibliography{neurips_main.bib}

\clearpage
\newpage

\appendix
\section{Protected variable removal} \label{appx:protected_removal}

A trivial method for satisfying FTU fairness, is to remove the protected attribute from downstream learners.  We first provide a motivating example explaining why this is sub-optimal.  We then follow this with an experiment on the Adult dataset.


\begin{wrapfigure}{R}{4.4cm}
\vspace{-0.1cm}
\flushright
\includegraphics[width=4.5cm]{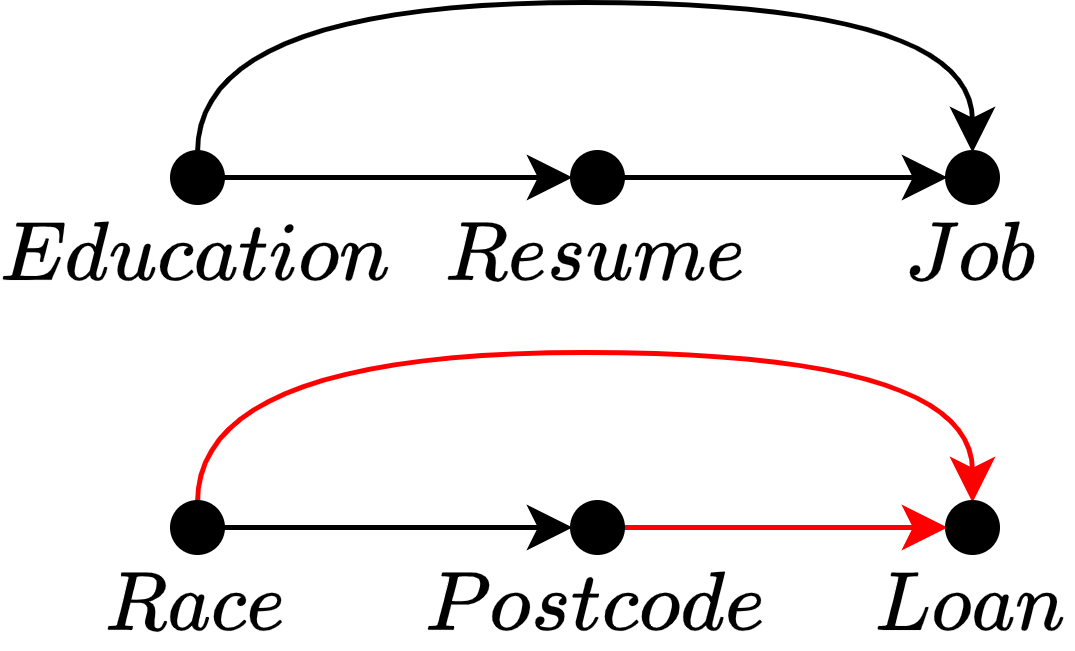}\label{fig:example_bias}
\caption{Human knowledge is essential for defining fairness.}
\label{fig:two_examples}
\vspace{-0.2cm}
\rule{\linewidth}{.75pt}
\end{wrapfigure}

\subsection{Example} Defining fairness is task and data dependent. For example, let us assume two datasets are generated by the graphical models in Figure \ref{fig:two_examples}. Data generated by the top graph is considered fair: $Education$ affects past experience ($Resume$), which together affect future job prospects ($Job$). The bottom graph is a historical example  of unfairness: even if there would be no bias between $Loan$ and $Race$, \textit{redlining} (i.e. the practice of refusing a loan to people living in certain areas) would discriminate indirectly based on race \citep{bias-lending2009,bias-lending2012, bias-lending2018, bias-lending2019}. Human knowledge is thus essential for defining fairness correctly, and making sure (e.g., historical) bias is not propagated by the models we deploy. This example also shows why simply removing or not measuring a sensitive attribute does not suffice: not only does this ignore indirect bias, but hiding the protected attribute leads to an (additional) correlation between $Postcode$ and $Loan$ due to confounding. A smart debiasing method is required that can distinguish fair from unfair relations.

\subsection{Experiment}
As explained in the previous example, simply removing the protected attribute is a naive and sub-optimal solution to FTU fairness due to confounding.  
Let us test this experimentally. We use the same experimental setup described in Section 6 for the Adult dataset, but we include additional metrics for protected attribute removal.  We denote protected attribute removal by the *-PR suffix. In Table~\ref{table:adult_withpr}, we observe that naively removing the protected attribute only ensures FTU fairness, as shown by: GAN-PR, WGAN-GP-PR, and DECAF-PR.  Furthermore, we observe that synthetic data quality diminishes as well for WGAN-GP-PR and DECAF-PR across precision, recall, and AUROC. For GAN-PR we see a slight improvement in data quality over GAN, however this improvement is very minimal in comparison to DECAF.

\begin{table}[!h]
\caption{Full table of bias removal experiment on Adult dataset \cite{uci} including protected removal (PR) metrics. For methods *-PR, we remove the protected attribute from the dataset before synthesizing data. $\ddagger$Note that the FTU values for the *-PR values will be zero since they are removed from the data generation method.} 

\centering
\resizebox{0.95\linewidth}{!}{
\centering
\begin{tabular}{l|ccc|ccc}
    \toprule
     &\multicolumn{3}{c}{\textbf{Data Quality}}  &\multicolumn{2}{c}{\textbf{Fairness}}   \\ 
    \cmidrule{2-6}  
     \textbf{Method} & Precision$\uparrow$ & Recall$\uparrow$ & AUROC$\uparrow$ & FTU$\downarrow$ & DP$\downarrow$\\ 
     \midrule
Original data $\mathcal{D}$&$0.920\pm0.006$&$0.936\pm0.008$&$0.807\pm0.004$&$0.116\pm0.028$&$0.180\pm0.010$\\
GAN&$0.607\pm0.080$&$0.439\pm0.037$&$0.567\pm0.132$&$0.023\pm0.010$&$0.089\pm0.008$\\
WGAN-GP&$0.683\pm0.015$&$0.914\pm0.005$&$0.798\pm0.009$&$0.120\pm0.014$&$0.189\pm0.024$\\
FairGAN&$0.681\pm0.023$&$0.814\pm0.079$&$0.766\pm0.029$&$0.009\pm0.002$&$0.097\pm0.018$\\
GAN-PR&$0.632\pm0.077$&$0.509\pm0.110$&$0.612\pm0.106$&$\ddagger0.0\pm0.0$&$0.120\pm0.012$\\
WGAN-GP-PR&$0.640\pm0.019$&$0.848\pm0.028$&$0.739\pm0.034$&$\ddagger0.0\pm0.0$&$0.078\pm0.014$\\
DECAF-PR&$0.717\pm0.021$&$0.839\pm0.033$&$0.769\pm0.020$&$\ddagger0.0\pm0.0$&$0.044\pm0.013$\\
DECAF-ND&$0.780\pm0.023$&$0.920\pm0.045$&$0.781\pm0.007$&$0.152\pm0.013$&$0.198\pm0.013$\\
DECAF-FTU&$0.763\pm0.033$&$0.925\pm0.040$&$0.765\pm0.010$&$0.004\pm0.004$&$0.054\pm0.005$\\
DECAF-CF&$0.743\pm0.022$&$0.875\pm0.038$&$0.769\pm0.004$&$0.003\pm0.006$&$0.039\pm0.011$\\
DECAF-DP&$0.781\pm0.018$&$0.881\pm0.050$&$0.672\pm0.014$&$0.001\pm0.002$&$0.001\pm0.001$\\


    \bottomrule
\end{tabular}
}

\label{table:adult_withpr}
\end{table}
\section{Convergence guarantees DECAF GAN} \label{sec:theory}
Assuming the correct underlying data generating DAG is known, well-known theoretical results for GANs transfer to DECAF. We highlight the main results.
The typical GAN minimax objective (Eq. 3 paper) is optimized by iteratively updating the discriminator and generator, with respective losses:
\begin{align}
    \mathcal{L}_D(\hat{X}, X) &= \log D(\hat{X}) + \log(1-D(X))\\
    \label{eq:loss_G}
    \mathcal{L}_G(\hat{X}) &= -\log D(\hat{X})
\end{align}

First, we reiterate the following theorem from \cite{goodfellow2014generative}. Let $P_G$ and $P_X$ denote generator and original data distributions, respectively. 
\begin{theorem} \label{theorem:goodfellow}
Given fixed optimal discriminator $D^*$, the global minimum of the generator loss (Eq. \ref{eq:loss_G}) is achieved if and only if $P_G = P_X$.
\end{theorem}
\begin{proof}
Noting that we have made no changes to the GAN discriminator, we refer to Theorem 1 of \cite{goodfellow2014generative}.
\end{proof}

\begin{theorem}(Convergence guarantee) \label{theorem:conv_app}
Assuming the following three conditions hold:
\begin{enumerate}[(i)]
    \item data generating distribution $P_X$ is Markov compatible with a known DAG $\mathcal{G}=(V,E)$;
    \item generator $G$ and discriminator $D$ have enough capacity; and
    \item in every training step the discriminator is trained to optimality given fixed $G$, and $G$ is subsequently updated as to maximise the discriminator loss (Eq. 3 paper);
\end{enumerate}
then generator distribution $P_G$ converges to true data distribution $P_X$
\end{theorem}

\begin{proof}
This is the direct result of the construction of generator $G$ and follows a similar argument as Proposition 2 of \cite{goodfellow2014generative}. Note that by the definition of compatibility of $P_X$ and $\mathcal{G}=(V,E)$, we can write:
\begin{equation*}
    P_X(X) = \prod_{X_i\in V} P(X_i|\{X_j:(X_j\rightarrow X_i)\in E\}) 
\end{equation*}
Given each $G_i$ (see Eq. 2 paper) has enough capacity, $G$ can thus express the full distribution $P_X(X)$. 
By convexity of the loss functions and the existence of a unique global optimum (Theorem \ref{theorem:goodfellow}), gradient descent is theoretically guaranteed to converge, $P_G\rightarrow P_X$ \cite{goodfellow2014generative}. 
\end{proof}

Note that for condition (i) of Theorem \ref{theorem:conv_app} to be valid, we do not require that graph $\mathcal{G}$ equals the true underlying DAG of the data generating distribution $P_X$; $P_G$ is only required to disentangle into the causal factors implied by $\mathcal{G}$. This is highly beneficial, as it enables generation of perfect synthetic data without perfect causal knowledge. For example, if the Markov equivalence class of the true underlying DAG has been determined through causal discovery, any graph $\mathcal{G}$ in the equivalence class satisfies condition (i) of Theorem \ref{theorem:conv_app}.

{\bf Remarks} The convergence guarantees do not necessarily hold in practice. First, finite data means there there is no guarantee the algorithm converges to the true underlying data distribution instead of, for example, the observed empirical data distribution. Second, in practice each generator $G_i$ will have limited capacity and $P(X_i|\pa(X_i))$ might not lie in its support. On a more positive note, these limitations are not specific for DECAF and generally GANs have done well in the past. Additionally, our method is directly extendable to the more stable WGAN-GP \cite{gulrajani2017improved} and other generative models. 

\section{Compatibility different fairness definitions} \label{appx:fairness_defs_continued}
{\bf Related definitions} In the paper we have discussed FTU, DP and CF, which are independence-based definitions and do not take directionality explicitly into account when defining fairness. Some authors use similar definitions, but instead of looking at (conditional) independencies of $A$ and $Y$, they consider (blocked) directed paths from protected attribute $A$ to $Y$. These definitions are compatible with DECAF, but mean less edges need to be removed. See Table \ref{tab:other_defs} and Figure \ref{fig:diff_defs}. \citet{zhang2016causal} consider direct and indirect discrimination, which can be understood as the ``directed path'' equivalents of FTU and DP.\footnote{Note: the legal definitions of direct and indirect discrimination are in fact defined as FTU and DP.} Assuming faithfulness and not allowing any discrimination---i.e. $\tau=0$ in \cite{zhang2016causal}---direct and indirect discrimination prohibit the existence of edge $A\rightarrow Y$ and directed path $A$ to $Y$, respectively. \citet{zhang2018fairness} disentangle the total effect of $A$ on $Y$ into direct, indirect and spurious relations. This leads to the same definition for direct discrimination as \citep{zhang2016causal}, but a different definition of indirect discrimination as it \emph{does} allow for direct influence of $A$ on $Y$. A very similar definition, coined counterfactual fairness, is proposed by \cite{kusner2017counterfactual}. \citet{kilbertus2017avoiding} introduce \textit{unresolved discrimination} (UD) as the path-equivalent version of conditional fairness. They define \textit{proxy} discrimination as well, which can be considered the dual of UD \citep{kilbertus2017avoiding}. 

\textbf{\bf Incompatible definitions} Some definitions are not compatible with fair synthetic data generation because they rely on the final prediction, e.g. equality of opportunity \cite{hardt2016equality} and calibration (e.g. see \cite{barocas2017fairness}). As a consequence, DECAF cannot be used for these. Furthermore, we note that all our fairness definitions are binary: a distribution is fair or unfair. In practice some level of unfairness might be tolerated. For example, the US Supreme Court's 80\% rule \cite{alessandra1988doctrines} essentially states that a prediction has disparate impact if for disadvantaged group $A=1$ and positive outcome $\hat{Y}=1$, $\frac{P(\hat{Y}=1|A=1)}{P(\hat{Y}=1|A=0)}<0.8$ \citep{feldman2015certifying}. Some authors (e.g. \citet{feldman2015certifying}) have explored this continuous definition, but because it requires quantification of path-specific effects work is limited by a linearity assumption. Extending this to nonlinear path-specific effects is an interesting direction for future work, with great relevance for real-life applications.

\begin{table}[!htbp]
    \centering
    \caption{Different definitions of fairness that are compatible with DECAF and which edges need removal when evaluation distribution $P(X)=P_X(X)$. The first three definitions are non-causal, the others only prohibit causal paths. $A,Y,P, R$ denote respectively the protected attribute, label, proxy variables and explanatory variables. Let $\pi_{A\rightarrow Y}$ denote a directed path from $A$ to $Y$ that ends with $B\rightarrow Y$ for some $B$.}
    \label{tab:other_defs}
    \scalebox{1}{
    \begin{tabularx}{\textwidth}{l l}
    \toprule
        {\bf Definition}& Edges to remove \\ \midrule
    Demographic Parity (DP) \cite{zafar2017fairness}                                  & $B\leftrightarrow Y: \forall B\in Bl_{\mathcal{G}'}(Y)$ with $A\not\indep B$\\
    Conditional Fairness (CF)  & $B\leftrightarrow Y: \forall B\in Bl_{\mathcal{G}'}(Y)$ with $A\not\indep B|R$\\
    Fairness through Unawareness (FTU)                                               & $A\leftrightarrow Y$ and ($A\rightarrow C$ or $Y\rightarrow C: \forall C$ with $A\rightarrow C \leftarrow Y$)\\
    \midrule 
    No Indirect Discrim. ($\neg$ ID) \cite{zhang2016causal}                    & $B\rightarrow Y$ if there exists $\pi_{A\rightarrow Y}$ \\
    No Proxy Discrim. ($\neg$PD) \cite{kilbertus2017avoiding}                  & $B\rightarrow Y$ if there exists $\pi_{A\rightarrow Y}$ that is blocked by $P$\\
    No Unresolved Discrim. ($\neg$UD) \cite{kilbertus2017avoiding}             & $B\rightarrow Y$ if there exists $\pi_{A\rightarrow Y}$ that is not blocked by $R$ \\
    No Direct Discrim. ($\neg$ DD) \cite{zhang2016causal,zhang2018fairness}                      & $A\rightarrow Y$\\

    \bottomrule
    \end{tabularx}}
\end{table}

\begin{figure}[h]
    \centering
    \includegraphics[width=0.8\textwidth]{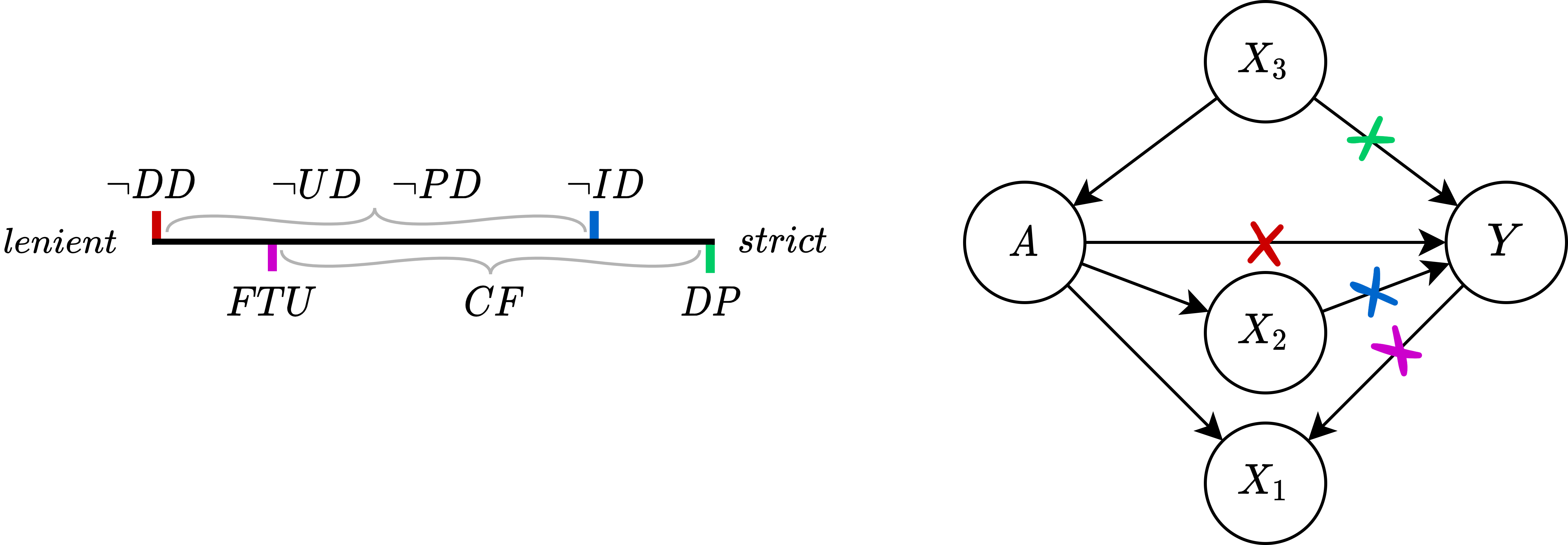}
    \caption{(Left) Typical strictness of different definitions. Note that the strictness of CF, $\neg$ UD and $\neg PD$ depends on the choice of explanatory variables/proxies. (Right) Example showing different definitions and required edge removals. $\neg$ DD: \xmarkred; FTU: \xmarkred\xmarkpink; $\neg$ ID: \xmarkred \xmarkblue; DP: \xmarkred \xmarkpink\xmarkblue\xmarkgreen. Note that for FTU, $A\rightarrow X_1$ could have been removed instead of $Y\rightarrow X_1$.}
    \label{fig:diff_defs}
\end{figure}


\section{Additional Details and Results}\label{appx:additionalexps}

\subsection{Implementation details.} \label{appx:implementation_details}
We instantiate the generator of DECAF with $d$ sub-networks with shared hidden layers.  Both the generator and discriminator are constructed having 2 hidden layers with $2d$ neurons and initialized with random uniform weights.  Each benchmark is initialized with the same random weights and published hyperparameters.  For preprocessing, all continuous variables are standardized.  We use the Adam optimizer with a learning rate of $0.001$ for up to 50 epochs.  We update the generator once for every 10 discriminator updates.  We implement DECAF using \texttt{PyTorch Lightning}\footnote{Source code is available at \url{https://github.com/vanderschaarlab/DECAF}}. 

\paragraph{Computational hardware.}  All models were trained on an Ubuntu 18.04 OS with 64GB of RAM (Intel Core i7-6850K CPU @ 3.60GHz) and 2 NVidia 1080 Ti GPUs. 

\textbf{Scalability} Due to the sequential feature generation, DECAF's run time scales linearly with the number of variables. In practice---for the larger Communities and Crime dataset---this comes down to an average training time of just about 35s per epoch when run on a machine with hexacore Intel i7-6850K CPU. Practical improvements can be made to speed this up further: when the graph is sparse one can parallelize calculations and often one can cluster (some) variables and model clusters together using a single generator network.  

\paragraph{Generating discrete variables} In both datasets the only non-binary discrete variable is the protected attribute, which for simplicity we have binarised (discriminated vs non-discriminated). All variables are generated in the same way, but binary variables are rounded off after generation. 

\begin{table}[hbt]
\centering
\caption{Overview datasets}
\begin{tabular}{llll} \toprule
                           & Credit & Census & Communities \\ \midrule
Number of features         & 15     & 10     & 128         \\ 
- Continuous               & 3      & 4      & 120         \\
- Discrete                 & 12     & 6      & 8           \\
Target type                & Binary & Binary & Binary      \\
Number of samples          & 379    & 32,561 & 1994        \\
Number of discovered edges & 40     & 22     & 1288       \\ \bottomrule
\end{tabular}
\end{table}

\subsection{Census Dataset Details} \label{appx:census_details}

DECAF supports both FTU and DP debiasing, i.e. respectively direct and indirect discrimination removal. We use the DAG from \cite{feldman2015certifying,zhang2016causal} as shown in Figure~\ref{fig:censusDAG}. FTU is achieved by removing the directed edge between between \texttt{sex} and \texttt{income} (see Corollary 3), DP is achieved by removing\footnote{Specifically, we focus on the scenario of $P(X)$ being the original biased data distribution; we want a model trained on synthetic data $\mathcal{D}\sim P'(X)$ to be DP-fair when evaluated on $P(X)$, see remark Section 4.2.} all incoming edges into the target variable that have the protected variable as an ancestor (Corollary 2)- these include edges between the target variable \texttt{income} and each of \texttt{occupation}, \texttt{hours\_per\_week}, \texttt{occupation}, \texttt{workclass}, \texttt{education}, \texttt{relationship}, \texttt{marital\_status}, and \texttt{sex}. DP fairness is overly strict, so to satisfy CF fairness, we allow the variables \texttt{occupation}, \texttt{hours\_per\_week}, \texttt{workclass}, and \texttt{education} while removing the edges from \texttt{sex}, \texttt{marital\_status}, and \texttt{relationship}. 

We generate synthetic data from the ground truth dataset using each benchmark generator.  We randomly hold out a sample of 2000 samples as a test set.
We train an MLP using default \texttt{scikit-learn} hyperparameters on the generated dataset to use as our downstream classifier.  
We use a hidden layer with 100 neurons and ReLU activation functions.  For the output layer we use a softmax activation and binary cross entropy loss. We use Adam as the optimizer with a learning rate of 0.001.

\begin{figure}[!htbp]
    \centering
    \includegraphics[width=0.9\linewidth]{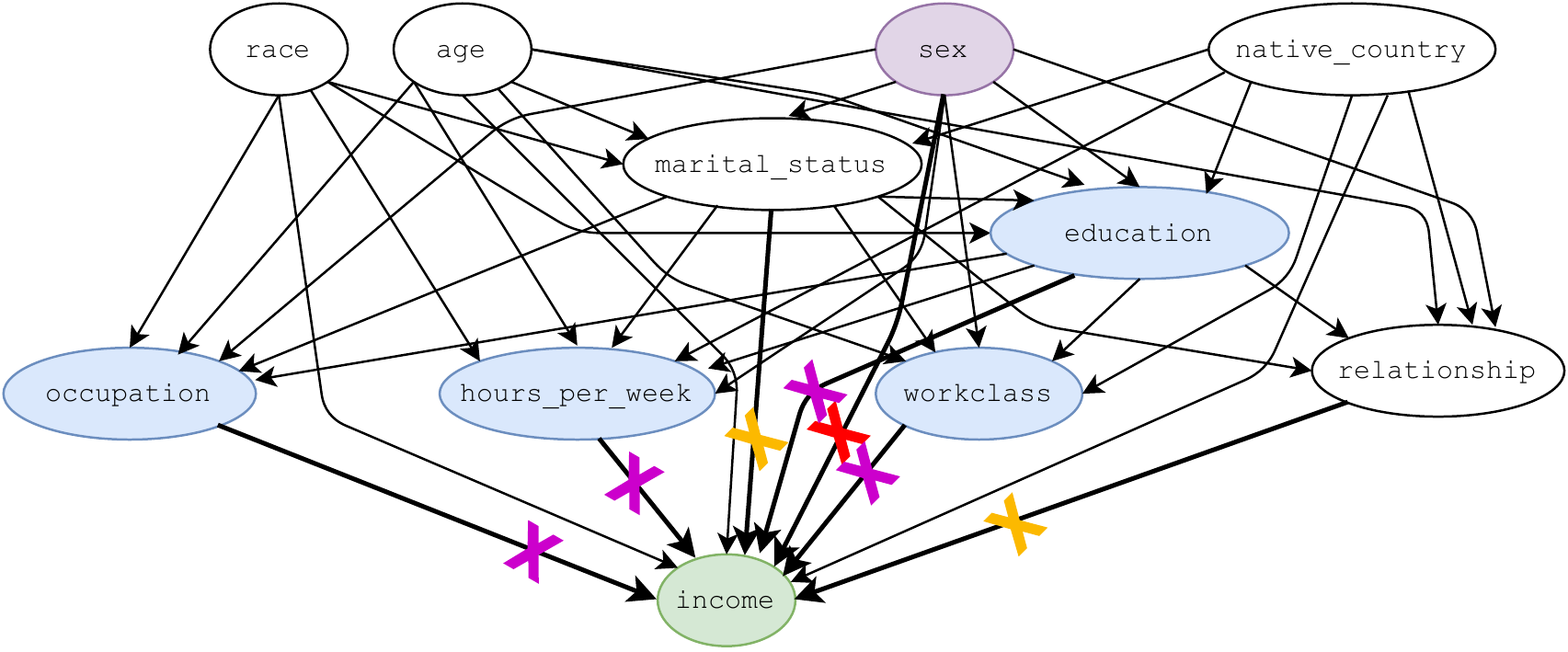}
    \caption{Adult dataset DAG from \cite{feldman2015certifying,zhang2016causal}.  
    The target variable is in green, the protected attribute in purple, and the allowed CF variables in blue. \textit{FTU is achieved by removing: \xmarkred; DP: \xmarkred\xmarkyellow\xmarkpink; CF: \xmarkred\xmarkyellow}.
    In this particular instance, we follow \cite{xu2018fairgan}, and remove gender discrimination.  However, our method generalizes to removing the highly problematic variable \texttt{race} to \texttt{income}.}
    \label{fig:censusDAG}
\end{figure}

\subsection{Fair Credit Details}\label{appx:credit_approval}

We use the Credit Approval dataset from \cite{uci} as our GT dataset.  
We synthetically add bias by decreasing the probability that a sample will be have their credit approved based on the chosen $A$.  We induce bias by choosing $A$ to be \texttt{ethnicity} \citep{bias-lending2009,bias-lending2012, bias-lending2018, bias-lending2019}, with a discriminated population having a value of 4\footnote{Note that the values have been anonymized in this dataset.}. The \texttt{credit\_approval} for this population was synthetically denied (set to 0) with some bias probability $\beta$ -- see Section~6.2 for more details.

The causal DAG used in this experiment is shown in Figure~\ref{fig:creditDAG}.  This DAG was found using the Fast Greedy Equivalence Search (FGES) \cite{fges} with the \texttt{pycausal} library \cite{tetrad}.  We provide the prior knowledge that \texttt{age} and \texttt{ethnicity} are root nodes to the FGES algorithm.  

We train an MLP using default \texttt{scikit-learn} hyperparameters on the generated dataset to use as our downstream classifier.  
We use a hidden layer with 100 neurons and ReLU activation functions.  For the output layer we use a softmax activation and binary cross entropy loss. We use Adam as the optimizer with a learning rate of 0.001.

In Table~\ref{table:bias_removal}, we show the results of running this experiment 10 times over our biased dataset. Note that our method was able to generate synthetic examples that had the highest AUROC (demonstrating FTU fairness). Table ~\ref{table:bias_removal} shows that our method can perform debiasing without performance hits to the synthetic data metrics -- i.e., there are no significant difference (outside of a standard deviation) between the top methods.

\begin{figure}[!htbp]
    \centering
    \includegraphics[width=0.9\linewidth]{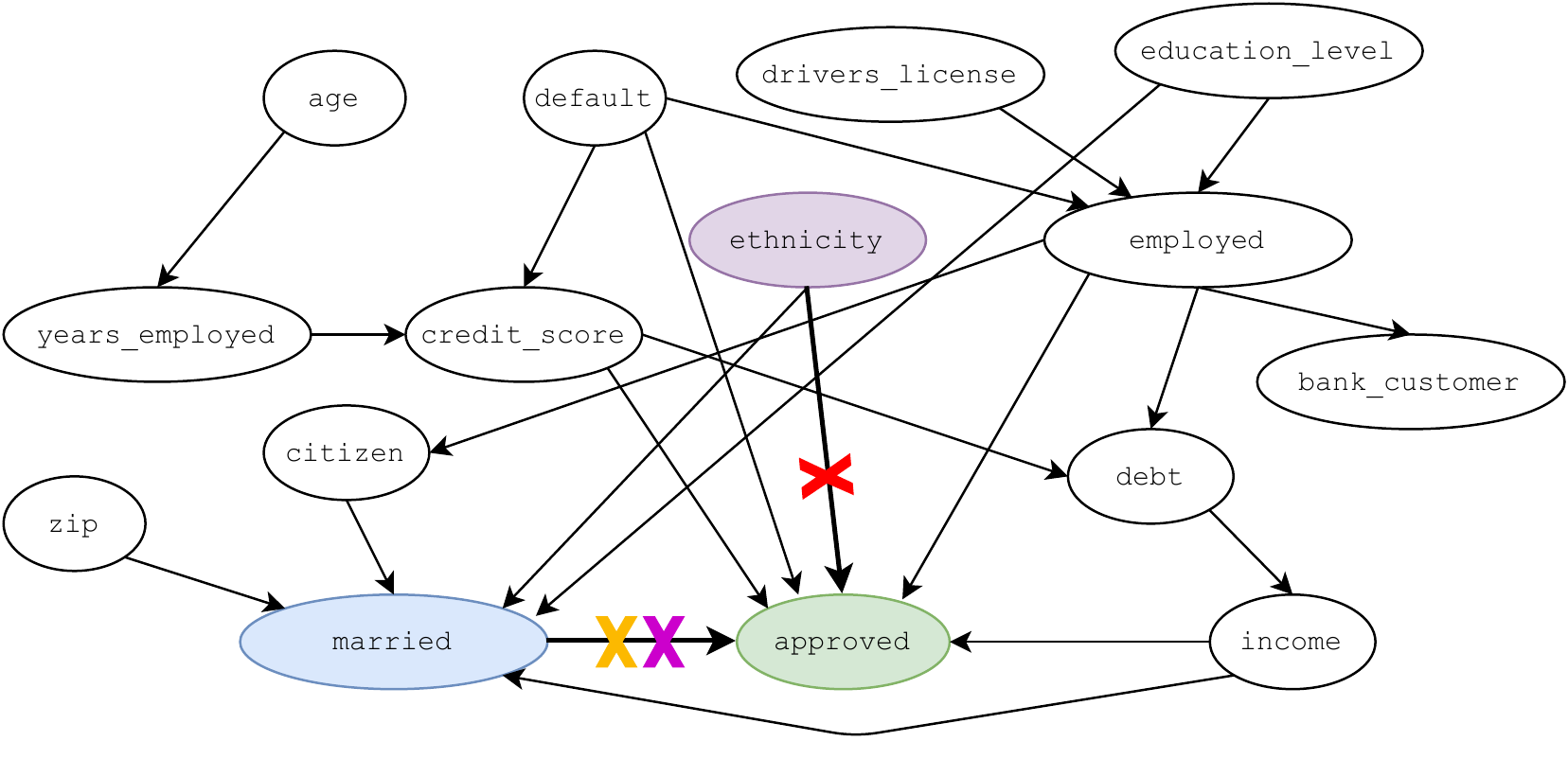}
    \caption{Credit Approval DAG discovered using FGES \cite{fges} and \texttt{Tetrad} \cite{tetrad}.  The target variable is in green, the protected attribute in purple, and the allowed CF variables in blue. \textit{FTU is achieved by removing: \xmarkred; DP: \xmarkred\xmarkyellow\xmarkpink; CF: \xmarkred\xmarkyellow}. Also, note that in this case CF fairness and DP fairness are the same.}
    \label{fig:creditDAG}
\end{figure}

\begin{table}[t]
\caption{Bias removal experiment on Credit Approval dataset.  Here we train an MLP on the listed dataset, and report the testing AUROC for credit approval prediction on the ground truth (GT) dataset for the biased population. Methods denoted *-PR represent modifications to the dataset by dropping the protected variable (PR). Note that there the FTU is zero for *-PR methods since the protected variable, P, has been removed.  
} 
\centering
\resizebox{1.0\linewidth}{!}{
\centering
\begin{tabular}{l|ccc|cc}
    \toprule
    &\multicolumn{3}{c}{\textbf{Data Quality}}  &\multicolumn{2}{c}{\textbf{Fairness}}  \\ 
    \cmidrule{2-6} 
     \textbf{Method} & Precision$\uparrow$ & Recall$\uparrow$& AUROC$\uparrow$ & DP$\downarrow$ & FTU$\downarrow$ \\ 
     \midrule
GAN&$0.921\pm0.036$&$0.335\pm0.029$&$0.743\pm0.047$&$0.405\pm0.077$&$0.194\pm0.058$\\
WGAN&$0.970\pm0.007$&$0.804\pm0.057$&$0.698\pm0.009$&$0.520\pm0.036$&$0.461\pm0.029$\\
ADSGAN&$0.963\pm0.009$&$0.841\pm0.052$&$0.708\pm0.009$&$0.506\pm0.013$&$0.429\pm0.059$\\

GAN-PR&$0.794\pm0.117$&$0.368\pm0.080$&$0.727\pm0.047$&$0.203\pm0.196$&$0.0\pm0.0$\\
WGAN-PR&$0.941\pm0.004$&$0.880\pm0.017$&$0.814\pm0.019$&$0.406\pm0.022$&$0.0\pm0.0$\\
ADSGAN-PR&$0.945\pm0.008$&$0.880\pm0.019$&$0.827\pm0.008$&$0.413\pm0.029$&$0.0\pm0.0$\\
FairGAN&$0.951\pm0.012$&$0.663\pm0.046$&$0.680\pm0.008$&$0.510\pm0.075$&$0.474\pm0.054$\\
DECAF&$0.954\pm0.012$&$0.601\pm0.015$&$0.713\pm0.045$&$0.511\pm0.130$&$0.432\pm0.127$\\
DECAF-FTU&$0.936\pm0.017$&$0.901\pm0.034$&$0.877\pm0.009$&$0.099\pm0.065$&$0.014\pm0.012$\\
DECAF-DP&$0.940\pm0.007$&$0.922\pm0.024$&$0.875\pm0.010$&$0.011\pm0.029$&$0.015\pm0.017$\\
    \bottomrule
\end{tabular}
}
\label{table:bias_removal}
\end{table}

\section{Surrogate variables} \label{appx:surrogate}

Debiasing in DECAF relies on removing edges from a trained model. As highlighted in Section 5.2, we need surrogate variables with which to replace the removed edges (Eq. 4 paper). In this section, we compare two surrogate variable mechanisms. The aim is show i) that debiasing is successful independent of the choice of surrogate variables, and ii) how prior knowledge helps in choosing surrogate variable mechanism, which leads to better data quality.

{\bf Mechanisms} Let $\tilde{X}_{ij}$ denote the surrogate variable used for the removed edge $(i\rightarrow j)$, i.e. the surrogate variable that replaces the influence of $X_i$ on $X_j$.  Here, we compare two surrogate mechanisms for this setting: 
\begin{enumerate}
    \item $\tilde{X}_{ij}\sim P(X_i)$, i.e. we sample from the parent's marginal distribution,
    \item $\tilde{X}_{ij} = \tilde{x}_{ij}$, where $\tilde{x}_{ij}$ is a fixed value. 
\end{enumerate}

Mechanism 1 is straightforward and most applicable when one does not know anything about the bias of a particular edge. By sampling from the marginal, each sample might use a different value of $\tilde{X}_{ij}$ when generating feature $X_j$, which means the diversity of the generated $X_j$ is retained better compared to mechanism 2. Mechanism 1 for all experiments in Section 6.

On the other hand, mechanism 2 is more suitable when we know explicitly that there is bias for some values of $X_i$, e.g. if $X_i$ is the protected attribute we might know there is a group $A=0$ that is being discriminated. In this case, sampling $\tilde{X}_{ij}$ from the marginal of $A$ is not desired: even though this means we remove direct bias from $A$ to $Y$, it still means we disadvantage some individuals randomly, i.e. every time we sample $\tilde{x}_{ij}=0$. We can employ the second mechanism instead, i.e. set $\tilde{x}_{ij}=1$ for all individuals. This corresponds to treating everyone like they are from the advantaged group.

{\bf Experiments} We repeat the experiment from Section 6.2, in which we insert direct bias from $A$ to $Y$ by denying loans for a disadvantaged group $A=0$ with probability $\beta$. Our aim is to remove the direct bias from $A$ to $Y$ and we evaluate the synthetic data quality and bias with respect to the original, unbiased dataset. As we will see, in this setting mechanism 2 is more appropriate: we want to treat everyone from group $A=0$ like they are from group $A=1$, thereby removing the bias we inserted. Meanwhile, we do not want to change the way we generate data for the advantaged group. More specifically, even though it would not be considered discrimination against a protected group, randomly denying loans to individuals of any group should still be considered unfair.

In Figure \ref{fig:appx_surrogate} we plot the quality metrics and FTU for three generation methods: DECAF-ND (no debiasing), DECAF-FTU1 (DECAF-FTU with surrogate mechanism 1) and DECAF-FTU2 (DECAF-FTU with mechanism 2). We plot three columns; on the left we plot the metrics for all generated data, in the middle we plot the metrics as computed on the discriminated group and on the right for the non-discriminated group. 

As we can see in the FTU plots (bottom), both debiasing mechanisms are equally valid for removing the injected bias from $A$ to $Y$. However, the precision metric tells a different story. Mechanism 1 disadvantages individuals randomly whenever it samples $\tilde{x}_{ij}=0$, but this is not in line with what we want the data to be like (no disadvantage like this at all). As a result, we see that the quality of both the discriminated group goes down. The same result can be observed in the recall and AUROC plot, though the overlapping error bars prohibit strong conclusions.

In a nutshell, these results indicate that for different mechanisms for surrogate variables, data fairness is guaranteed. However, knowledge about the origins of the bias can help increase the data utility.

\begin{figure}[hbt]
    \centering
    \begin{subfigure}[b]{0.31\textwidth}
         \centering
         \includegraphics[width=\textwidth]{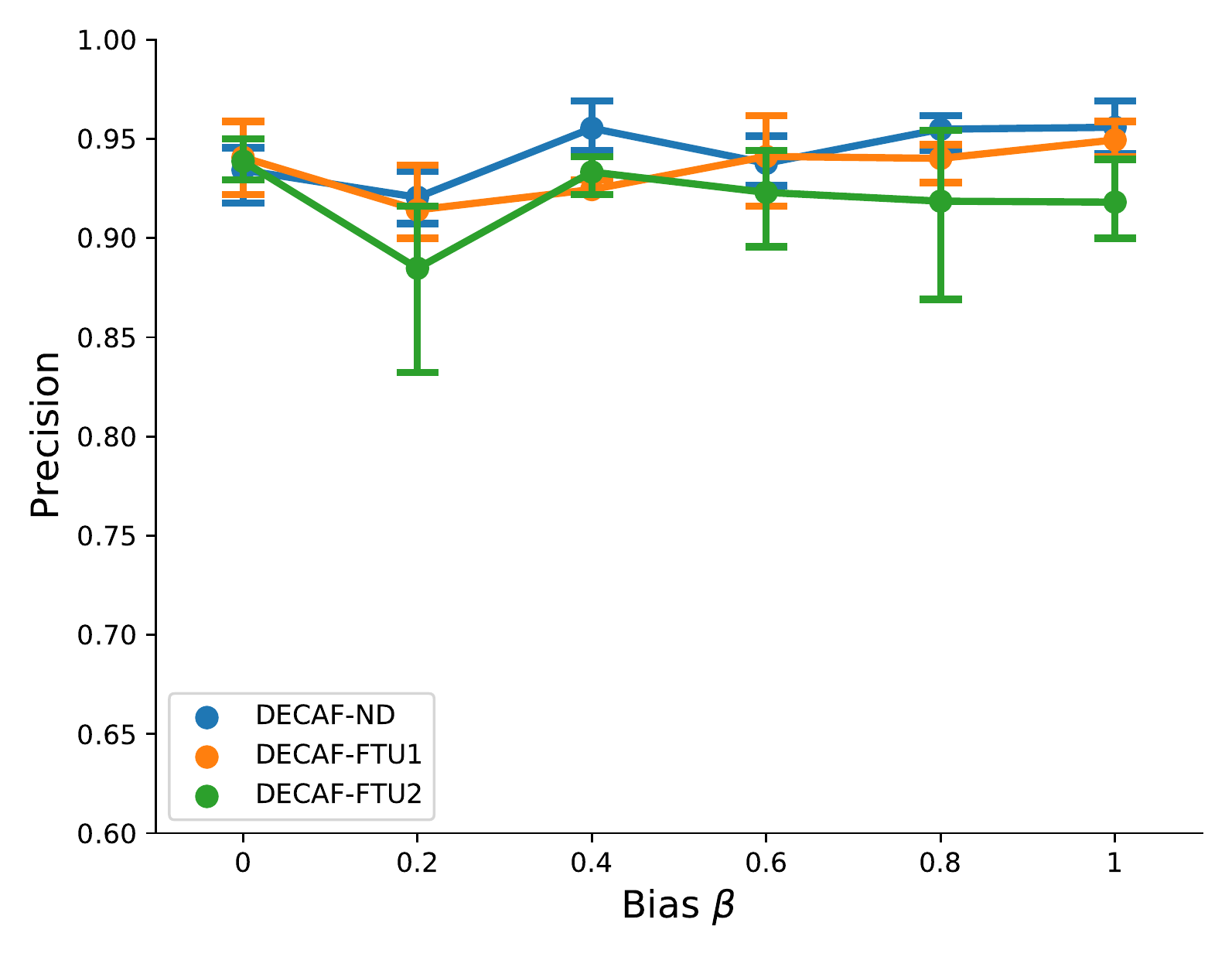}    
         \includegraphics[width=\textwidth]{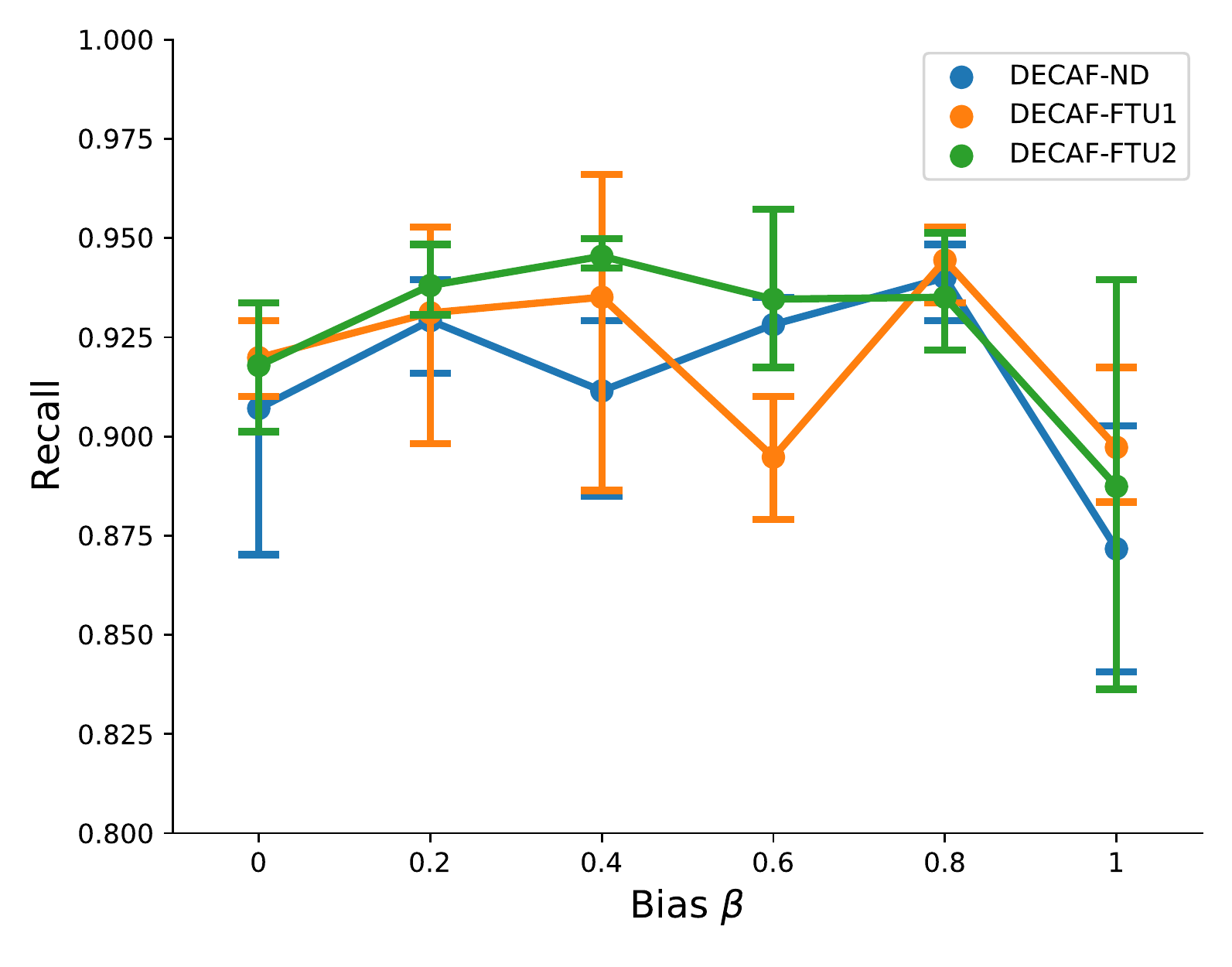}   
         \includegraphics[width=\textwidth]{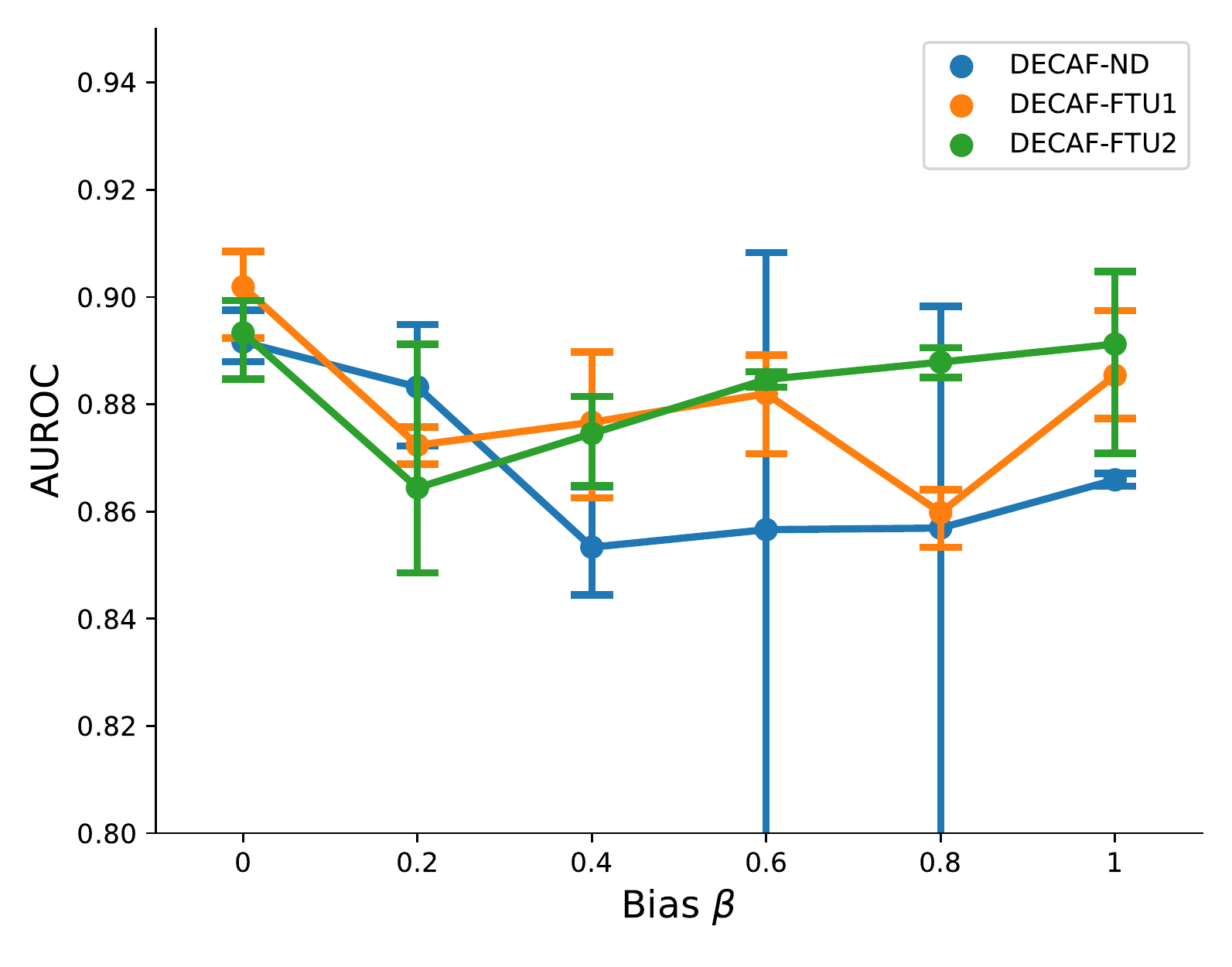} 
         \includegraphics[width=\textwidth]{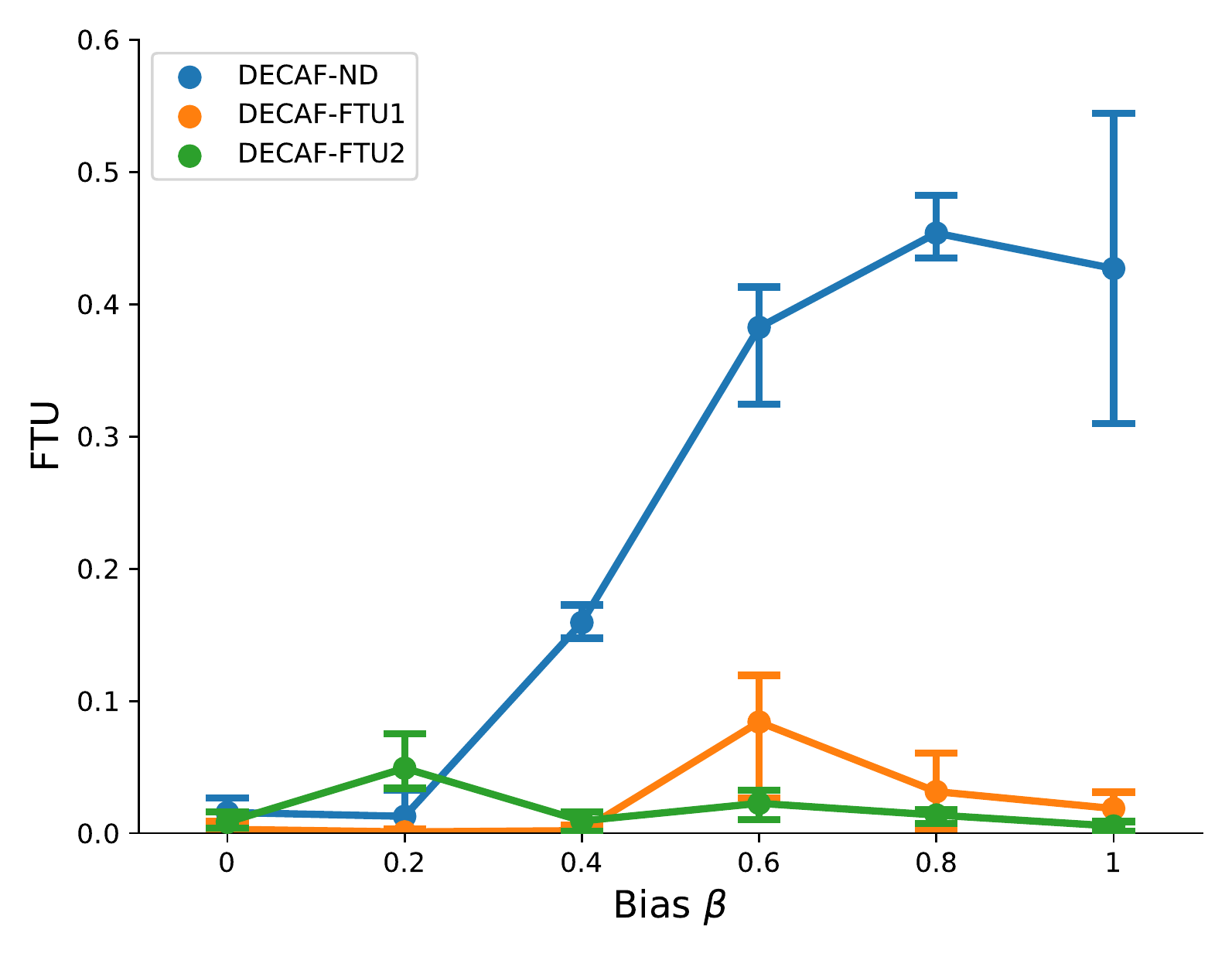} 
         \caption{All}
     \end{subfigure}
        \begin{subfigure}[b]{0.31\textwidth}
         \centering
         \includegraphics[width=\textwidth]{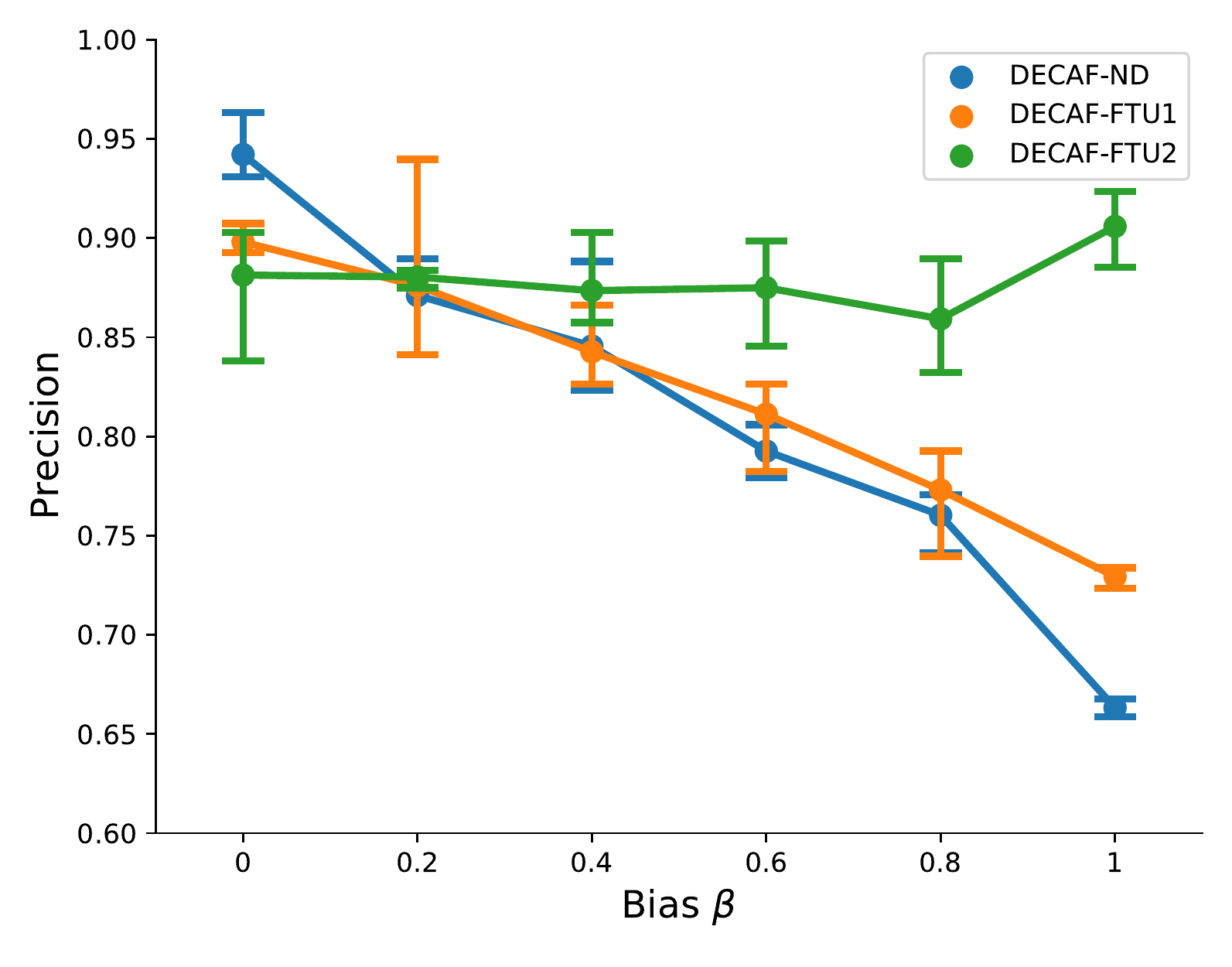}    
         \includegraphics[width=\textwidth]{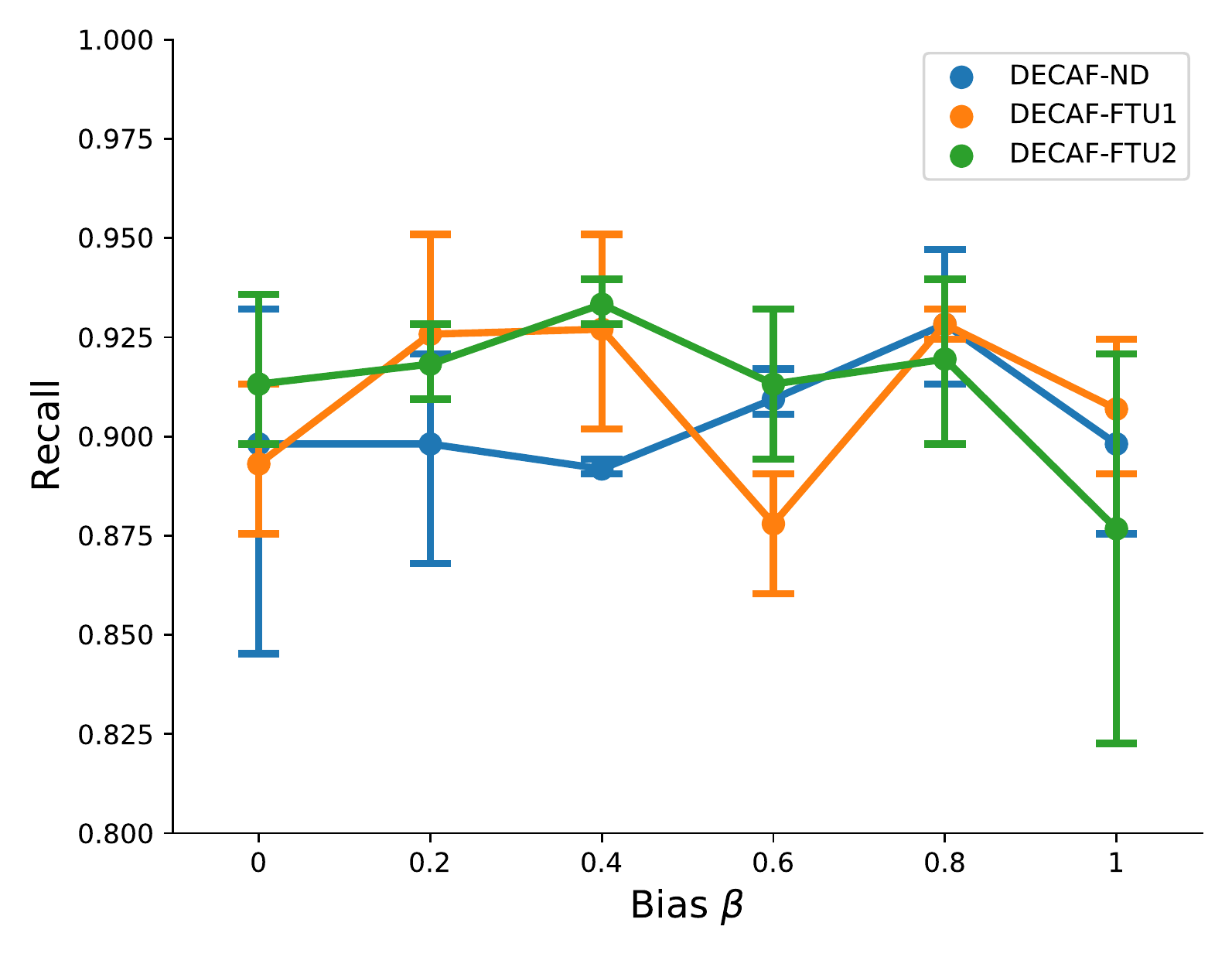}   
         \includegraphics[width=\textwidth]{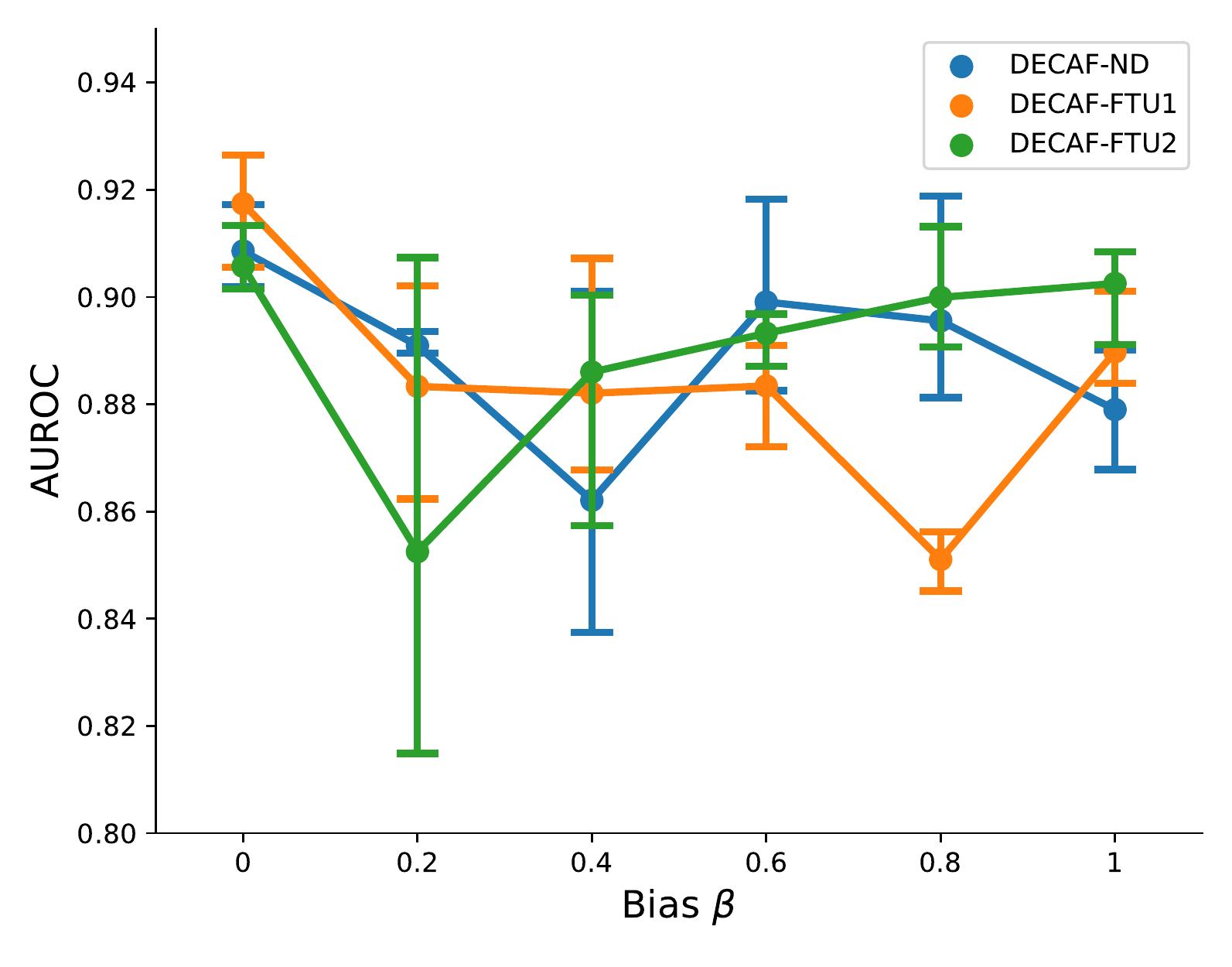} 
         \includegraphics[width=\textwidth]{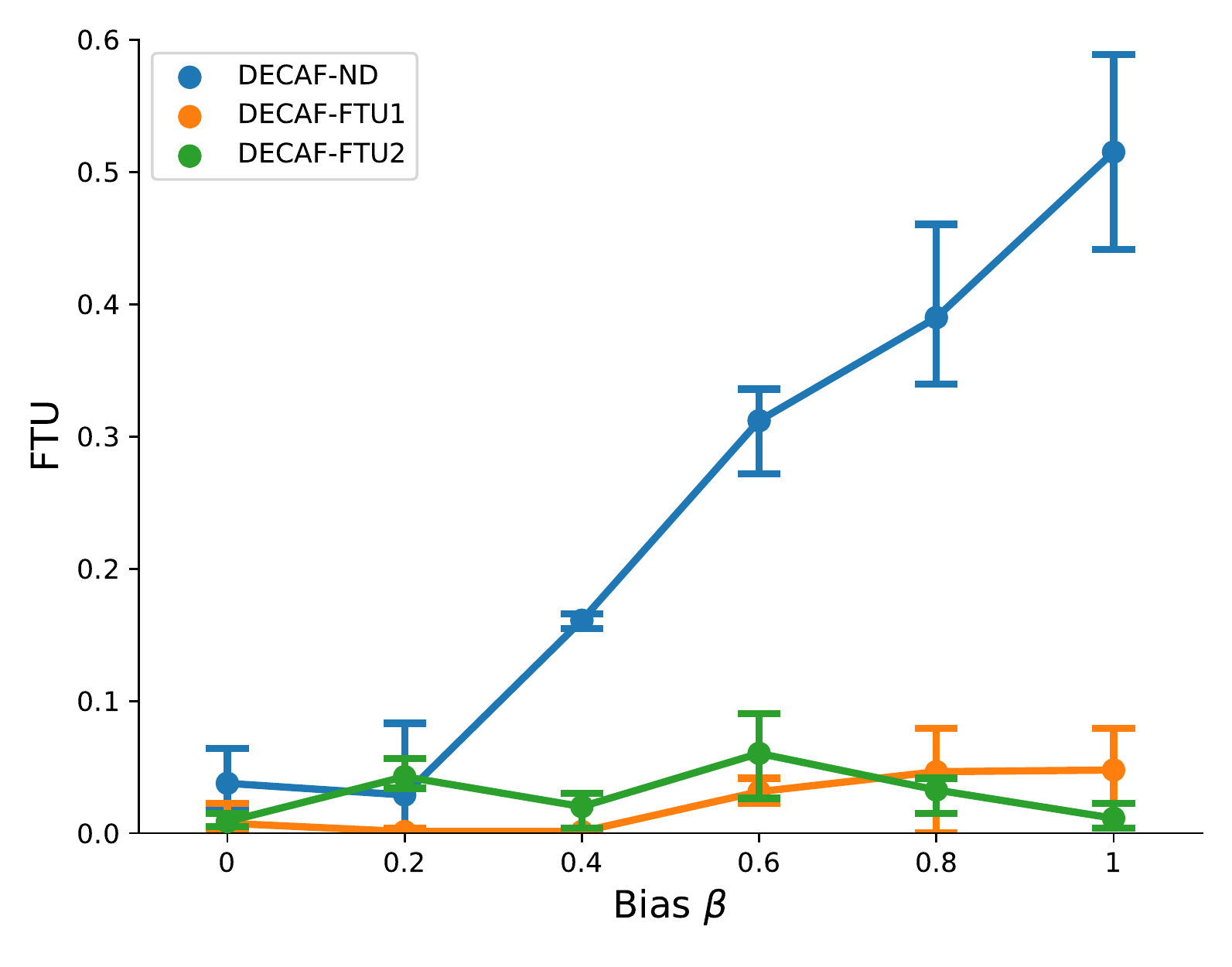} 
         \caption{Discriminated}
     \end{subfigure}
    \begin{subfigure}[b]{0.31\textwidth}
         \centering
         \includegraphics[width=\textwidth]{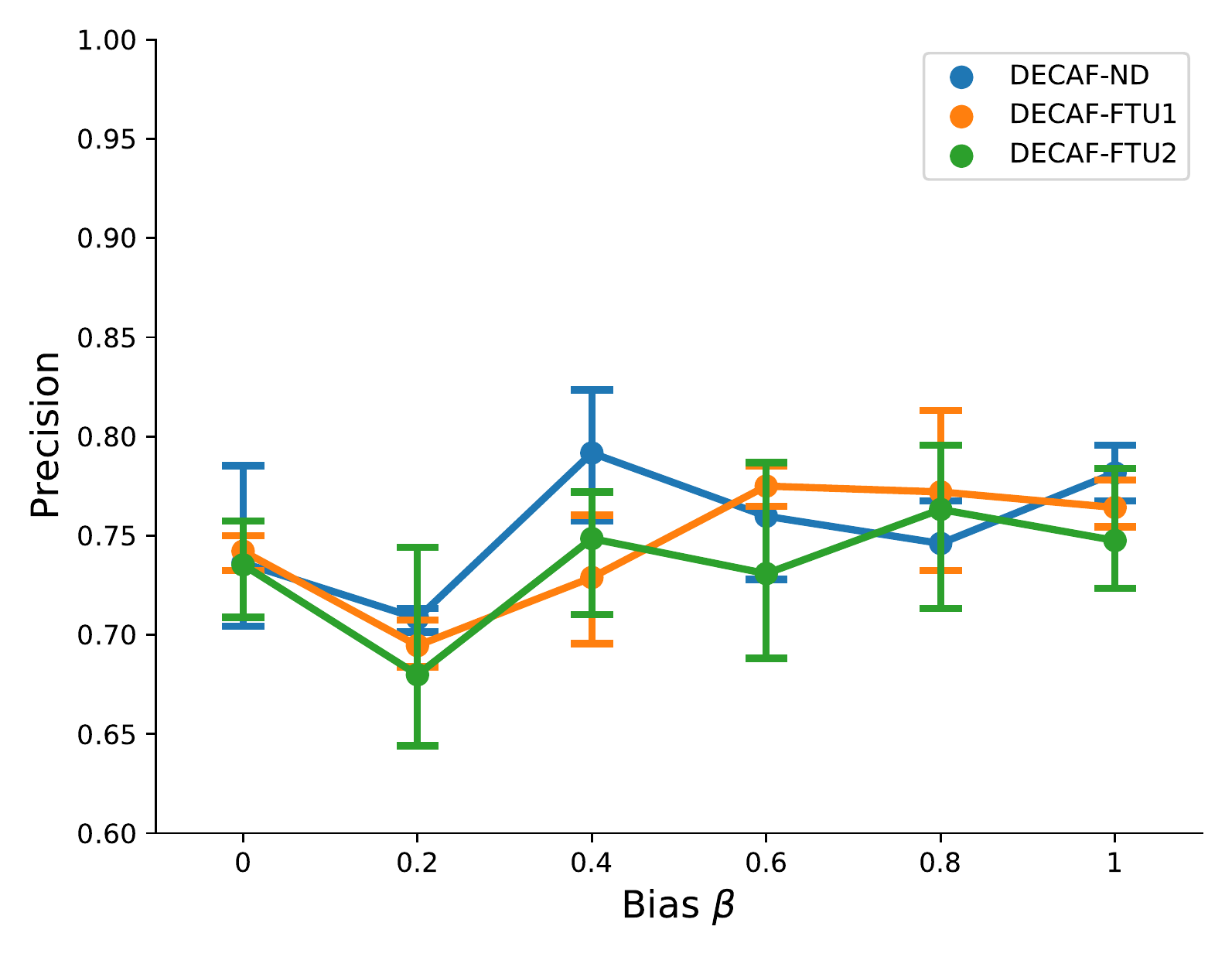}    
         \includegraphics[width=\textwidth]{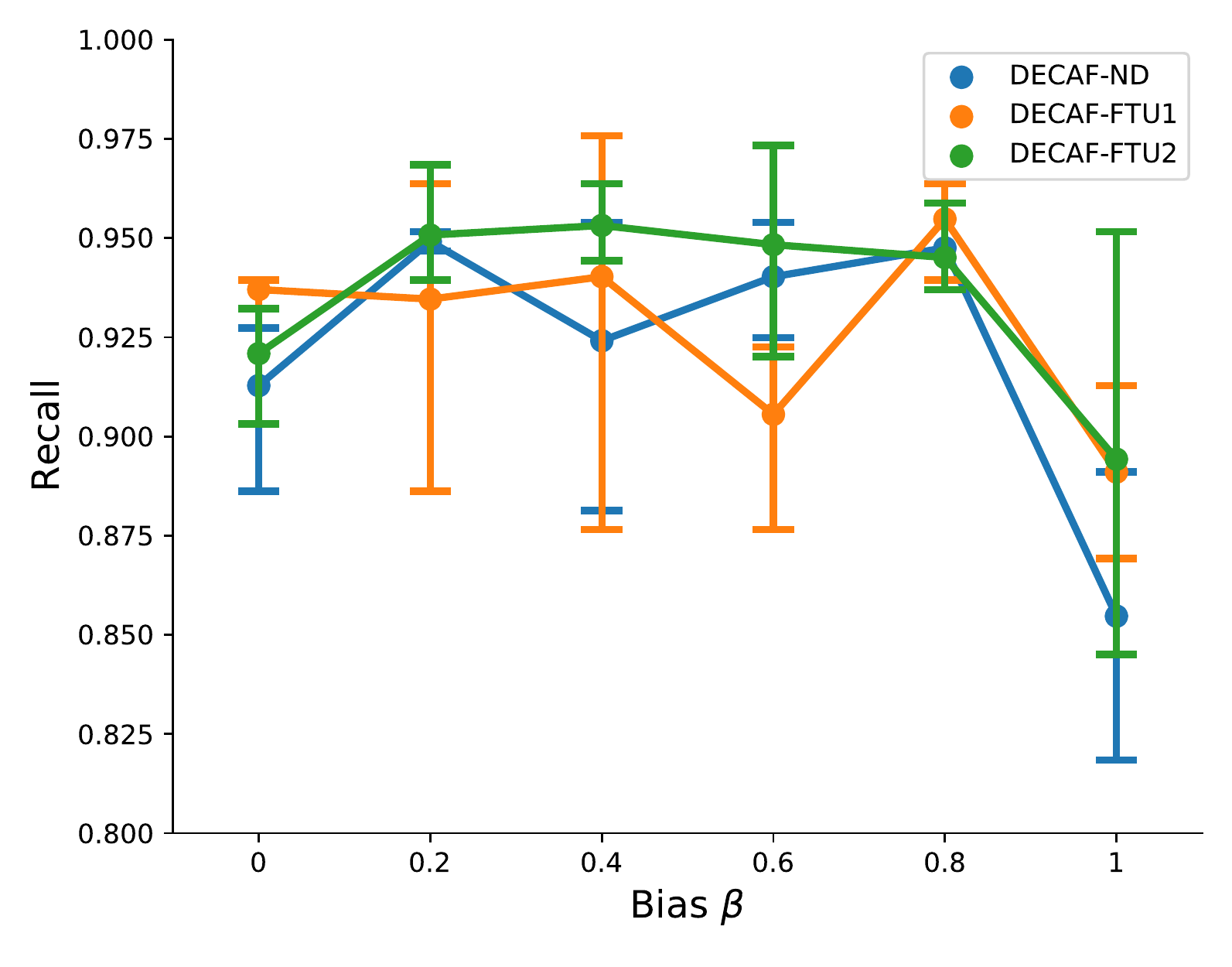}   
         \includegraphics[width=\textwidth]{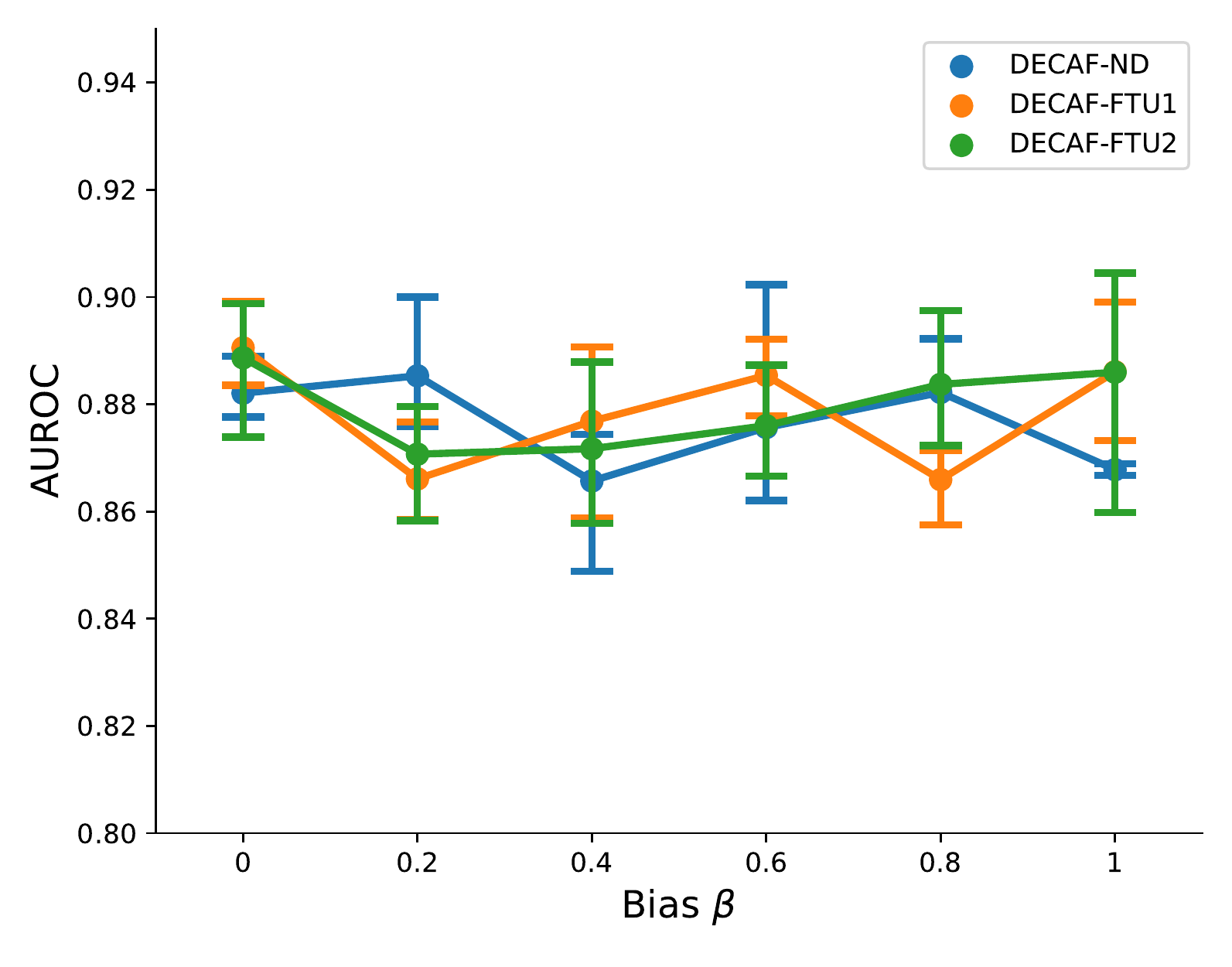} 
         \includegraphics[width=\textwidth]{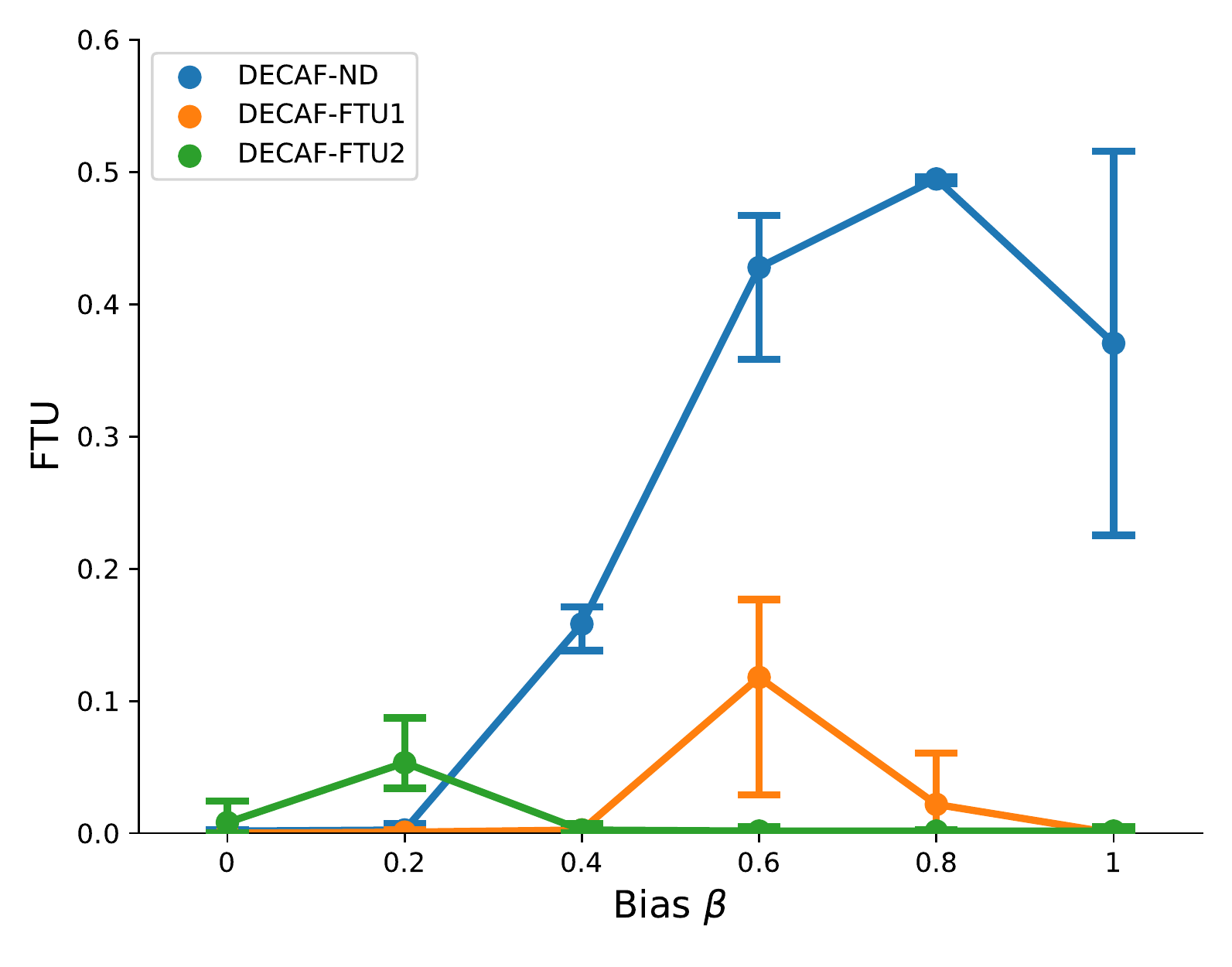} 
         \caption{Non-discriminated}
     \end{subfigure}

    \caption{Plot of precision, recall, AUROC, and FTU over various bias strengths for  \textbf{(a)} both populations (discriminated and non-discriminated), \textbf{(b)} discriminated population, and \textbf{(c)} non-discriminated population.}
    \label{fig:appx_surrogate}
\end{figure}

\section{DAG Sensitivity} \label{appx:ablation}

In this section, we investigate DECAF under imperfect knowledge.  Here, we are curious to understand what happens when our causal knowledge has: 1) has missing edges, 2) has spurious edges, i.e., edges that we assumed falsely, and 3) edges that are reversed in directionality.

We perform this experiment on the credit approval dataset \cite{uci}, with the known DAG used in the manuscript.  Using an identical experimental setup as described in Section 6.2 and a bias of $\beta=0.8$, we run our experiment 10 times each under random DAG perturbations.  Starting with the baseline DAG used in our credit approval experiment, we perform a sensitivity analysis to the following DAG perturbations:
\begin{itemize}
    \item \textbf{Edge removal} is done by randomly edges from the baseline DAG.
    \item \textbf{Edge addition} is done by randomly adding edges that are constrained by the following two criteria: 1) it does not create any cycles, and 2) it does not create any new indirect bias measures.  For the second condition, we ensure this by asserting that an edge is not added between the protected attribute \texttt{ethnicity} and an ancestor of \texttt{approved}.  We do this to ensure that the indirect bias is held consistent across each DAG instantiation and experimental run.
    \item \textbf{Edge reversal} is done by randomly reversing edges in the baseline DAG while preserving acyclicity.

\end{itemize}

Results for this experiment are shown in Figure~\ref{fig:appx_lineplots}. 
As expected, we see that edge removal degrades synthetic data quality (precision, recall, and AUROC) as the number of edges removed increases; this is not the case for adding and reversing edges -- where stable synthetic data quality is preserved.
In terms of debiasing, we see that DECAF-FTU and DECAF-ND is still able to debias consistently across all DAG perturbations.

\begin{figure}[!htbp]
    \centering
    \begin{subfigure}[b]{0.31\textwidth}
         \centering
         \includegraphics[width=\textwidth]{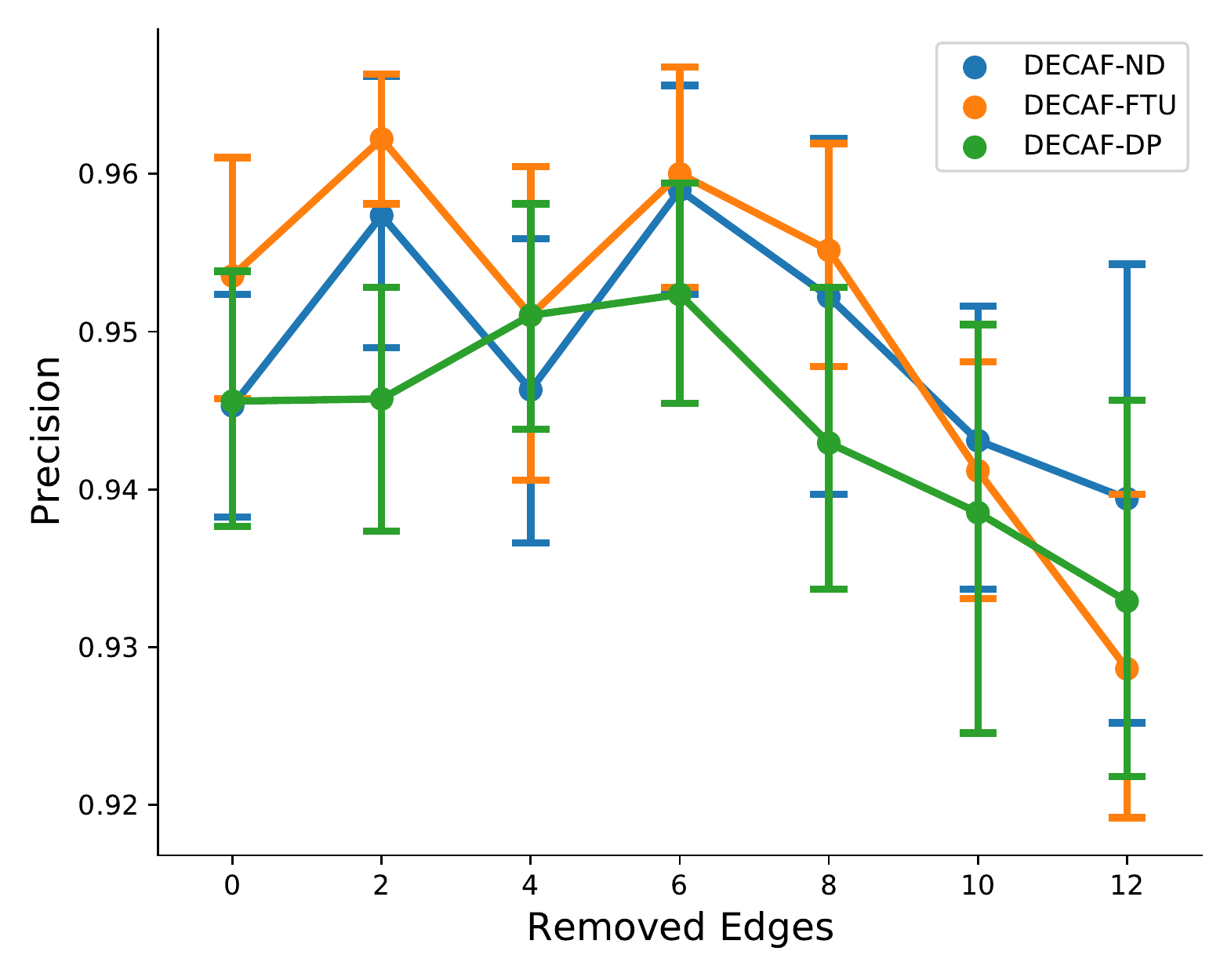}    
         \includegraphics[width=\textwidth]{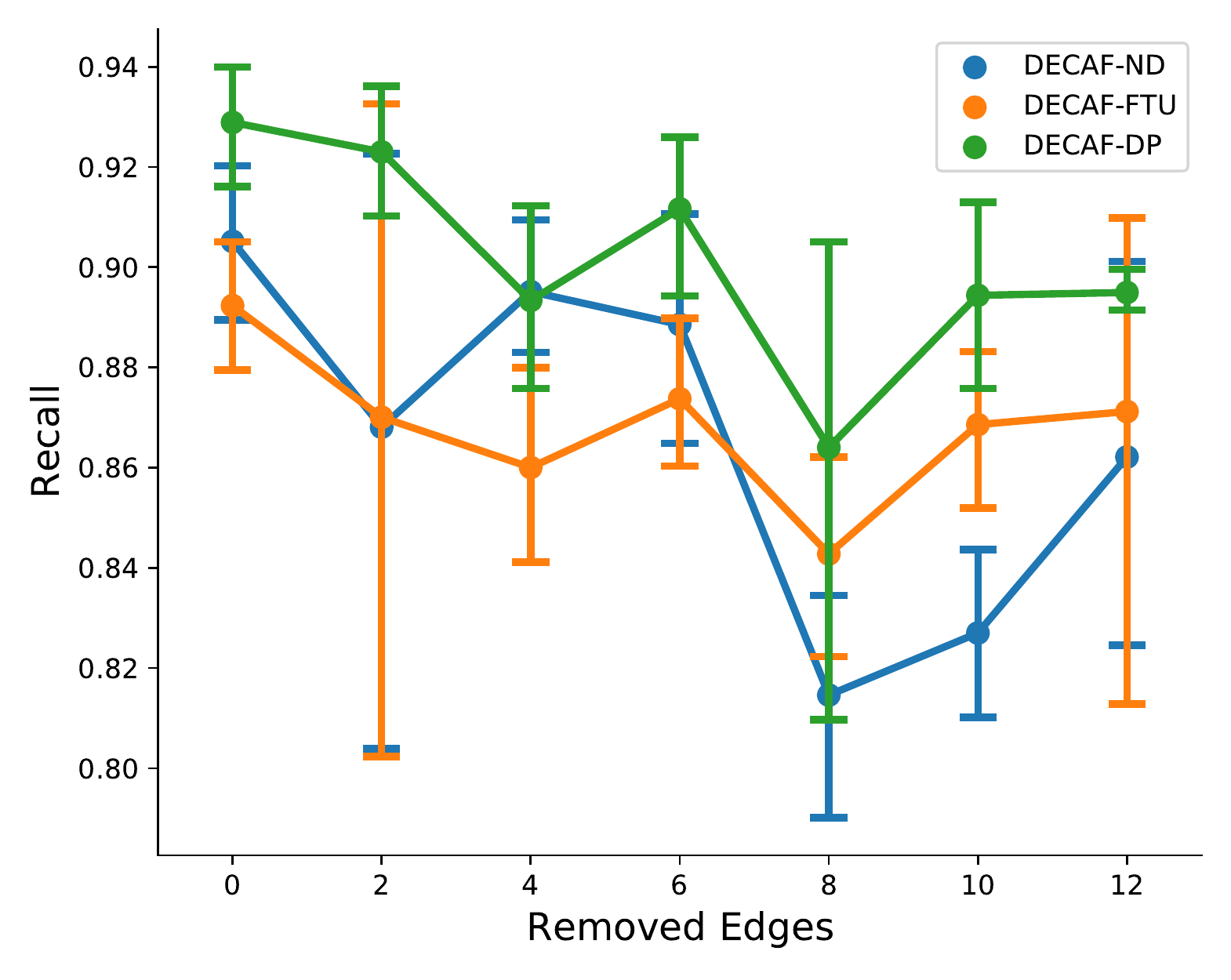}   
         \includegraphics[width=\textwidth]{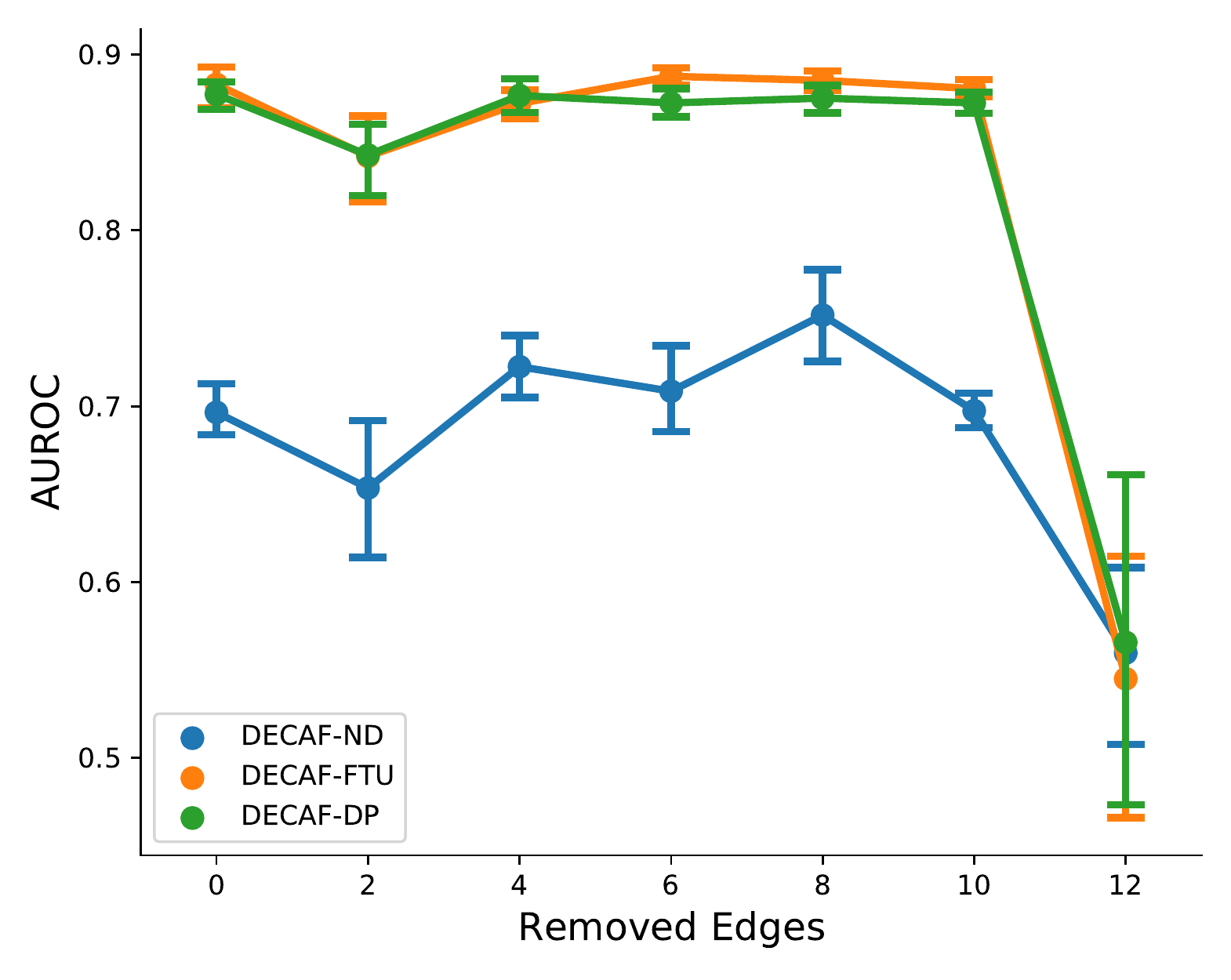} 
         \includegraphics[width=\textwidth]{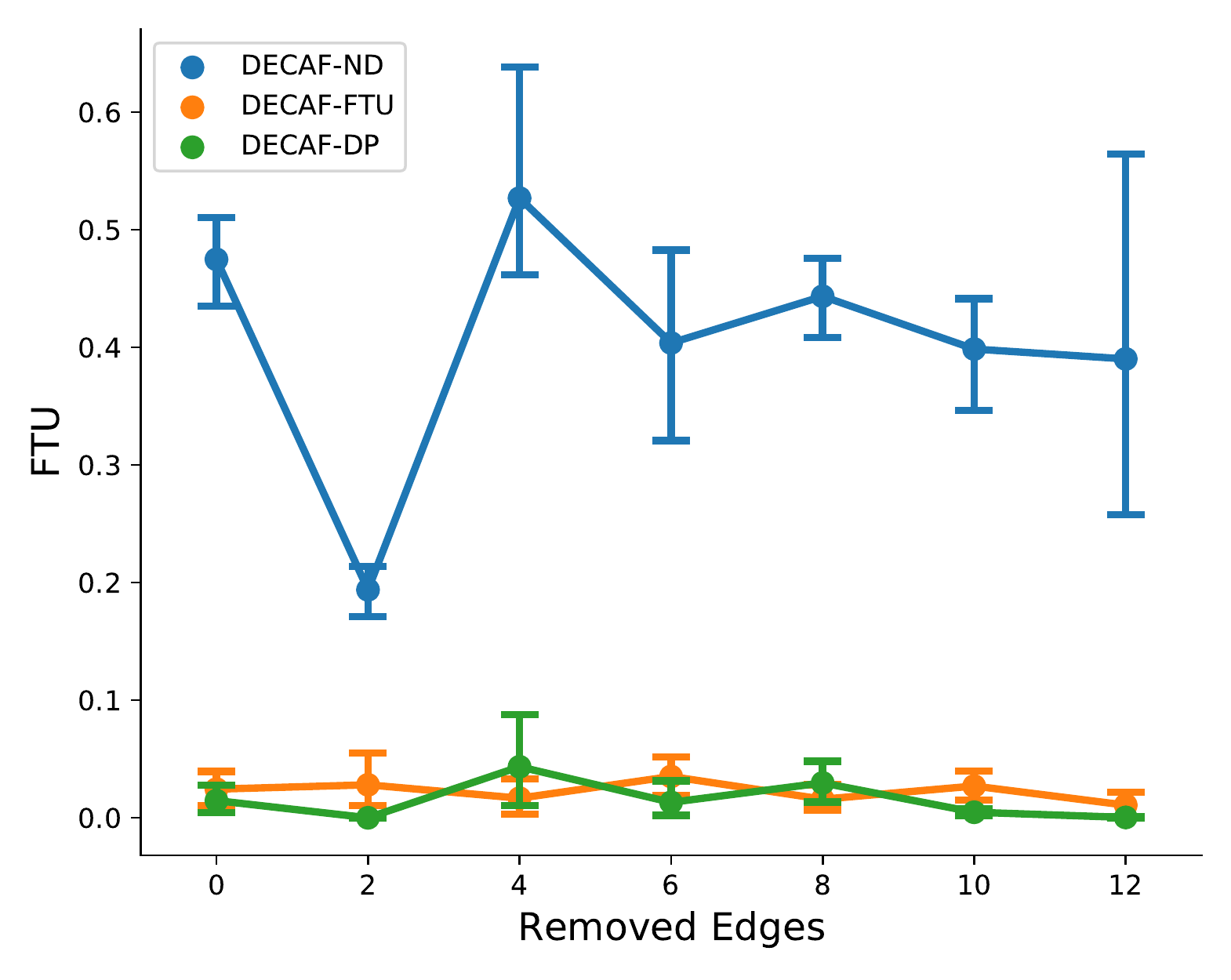} 
         \includegraphics[width=\textwidth]{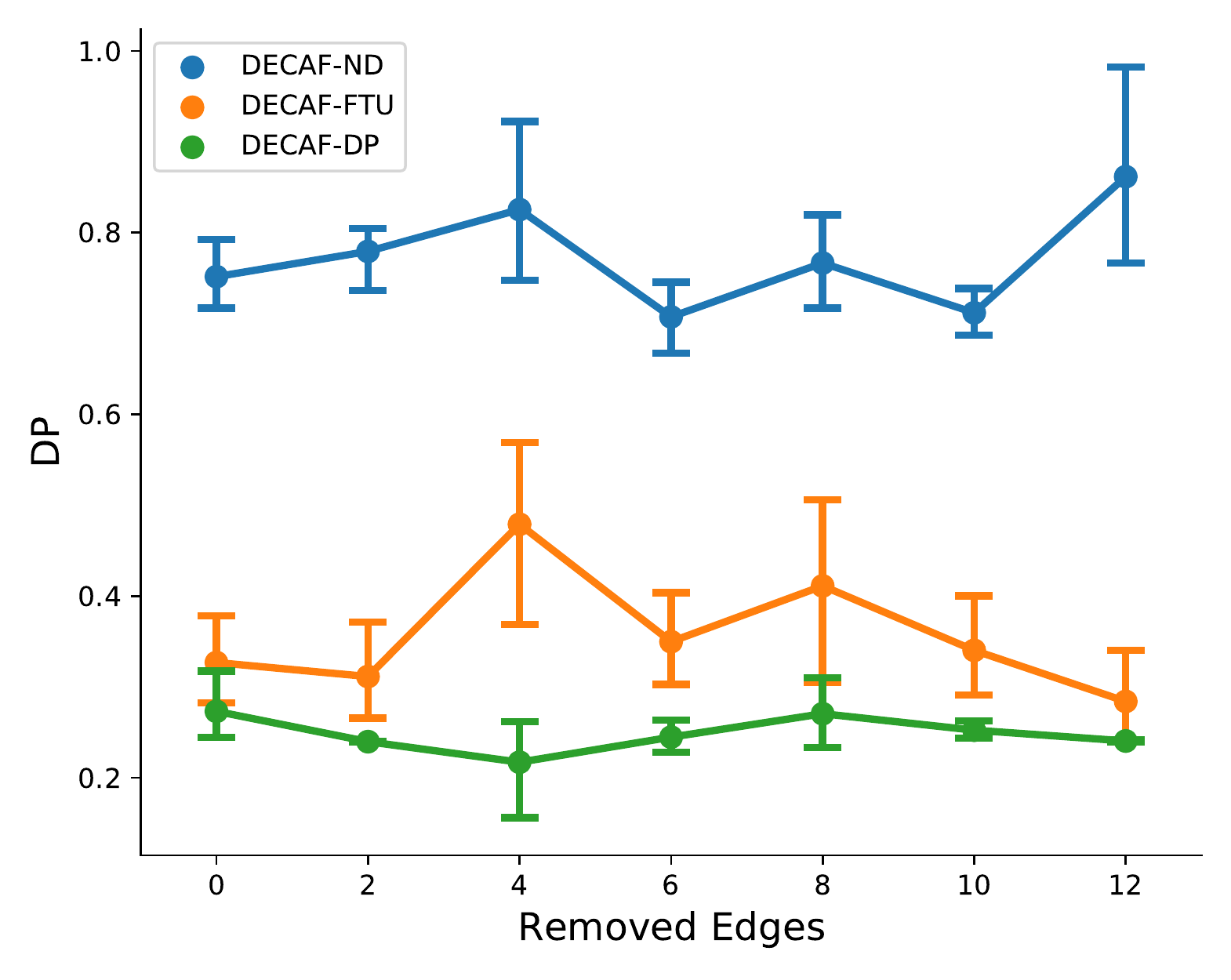} 
         \caption{Edge removal}
     \end{subfigure}
        \begin{subfigure}[b]{0.31\textwidth}
         \centering
         \includegraphics[width=\textwidth]{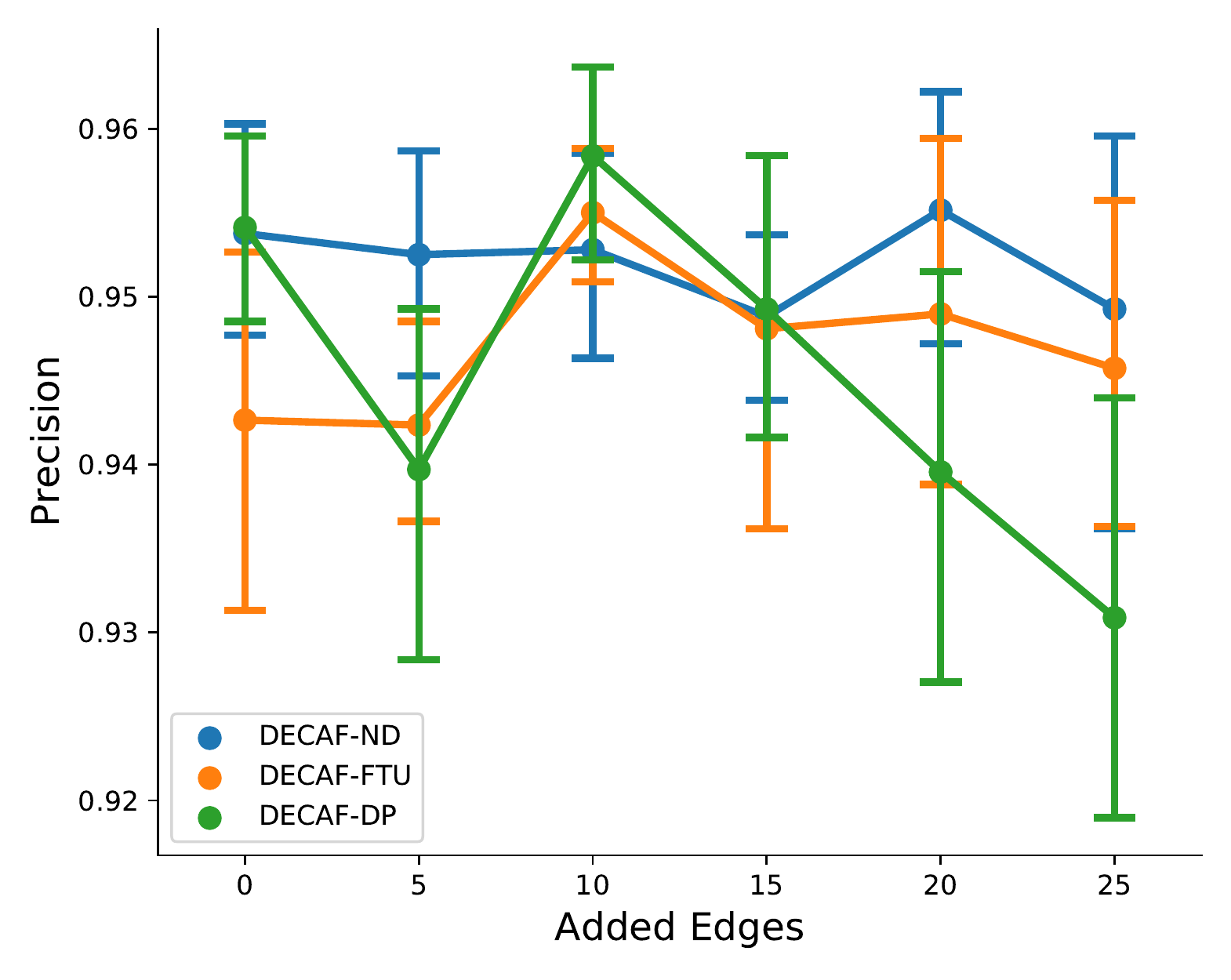}    
         \includegraphics[width=\textwidth]{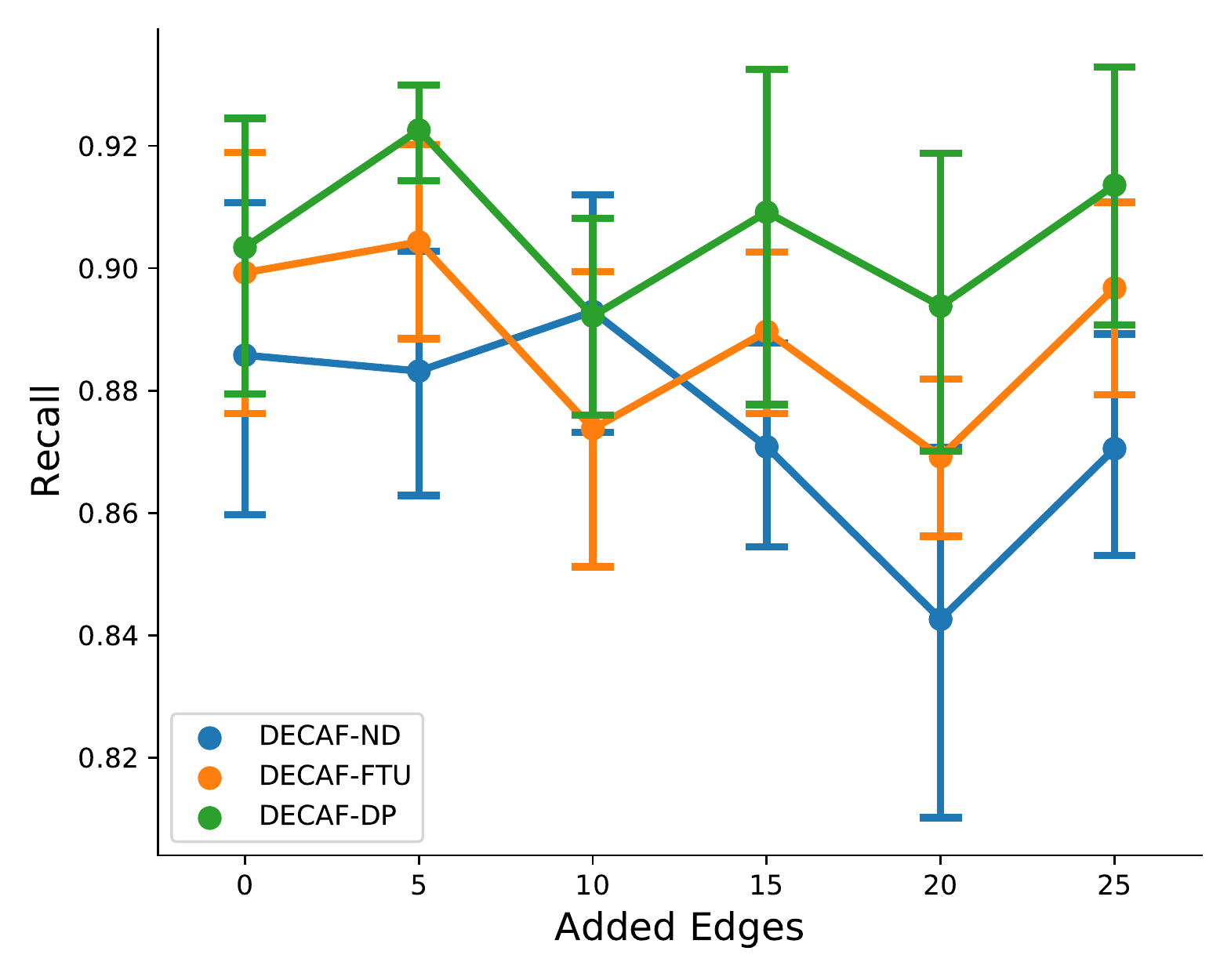}   
         \includegraphics[width=\textwidth]{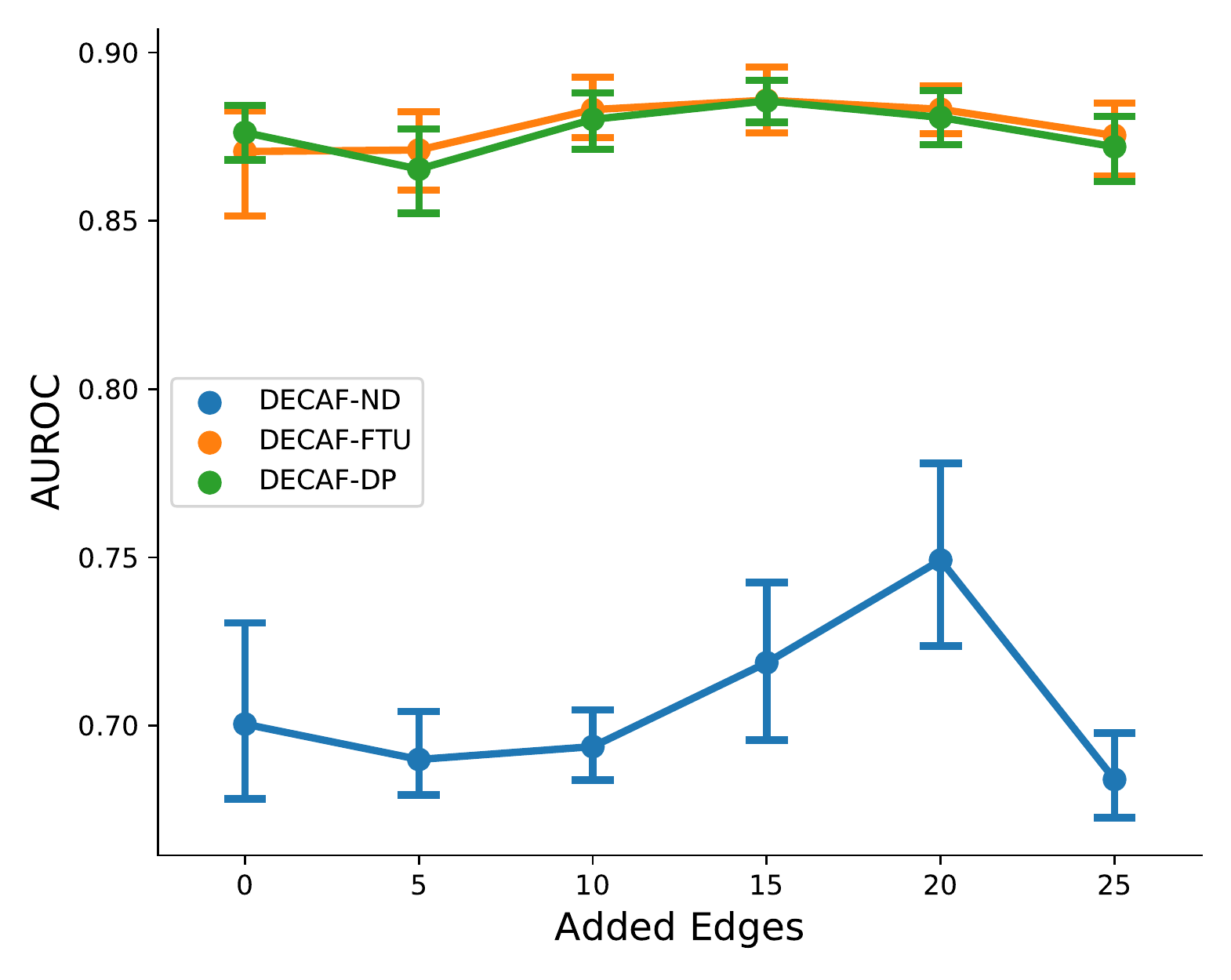} 
         \includegraphics[width=\textwidth]{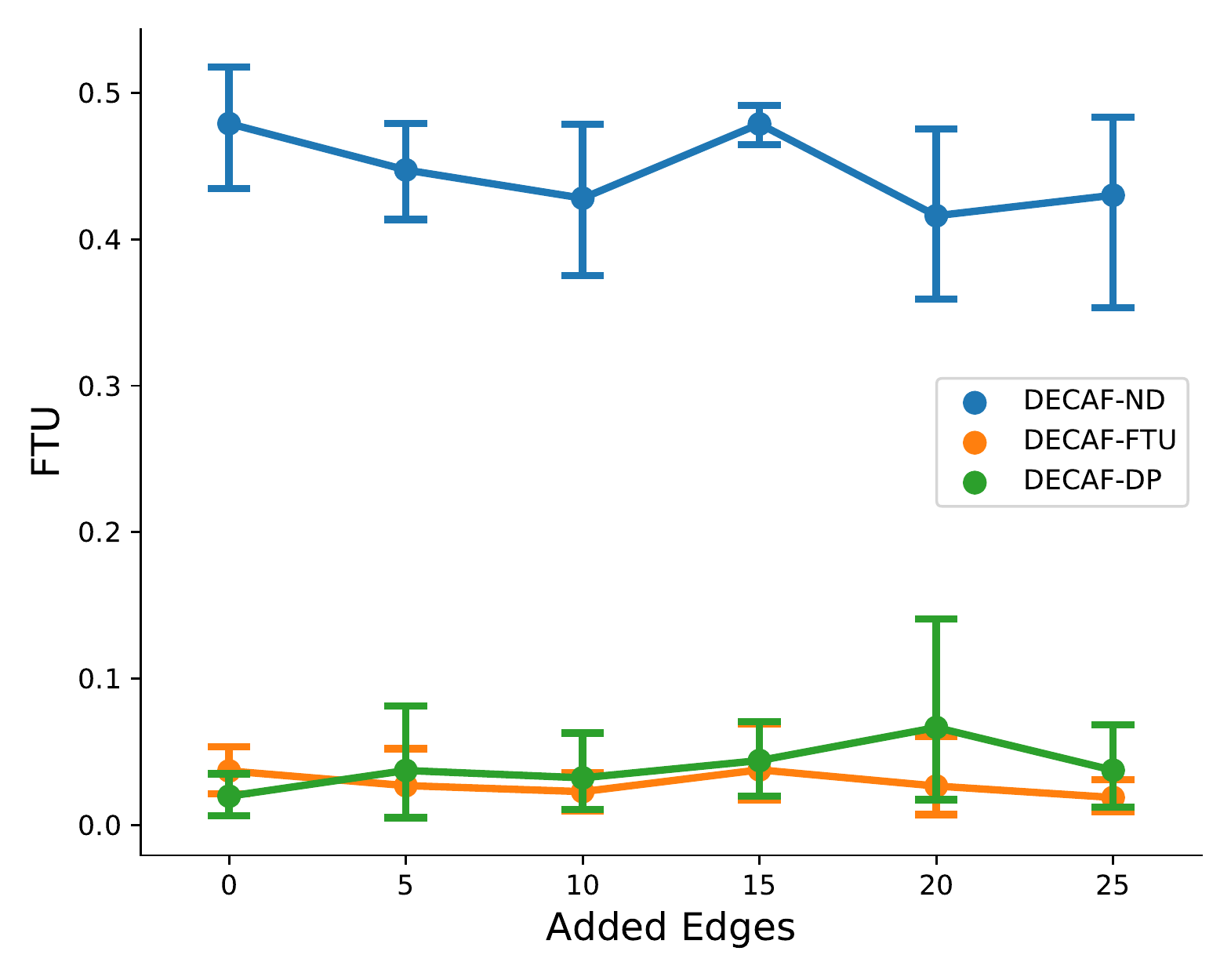} 
         \includegraphics[width=\textwidth]{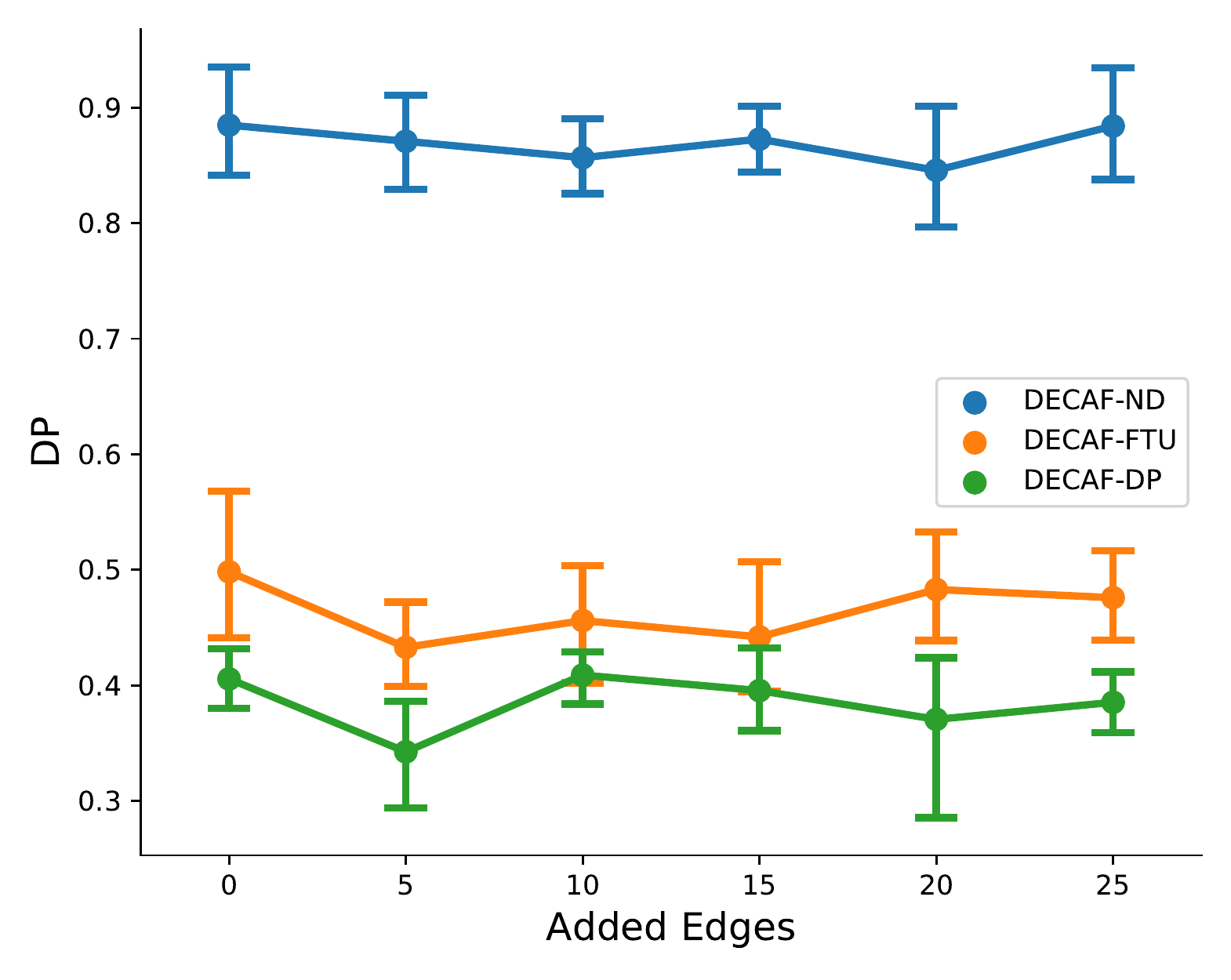} 
         \caption{Edge addition}
     \end{subfigure}
    \begin{subfigure}[b]{0.31\textwidth}
         \centering
         \includegraphics[width=\textwidth]{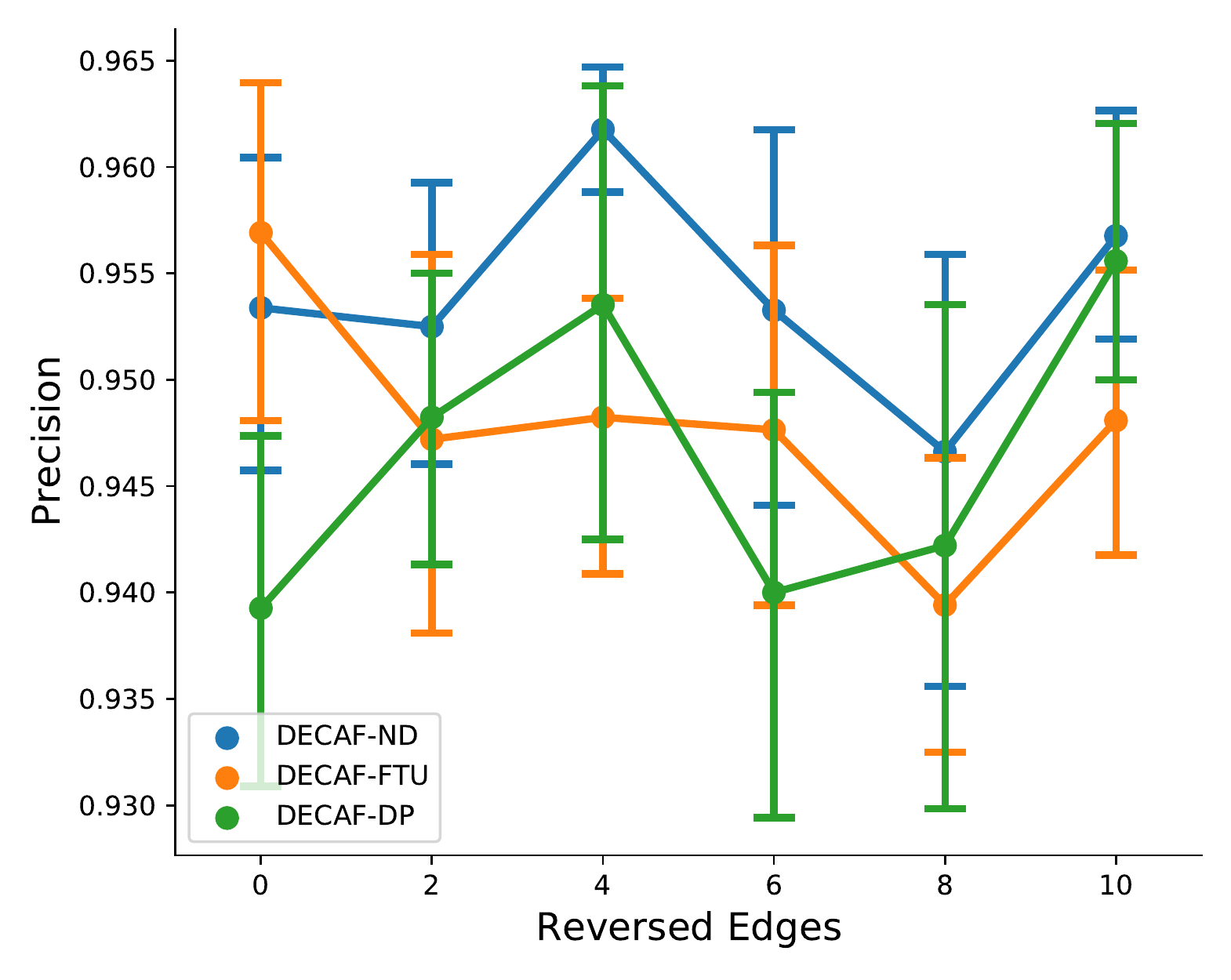}    
         \includegraphics[width=\textwidth]{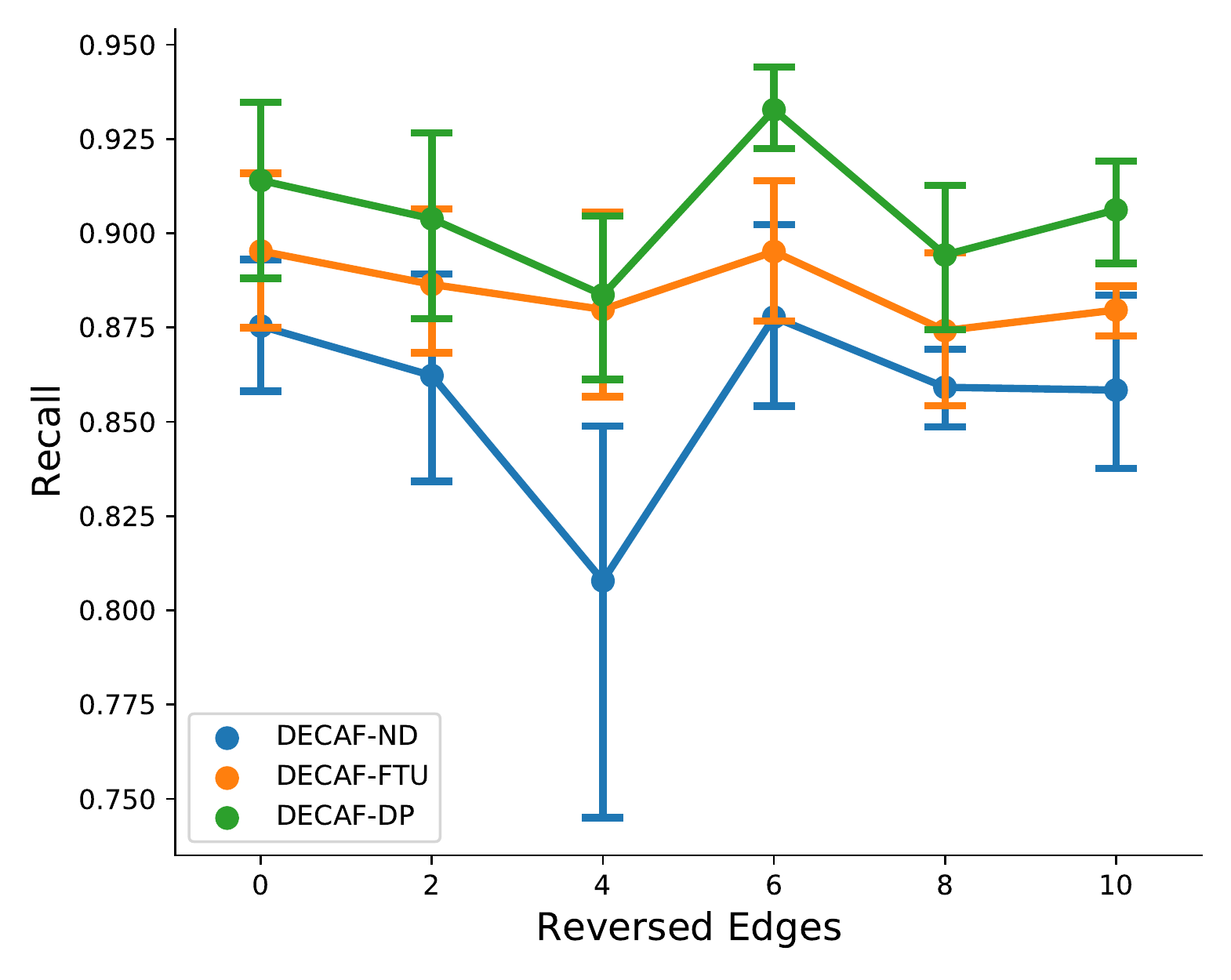}   
         \includegraphics[width=\textwidth]{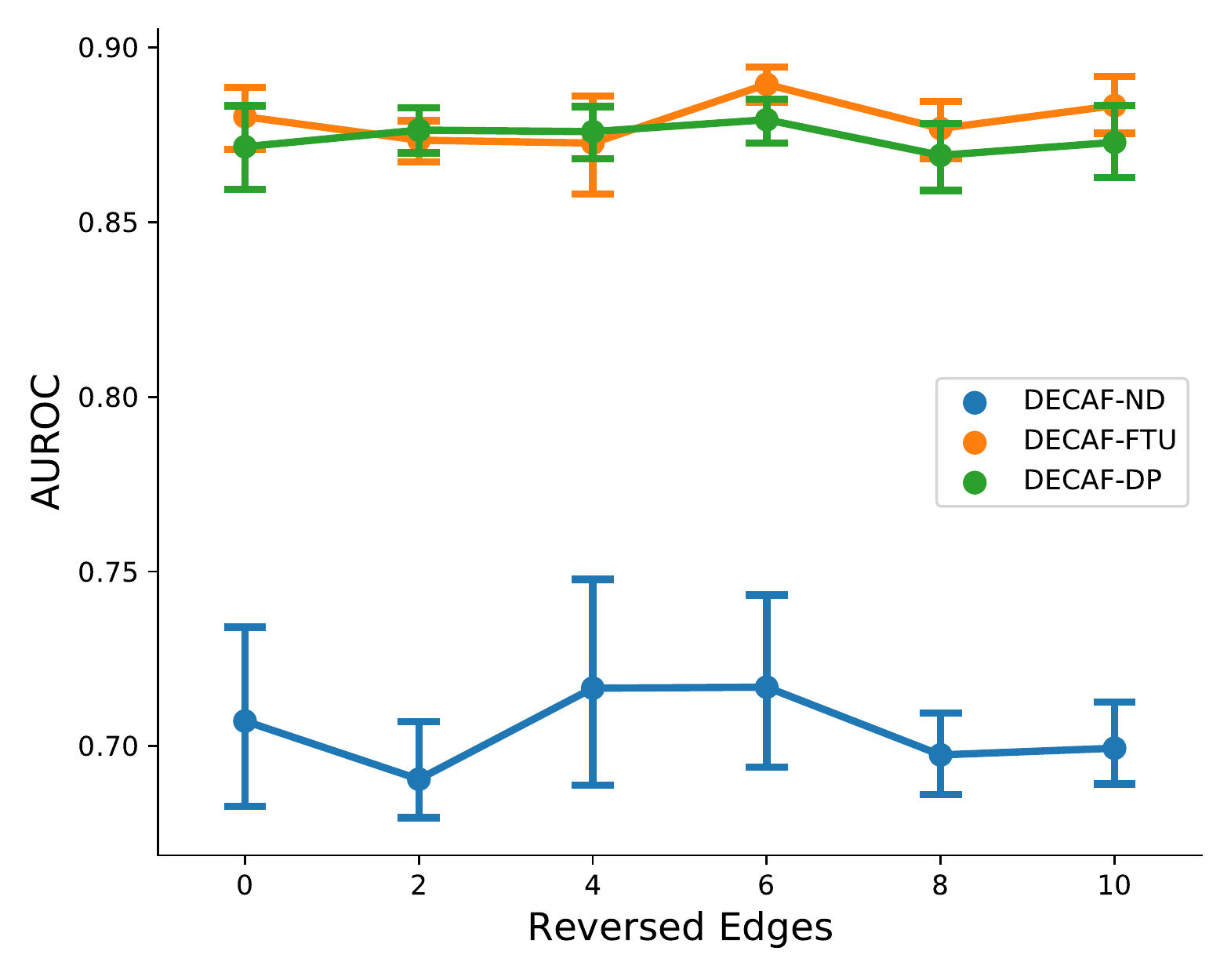} 
         \includegraphics[width=\textwidth]{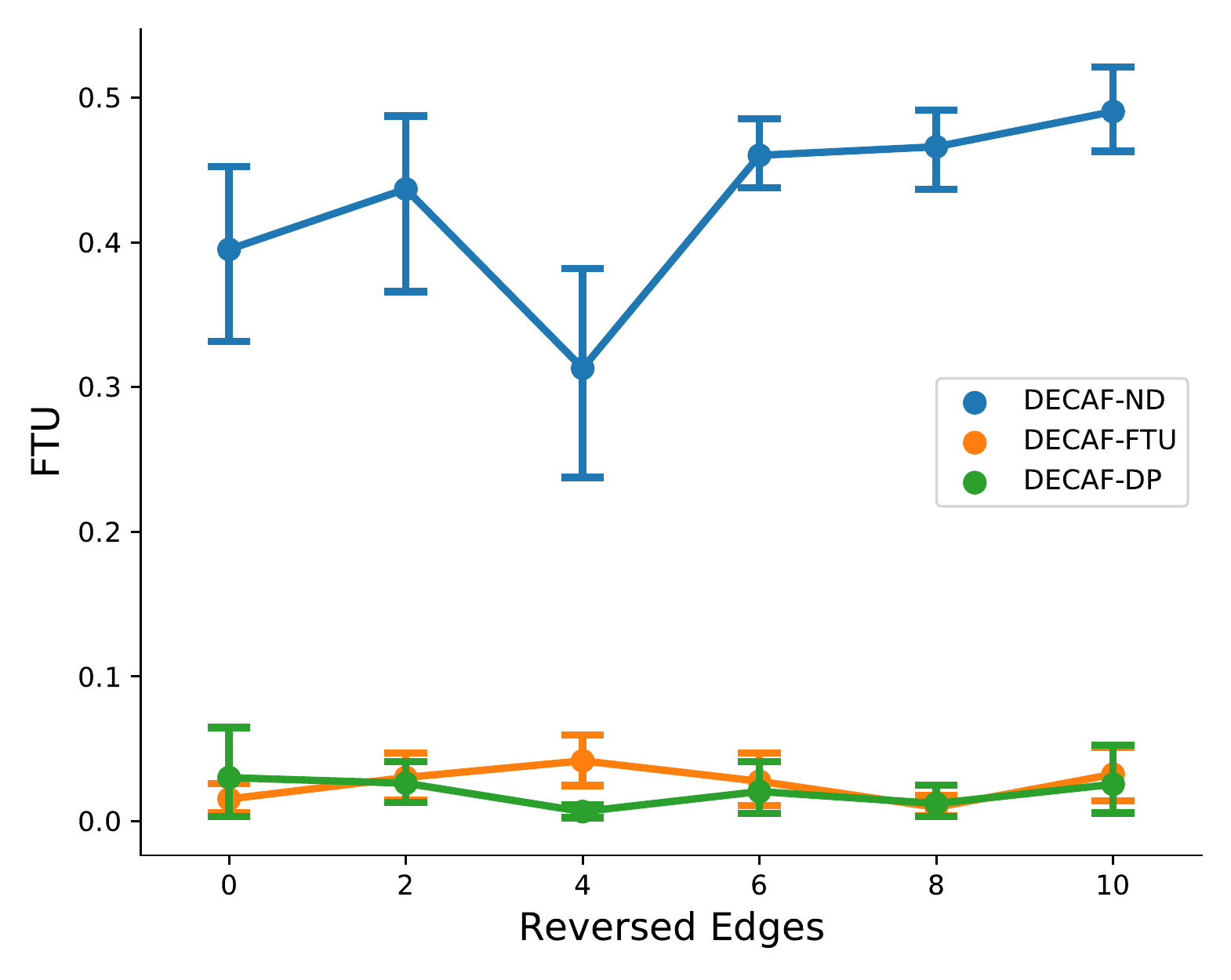} 
         \includegraphics[width=\textwidth]{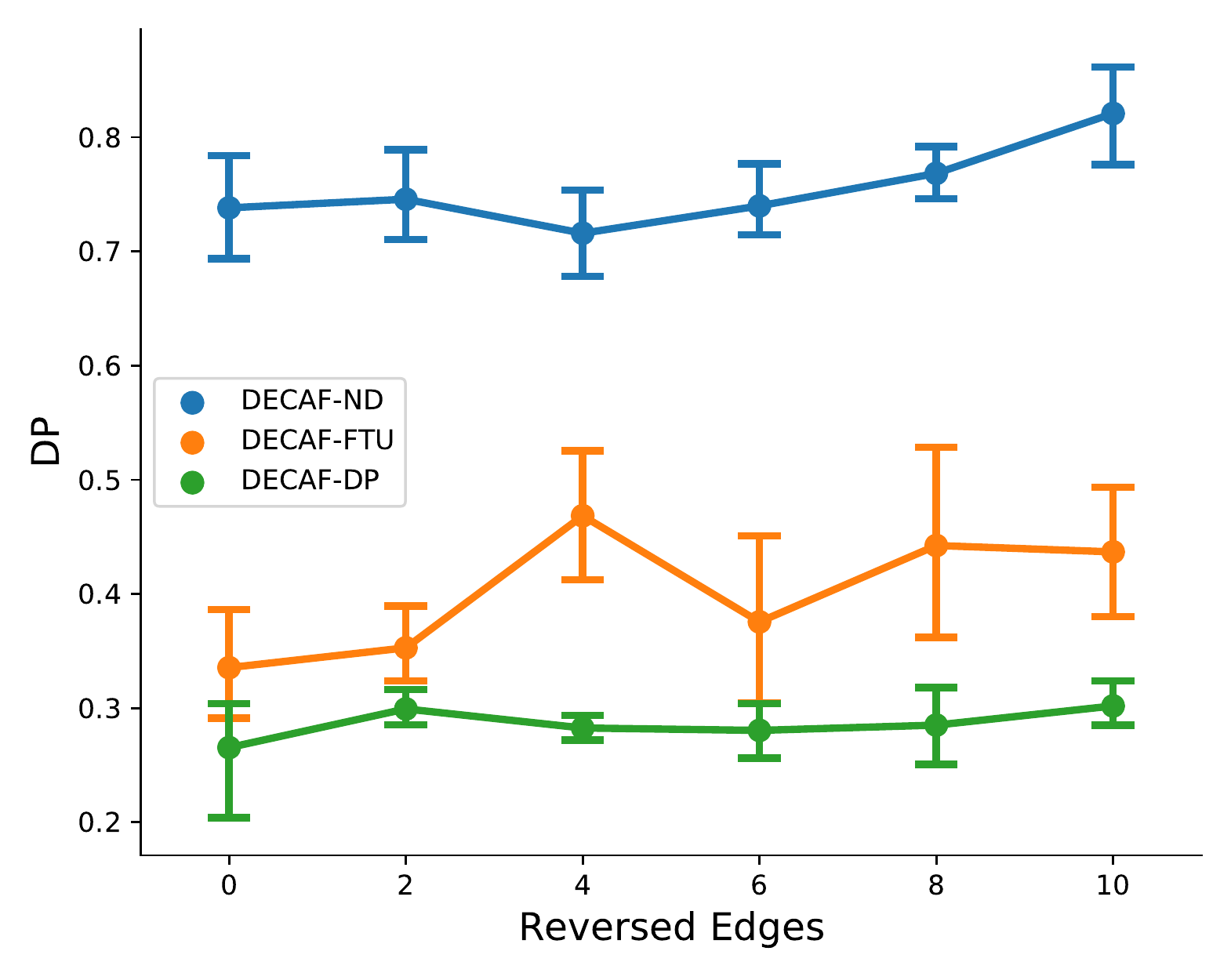}
         \caption{Edge reversal}
     \end{subfigure}
    \caption{Plot of precision, recall, AUROC, FTU, and DP over \textbf{(a)} edge removal, \textbf{(b)} edge addition, and \textbf{(c)} edge reversal on the credit approval dataset.}
    \label{fig:appx_lineplots}
\end{figure}


\section{Hidden Confounders} \label{appx:hiddenconfounder}
\begin{figure}
    \centering

    \begin{subfigure}[b]{0.194\textwidth}
         \centering
         \includegraphics[width=\textwidth]{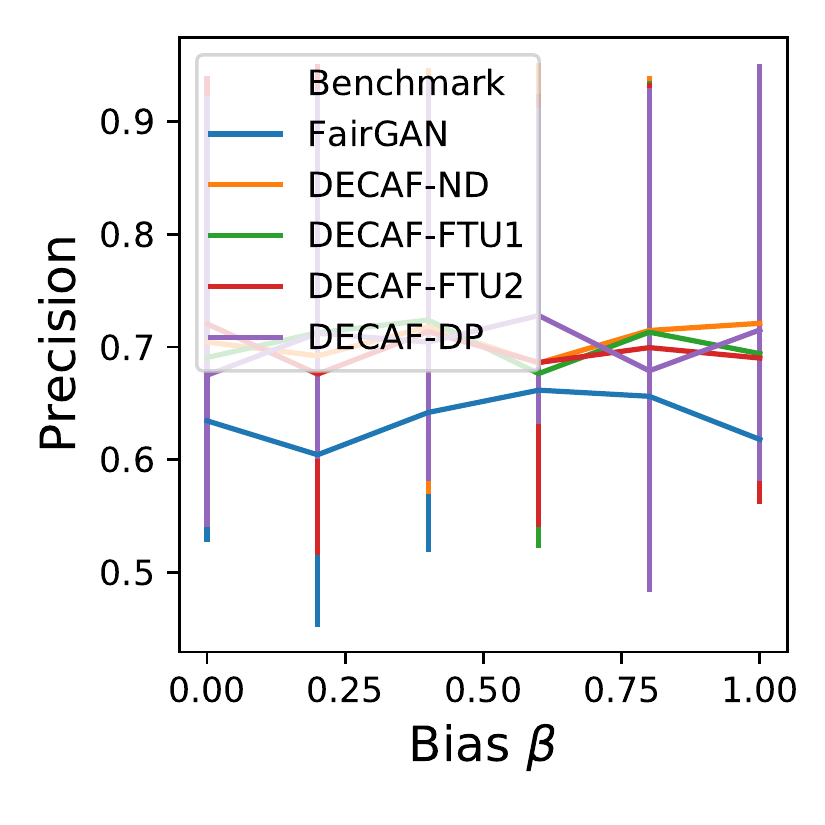}    \vspace{-0.7cm}
         \caption{Precision$\uparrow$}
         \label{fig:prec2}
     \end{subfigure}
        \begin{subfigure}[b]{0.194\textwidth}
         \centering
         \includegraphics[width=\textwidth]{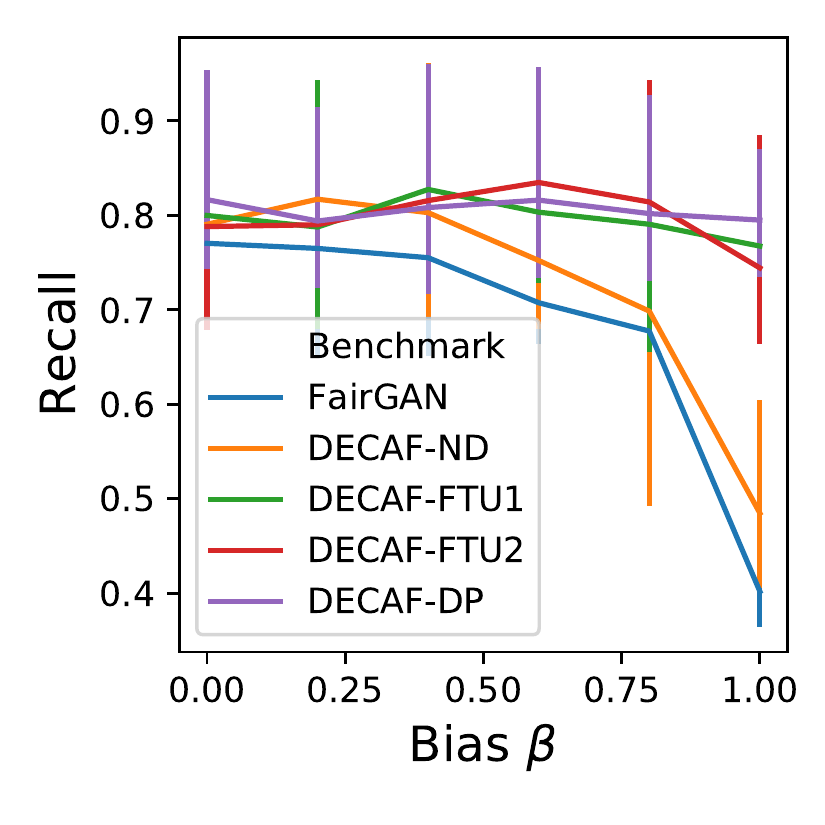}\vspace{-0.3cm}
         \caption{Recall$\uparrow$}
         \label{fig:rec2}
     \end{subfigure}
    \begin{subfigure}[b]{0.194\textwidth}
         \centering
         \includegraphics[width=\textwidth]{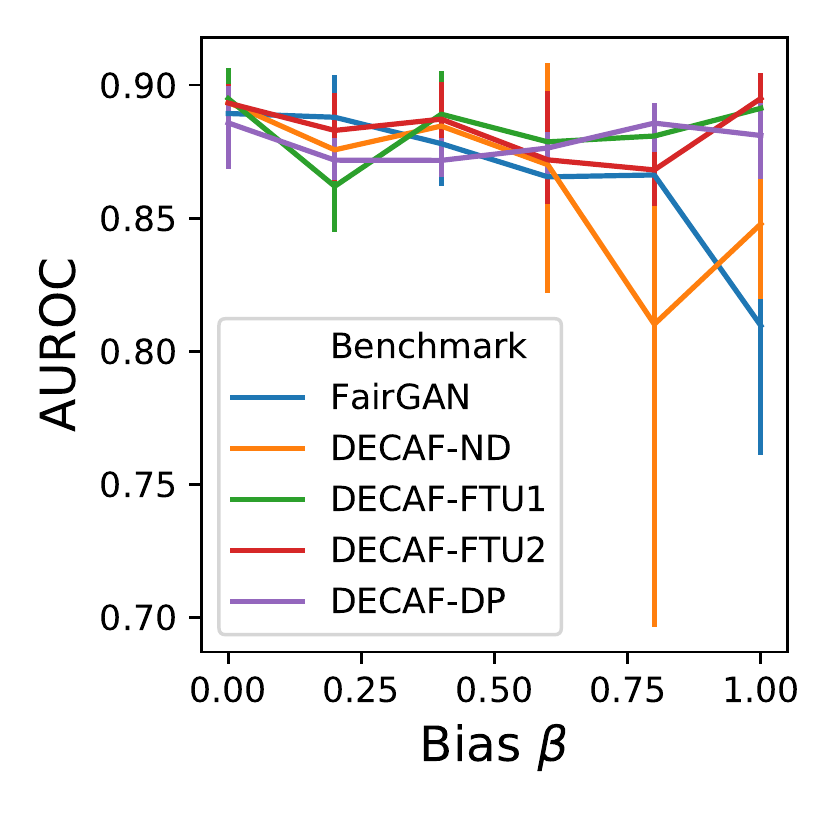}\vspace{-0.3cm}
         \caption{AUROC$\uparrow$}
         \label{fig:auc2}
     \end{subfigure}
        \begin{subfigure}[b]{0.194\textwidth}
         \centering
         \includegraphics[width=\textwidth]{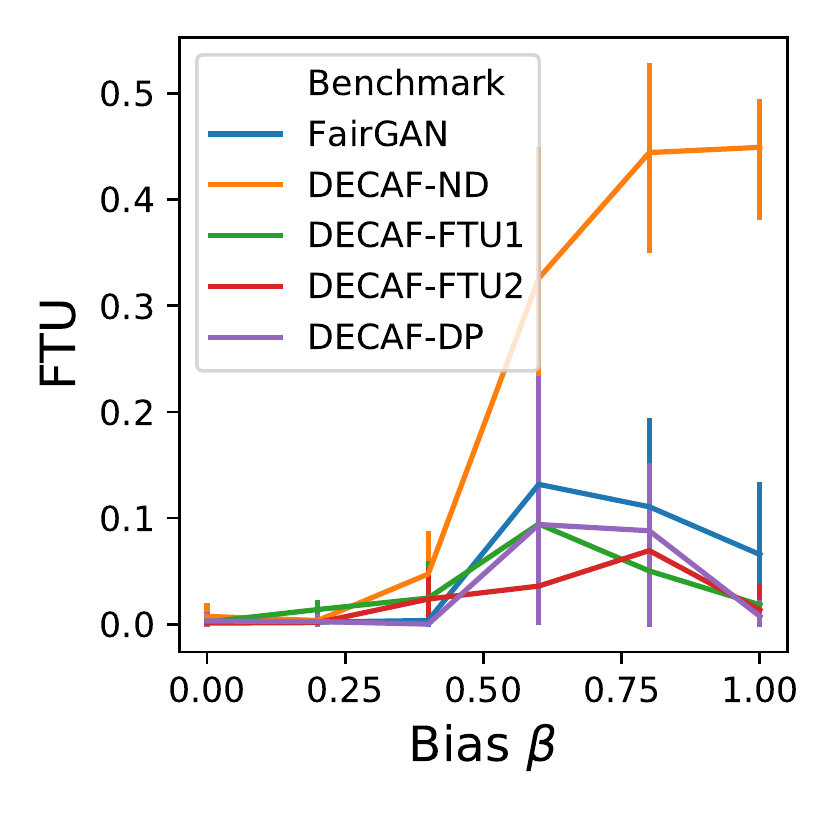}\vspace{-0.3cm}
         \caption{FTU$\downarrow$}
         \label{fig:ftu2}
     \end{subfigure}    
    \begin{subfigure}[b]{0.194\textwidth}
         \centering
         \includegraphics[width=\textwidth]{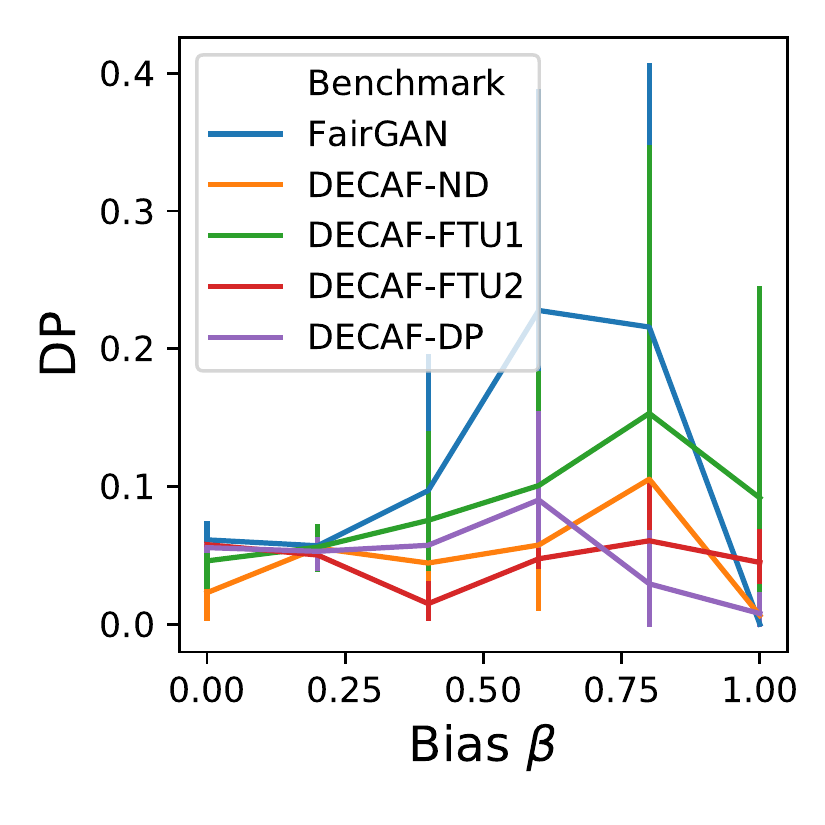}\vspace{-0.3cm}
         \caption{DP$\downarrow$}
         \label{fig:dp2}
     \end{subfigure}  

    \caption{Plot of precision \textbf{(a)}, recall \textbf{(b)}, AUROC \textbf{(c)}, FTU \textbf{(d)}, and DP \textbf{(e)} over bias strength $\beta$ for experiments with hidden confounding.  FairGAN performs similarly in terms of DP and FTU, but DECAF-FTU and DECAF-DP have significantly better data quality as well as down stream prediction capability (AUROC).}

    \label{fig:confounder}
\end{figure}

In this section, we examine DECAF under hidden confounders on the Credit Approval dataset.  Assuming the DAG in Figure~\ref{fig:creditDAG}, we create a hidden confounder by removing the variable for \texttt{education\_level} from the dataset and DAG. 
We then replicate the experimental setup for Section~\ref{sec:exp_ii} under identical conditions.  We present the results in Figure~\ref{fig:confounder}. 

\end{document}